\documentclass{article} 
\usepackage{iclr2016_conference,times}
\usepackage{hyperref}
\usepackage{url}
\usepackage{graphicx}
\usepackage[titletoc,title]{appendix}
\usepackage{amsfonts}
\usepackage{url}
\usepackage{multirow}
\usepackage{amsmath}
\usepackage{tablefootnote}
\usepackage[section]{placeins}

\graphicspath{ {figures/}}

\title{Actor-Mimic\\ Deep Multitask and Transfer Reinforcement Learning}

\author{
Emilio Parisotto, Jimmy Ba, Ruslan Salakhutdinov\\
Department of Computer Science\\
University of Toronto\\
Toronto, Ontario, Canada \\
\texttt{\{eparisotto,jimmy,rsalakhu\}@cs.toronto.edu} 
}

\iclrfinalcopy 

\usepackage{amsmath}
\usepackage{amsthm}
\usepackage{amsfonts}
\usepackage{amssymb}
\usepackage{graphicx}

\theoremstyle{definition}

\theoremstyle{plain}
\newtheorem{theorem}{Theorem}
\newtheorem{lemma}{Lemma}
\newtheorem{proposition}{Proposition}

\newcommand{\be}{\begin{equation}}
\newcommand{\ee}{\end{equation}}
\def\bea#1\eea{\begin{align}#1\end{align}}
\newcommand{\eqnarr}{\begin{eqnarray}}
\newcommand{\eqnend}{\end{eqnarray}}

\newcommand{\pderiv}[2]{\frac{\partial #1}{\partial #2}}

\DeclareMathSymbol{\umu}{\mathalpha}{operators}{0}
\newcommand{\expectation}{\mathop{\mathbb{E}}}
\newcommand{\stateSpace}{\mathcal{S}}
\newcommand{\state}{s}
\newcommand{\nextState}{s'}
\newcommand{\actionSpace}{\mathcal{A}}
\newcommand{\action}{a}
\newcommand{\actionTwo}{a'}
\newcommand{\nextAction}{a'}
\newcommand{\transProb}{\mathcal{T}}
\newcommand{\rewardFunc}{\mathcal{R}}
\newcommand{\reward}{r}
\newcommand{\disFactor}{\gamma}
\newcommand{\policy}{\pi}
\newcommand{\softTransform}{\Gamma}
\newcommand{\param}{\theta}
\newcommand{\replayMem}{\mathcal{M}}
\newcommand{\sourceGame}{S}
\newcommand{\expertNet}{E}
\newcommand{\actionSpaceExperti}{\mathcal{A}_{\expertNet_i}}
\newcommand{\amNet}{\text{AMN}}
\newcommand{\lossPolicy}{\mathcal{L}_{policy}}
\newcommand{\lossFeature}{\mathcal{L}_{FeatureRegression}}
\newcommand{\lossAM}{\mathcal{L}_{Actor Mimic}}
\newcommand{\featureAMNet}{h_\amNet}
\newcommand{\featureExpertNeti}{h_{\expertNet_i}}
\newcommand{\policyMat}{\Pi}
\newcommand{\stationaryProb}{D}
\newcommand{\transMat}{T}
\newcommand{\linearFeat}{\phi}
\newcommand{\egreedy}{\epsilon\text{-greedy}}
\newcommand{\softmaxMat}{P}

\begin{document}

\maketitle

\begin{abstract}
The ability to act in multiple environments and transfer previous knowledge to new situations can be considered a critical aspect of any intelligent agent. Towards this goal, we define a novel method of multitask and transfer learning that enables an autonomous agent to learn how to behave in multiple tasks simultaneously, and then generalize its knowledge to new domains. This method, termed ``Actor-Mimic'', exploits the use of deep reinforcement learning and model compression techniques to
train a single policy network that learns how to act in a set of distinct tasks by using the guidance of several expert teachers. We then show that the representations learnt by the deep policy network are capable of generalizing to new tasks with no prior expert guidance, speeding up learning in novel environments. Although our method can in general be applied to a wide range of problems, we use Atari games as a testing environment to demonstrate these methods.
\end{abstract}

\section{Introduction}

Deep Reinforcement Learning (DRL), the combination of reinforcement learning methods and deep neural network function approximators, has recently shown considerable success in high-dimensional challenging tasks,
such as robotic manipulation \citep{levine2015_visuomotor,lillicrap2015_continuous} and arcade games \citep{mnih2015_humanlevel}. These methods exploit the ability of deep networks to learn salient descriptions of raw state input, allowing the agent designer to essentially bypass the lengthy process of feature engineering. In addition, these automatically learnt descriptions often 
significantly outperform hand-crafted feature representations that require 
extensive domain knowledge. One such DRL approach, the Deep Q-Network (DQN) \citep{mnih2015_humanlevel}, has achieved state-of-the-art results on the Arcade Learning Environment (ALE) \citep{bellemare2013_ale}, a benchmark of Atari 2600 arcade games. The DQN uses a deep convolutional neural network over pixel inputs to parameterize a state-action value function. The DQN is trained using Q-learning combined with several tricks that stabilize the training of the network, such as a replay memory to store past transitions and target networks to define a more consistent temporal difference error.  

Although the DQN maintains the same network architecture and hyperparameters for all games, the approach is limited in the fact that each network only learns how to play a single game at a time, despite the existence of similarities between games. For example, the tennis-like game of pong and the squash-like game of breakout are both similar in that each game consists of trying to hit a moving ball with a rectangular paddle. 
A network trained to play multiple games would be able to generalize its knowledge between the games, achieving a single compact state representation as the inter-task similarities are exploited by the network. 
Having been trained on enough source tasks, the multitask network can also exhibit transfer to new target tasks, which can speed up learning. Training DRL agents can be extremely computationally intensive and therefore reducing training time is a significant practical benefit. 

The contribution of this paper is to develop and evaluate methods that enable multitask and transfer learning for DRL agents, using the ALE as a test environment. To first accomplish multitask learning, we design a method called ``Actor-Mimic'' that leverages techniques from model compression to train a single multitask network using guidance from a set of game-specific expert networks. The particular form of guidance can vary, and several different approaches are explored and tested empirically.
To then achieve transfer learning, we treat a multitask network as being a DQN which was pre-trained on a set of source tasks. 
We show experimentally that this multitask pre-training can result in a DQN that learns a target task significantly faster than a DQN starting from a random initialization, effectively demonstrating that the source task representations generalize to the target task.

\section{Background: Deep Reinforcement Learning}

A Markov Decision Process (MDP) is defined as a tuple ($\stateSpace$, $\actionSpace$, $\transProb$, $\rewardFunc$, $\disFactor$) where $\stateSpace$ is a set of states, $\actionSpace$ is a set of actions, $\transProb(s'|s,a)$ is the transition probability of ending up in state $s'$ when executing action $a$ in state $s$, $\mathcal{R}$ is the reward function mapping states in $\mathcal{S}$ to rewards in $\mathbb{R}$, and $\gamma$ is a discount factor. An agent's behaviour in an MDP is represented as a policy $\pi(a|s)$ which defines the probability of executing action $a$ in state $s$. For a given policy, we can further define the Q-value function $Q^{\pi}(s,a) = \mathop{\mathbb{E}}[\sum_{t=0}^{H}\gamma^t r_t | s_0=s, a_0=a]$ where $H$ is the step when the game ends. The Q-function represents the expected future discounted reward when starting in a state $s$, executing $a$, and then following policy $\pi$ until a terminating state is reached. There always exists at least one optimal state-action value function, $Q^{*}(s,a)$, such that $\forall s\in S,a\in A$, $Q^{*}(s,a) = \max_{\pi}Q^{\pi}(s,a)$ \citep{sutton1998_rl}. The optimal Q-function can be rewritten as a Bellman equation:

\begin{equation}
Q^{*}(\state,\action) = \mathop{\mathbb{E}}_{\nextState\sim\transProb(\cdot|\state,\action)}\left[\reward + \disFactor \cdot \max_{\nextAction \in \actionSpace} Q^{*}(\nextState, \nextAction)\right].
\end{equation}

An optimal policy can be constructed from the optimal Q-function by choosing, for a given state, the action with highest Q-value. Q-learning, a reinforcement learning algorithm, uses iterative backups of the Q-function to converge towards the optimal Q-function. Using a tabular representation of the Q-function, this is equivalent to setting $Q^{(n+1)}(\state,\action) = \mathop{\mathbb{E}}_{\nextState\sim\transProb(\cdot|\state,\action)}[\reward + \disFactor \cdot \max_{\nextAction \in \actionSpace} Q^{(n)}(\nextState,\nextAction)]$ for the (n+1)th update step \citep{sutton1998_rl}. Because the state space in the ALE is too large to tractably store a tabular representation of the Q-function, the Deep Q-Network (DQN) approach uses a deep function approximator to represent the state-action value function \citep{mnih2015_humanlevel}. To train a DQN on the (n+1)th step, we set the network's loss to

\begin{equation}
L^{(n+1)}(\param^{(n+1)}) = \mathop{\mathbb{E}}_{\state,\action,\reward,\nextState\sim\mathcal{\replayMem(\cdot)}}\left[\left(\reward + \disFactor \cdot \max_{\nextAction \in \actionSpace} Q(\nextState,\nextAction;\param^{(n)}) - Q(\state,\action;\param^{(n+1)})\right)^2\right],
\end{equation}

where $\replayMem(\cdot)$ is a uniform probability distribution over a replay memory, which is a set of the m previous $(\state,\action,\reward,\nextState)$ transition tuples seen during play, where m is the size of the memory. The replay memory is used to reduce correlations between adjacent states and is shown to have large effect on the stability of training the network in some games.

\section{Actor-Mimic}

\subsection{Policy Regression Objective}

Given a set of source games ${\sourceGame_1,...,\sourceGame_N}$, our first goal
is to obtain a single multitask policy network that can play any source game at
as near an expert level as possible.  
To train this multitask policy network,
we use guidance from a set of expert DQN networks
${\expertNet_1,...,\expertNet_N}$, where $\expertNet_i$ is an expert
specialized in source task $\sourceGame_i$. 
One possible definition of ``guidance'' would be to define a squared loss that would match Q-values between the student network and the experts. As the range of the expert value functions could vary widely between games, we found it difficult to directly distill knowledge from the expert value functions. The alternative we develop here is to instead match policies by first transforming Q-values using a softmax. Using the softmax gives us outputs which are bounded in the unit interval and so the effects of the different scales of each expert's Q-function are diminished, achieving higher stability during learning.
Intuitively, we can view using the softmax from the perspective of forcing the student to focus more on mimicking the action chosen by the guiding expert at each state, where the exact values of the state are less important. 
We call this method ``Actor-Mimic'' as it is an actor, i.e. policy, that mimics the decisions of a set of experts. In
particular, our technique first transforms each expert DQN into a policy
network by a Boltzmann distribution defined over the Q-value outputs,
\begin{equation}
\policy_{\expertNet_i}(\action|\state) = \frac{e^{\tau^{-1}Q_{\expertNet_i}(\state,\action)}}{\sum\limits_{\actionTwo\in\actionSpaceExperti} e^{\tau^{-1}Q_{\expertNet_i}(\state,\actionTwo)}},
\end{equation}
where $\tau$ is a temperature parameter and $\actionSpaceExperti$ is the action space used by the expert $\expertNet_i$, $\actionSpaceExperti \subseteq \actionSpace$. 
Given a state $s$ from source task $S_i$, we then define the policy objective over the multitask network as the cross-entropy between the expert network's policy and the current multitask policy: 
\begin{equation}
\label{eq:policy}
\lossPolicy^i(\param) = \sum_{\action\in\actionSpaceExperti}\policy_{\expertNet_i}(a|s)\log\policy_{\amNet}(\action|\state;\param),
\end{equation}
where $\policy_{\amNet}(\action|\state;\param)$ is the multitask Actor-Mimic Network (AMN) policy, parameterized by $\param$. In contrast to the Q-learning objective which recursively relies on itself as a target value, we now have a stable supervised training signal (the expert network output) to guide the multitask network. 

To acquire training data, we can sample either the expert network or the AMN action outputs to generate the trajectories used in the loss. Empirically we have observed that sampling from the AMN while it is learning gives the best results. We later prove that in either case of sampling from the expert or AMN as it is learning, the AMN will converge to the expert policy using the policy regression loss, at least in the case when the AMN is a linear function approximator. We use an $\epsilon$-greedy policy no matter which network we sample actions from, which with probability $\epsilon$ picks a random action uniformly and with probability $1 - \epsilon$ chooses an action from the network.

\subsection{Feature Regression Objective}

We can obtain further guidance from the expert networks in the following way. Let $\featureAMNet(\state)$ and $\featureExpertNeti(\state)$ be the hidden activations in the feature (pre-output) layer of the AMN and i'th expert network computed from the input state $\state$, respectively. Note that the dimension of $\featureAMNet(\state)$ does not necessarily need to be equal to $\featureExpertNeti(\state)$, and this is the case in some of our experiments. We define a feature regression network $f_i(\featureAMNet(\state))$ that, for a given state $\state$, attempts to predict the features $\featureExpertNeti(\state)$ from $\featureAMNet(\state)$. The architecture of the 
mapping $f_i$ can be defined arbitrarily, and $f_i$ can be trained using the following feature regression loss: 
\begin{equation}
\lossFeature^i(\param, \param_{f_i}) = \left\|f_i(\featureAMNet(\state;\param);\param_{f_i}) - \featureExpertNeti(\state)\right\|^2_2,
\end{equation}
where $\param$ and $\param_{f_i}$ are the parameters of the AMN and $i^{th}$ feature regression network, respectively. When training this objective, the error is fully back-propagated from the feature regression network output through the layers of the AMN. In this way, the feature regression objective provides pressure on the AMN to compute features that can predict an expert's features. A justification for this objective is that if we have a perfect regression from multitask to expert features, all the information in the expert features is contained in the multitask features. The use of the separate feature prediction network $f_i$ for each task enables the multitask network to have a different feature dimension than the experts as well as prevent issues with identifiability. Empirically we have found that the feature regression objective's primary benefit is that it can increase the performance of transfer learning in some target tasks.

\subsection{Actor-Mimic Objective}

Combining both regression objectives, the Actor-Mimic objective is thus defined as
\begin{equation}
\lossAM^i(\param, \param_{f_i}) = \lossPolicy^i(\param) + \beta * \lossFeature^i(\param, \param_{f_i}),
\end{equation}
where $\beta$ is a scaling parameter which controls the relative weighting of the two objectives. Intuitively, we can think of the policy regression objective as a teacher (expert network) telling a student (AMN) how they should act (mimic expert's actions), while the feature regression objective is analogous to a teacher telling a student why it should act that way (mimic expert's thinking process). 

\subsection{Transfering Knowledge: Actor-Mimic as Pretraining}

Now that we have a method of training a network that is an expert at all source tasks, we can proceed to the task of transferring source task knowledge to a novel but related target task. To enable transfer to a new task, we first remove the final softmax layer of the AMN. We then use the weights of AMN as an instantiation for a DQN that will be trained on the new target task. The pretrained DQN is then trained using the same training procedure as the one used with a standard DQN.
Multitask pretraining can be seen as initializing the DQN with a set of features that are effective at defining policies in related tasks. If the source and target tasks share similarities, it is probable that some of these pretrained features will also be effective at the target task (perhaps after slight fine-tuning).

\section{Convergence properties of Actor-Mimic}
\vspace{-0.1in}
We further study the convergence properties of the proposed Actor-Mimic under a framework similar to \citep{perkins2002convergent}. The analysis mainly focuses on L2-regularized policy regression without feature regression. Without losing generality, the following analysis focuses on learning from a single game expert softmax policy $\policy_\expertNet$. The analysis can be readily extended to consider multiple experts on multiple games by absorbing different games into the same state space. Let $\stationaryProb^{\policy}(\state)$ be the stationary distribution of the Markov decision process under policy $\policy$ over states $\state \in \stateSpace$. The policy regression objective function can be rewritten using expectation under the stationary distribution of the Markov decision process:
\be
\min_\param  \expectation_{\state\sim \stationaryProb^{\policy_{\amNet,\epsilon\text{-greedy}}}(\cdot)}\bigg[\mathcal{H}\bigg(\policy_{\expertNet}(\action|\state),\  \policy_\amNet(\action|\state;\param)\bigg)\bigg] + \lambda \|\param\|^2_2,
\label{eq:ev_obj}
\ee
where $\mathcal{H(\cdot)}$ is the cross-entropy measure and $\lambda$ is the coefficient of weight decay that is necessary in the following analysis of the policy regression. 
Under Actor-Mimic, the learning agent interacts with the environment by following an $\epsilon$-greedy strategy of some Q function. The mapping from a Q function to an $\epsilon$-greedy policy $\policy_{\epsilon\text{-greedy}}$ is denoted by an operator $\softTransform$, where $\policy_{\epsilon\text{-greedy}} = \softTransform(Q)$.
 To avoid confusion onwards, we use notation $p(\action|\state;\param)$ for the softmax policies in the policy regression objective. 

Assume each state in a Markov decision process is represented by a compact $K$-dimensional feature representation $\linearFeat(\state) \in \mathbb{R}^K$. Consider a linear function approximator for Q values with parameter matrix $\param \in \mathbb{R}^{K \times |\actionSpace|}$, $\hat{Q}(\state,\action;\param) = \linearFeat(\state)^T\param_\action$, where $\param_\action$ is the $\action^{th}$ column of $\param$. The corresponding softmax policy of the linear approximator is defined by $p(\action|\state;\param) \propto \text{exp}\{\hat{Q}(\state,\action;\param)\}$.  
\subsection{Stochastic stationary policy}
\vspace{-0.1in}
For any stationary policy $\pi^*$, the stationary point of the objective function Eq. (\ref{eq:ev_obj}) can be found by setting its gradient w.r.t. $\param$ to zero. Let $P_\param$ be a $|\stateSpace| \times |\actionSpace|$ matrix where its $i^{th}$ row $j^{th}$ column element is the softmax policy prediction $p(\action_j|\state_i;\param)$ from the linear approximator. Similarly, let $\policyMat_\expertNet$ be a $|\stateSpace| \times |\actionSpace|$ matrix for the softmax policy prediction from the expert model. Additionally, let $\stationaryProb_{\policy}$ be a diagonal matrix whose entries are $\stationaryProb_{\policy}(s)$. A simple gradient following algorithm on the objective function Eq. (\ref{eq:ev_obj}) has the following expected update rule using a learning rate $\alpha_t >0$ at the $t^{th}$ iteration:
\be
\Delta \param_t = -\alpha_t \bigg[ \Phi^TD_{\pi}(P_{\param_{t-1}} - \policyMat_\expertNet) + \lambda \param_{t-1} \bigg].
\label{eq:ev_grad}
\ee
\begin{lemma}
\label{lemma:fixed_policy}
Under a fixed policy $\policy^*$ and a learning rate schedule that satisfies $\sum_{t=1}^\infty \alpha_t = \infty$, $\sum_{t=1}^\infty \alpha^2_t < \infty$, the parameters $\param$, updated by the stochastic gradient descent learning algorithm described above, asymptotically almost surely converge to a unique solution $\param^*$. 
\end{lemma}
When the policy $\pi^*$ is fixed, the objective function Eq. (\ref{eq:ev_obj}) is convex and is the same as a multinomial logistic regression problem with a bounded Lipschitz constant due to its compact input features. Hence there is a unique stationary point $\param^*$ such that $\Delta \param^* = 0$. The proof of Lemma \ref{lemma:fixed_policy} follows the stochastic approximation argument~\citep{RobbinsMonro}.

\subsection{Stochastic adaptive policy}
\vspace{-0.1in}
\label{sec:stochastic_adaptive_policy}
Consider the following learning scheme to adapt the agent's policy. The learning agent interacts with the environment and samples states by following a fixed $\egreedy$ policy $\policy'$. Given the samples and the expert prediction, the linear function approximator parameters are updated using Eq. (\ref{eq:ev_grad}) to a unique stationary point $\param'$. The new parameters $\param'$ are then used to establish a new $\egreedy$ policy $\policy'' = \softTransform(\hat{Q}_{\param'})$ through the $\softTransform$ operator over the linear function $\hat{Q}_{\param'}$. The agent under the new policy $\policy''$ subsequently samples a new set of states and actions from the Markov decision process to update its parameters.  The learning agent therefore generates a sequence of policies $\{\policy^1, \policy^2, \policy^3, ...\}$. The proof for the following theorem is given in Appendix~\ref{app:convergence}.
\begin{theorem}
\label{thm:convergence}
Assume the Markov decision process is irreducible and aperiodic for any policy $\policy$ induced by the $\Gamma$ operator and $\Gamma$ is Lipschitz continuous with a constant $c_\epsilon$, then the sequence of policies and model parameters generated by the iterative algorithm above converges almost surely to a unique solution $\policy^*$ and $\param^*$. 
\end{theorem}
\vspace{-0.1in}
\subsection{Performance Guarantee}
\vspace{-0.1in}
The convergence theorem implies the Actor-Mimic learning algorithm also belongs to the family of no-regret algorithms in the online learning framework, see \cite{ross2011_dagger} for more details. Their theoretical analysis can be directly applied to Actor-Mimic and results in a performance guarantee bound on how well the Actor-Mimic model performs with respect to the guiding expert. 

Let $Z^{\pi'}_t (s,\pi)$ be the t-step reward of executing $\pi$ in the initial state $s$ and then following policy $\pi'$. The cost-to-go for a policy $\pi$ after $T$-steps is defined as $J_T(\pi) = -T\expectation_{\state \sim \stationaryProb(\cdot)}\left[\rewardFunc(\state,\action)\right]$, where $\rewardFunc(\state, \action)$ is the reward after executing action $\action$ in state $\state$.
\begin{proposition}
For the iterative algorithm described in Section (\ref{sec:stochastic_adaptive_policy}), if the loss function in Eq. (\ref{eq:ev_obj}) converges to $\epsilon$ with the solution $\policy_{\amNet}$ and $Z^{\pi^*}_{T-t+1} (s,\pi^*) - Z^{\pi^*}_{T-t+1} (s,a) \ge u$ for all actions $\action \in \actionSpace$ and $t \in \{1, \cdots, T\}$, then the cost-to-go of Actor-Mimic $J_T(\pi_{\amNet})$ grows linearly after executing $T$ actions: $J_T(\pi_{\amNet}) \le J_T(\pi_{\expertNet}) +  uT\epsilon/\log2$.
\end{proposition}
The above linear growth rate of the cost-to-go is achieved through sampling from AMN action output $\policy_{\amNet}$, while the cost grows quadratically if the algorithm only samples from the expert action output. Our empirical observations confirm this theoretical prediction. 

\vspace{-0.1in}

\section{Experiments}
\vspace{-0.1in}
In the following experiments, we validate the Actor-Mimic method by demonstrating its effectiveness at both multitask and transfer learning in the 
Arcade Learning Environment (ALE). For our experiments, we use subsets of a collection of 20 Atari games. 19 games of this set were among the 29 games that the DQN method performed at a super-human level. 
We additionally chose 1 game, the game of Seaquest, on which the DQN had performed poorly when compared to a human expert. Details on the training procedure are described in Appendix~\ref{app:trainingdetails}.

\vspace{-0.1in}
\subsection{Multitask}
\vspace{-0.1in}
\label{Actor_Mimic_exp}

\begin{figure}[t!]
\vspace{-0.5in}
\label{Actor_Mimic_results_small}
\begin{tabular}{cccc}
\includegraphics[width=0.225\linewidth]{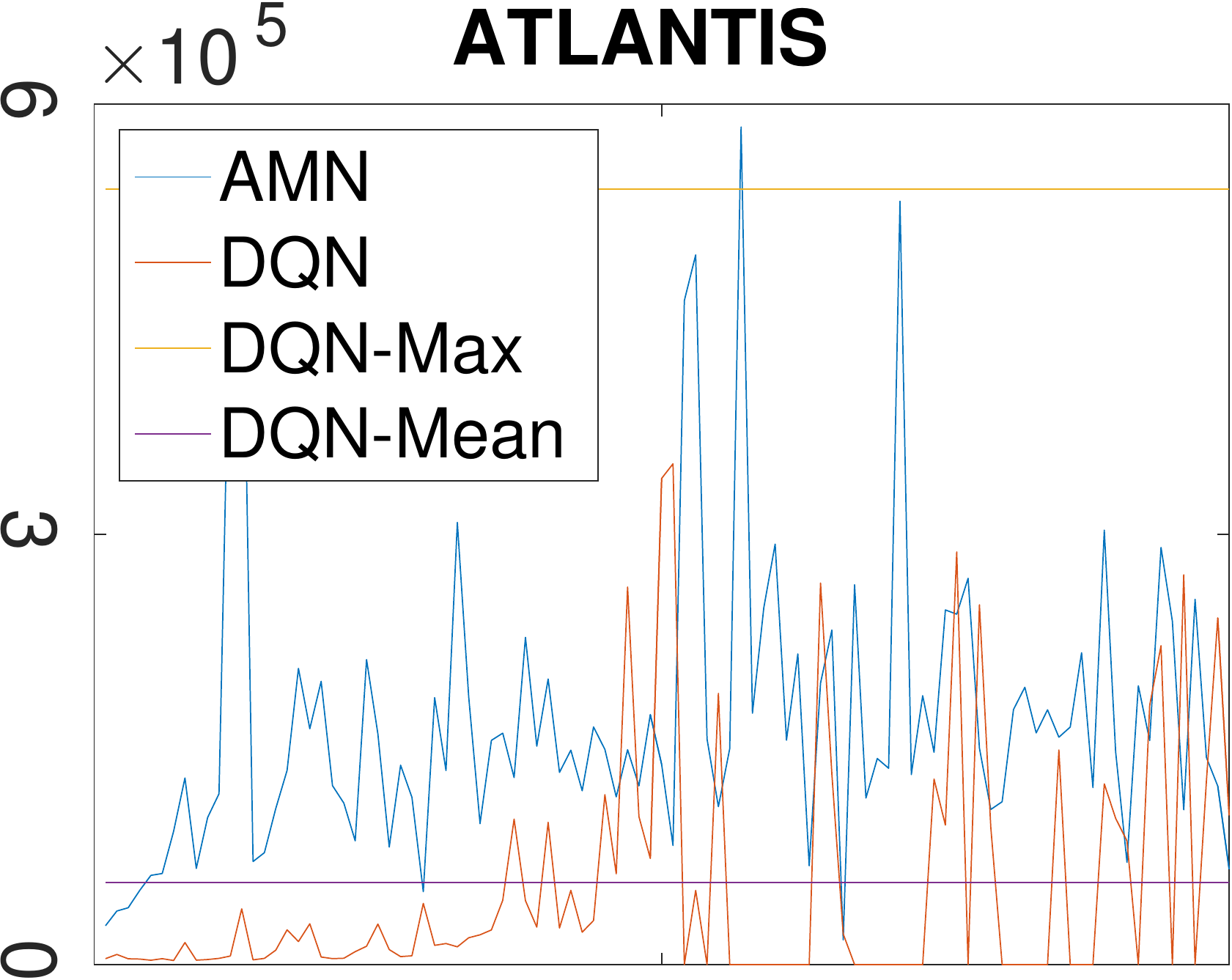} &
\includegraphics[width=0.22\linewidth]{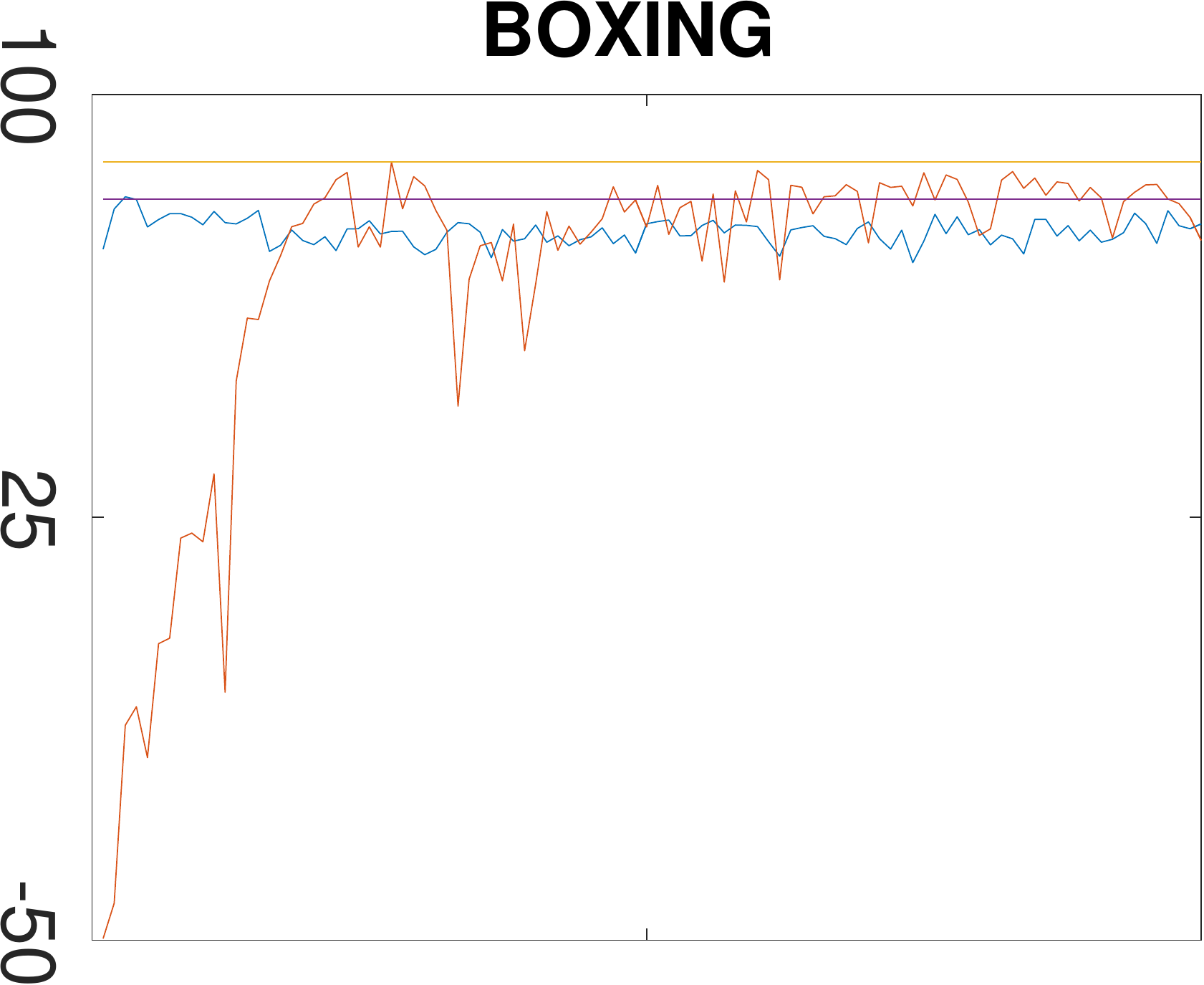} &
\includegraphics[width=0.225\linewidth]{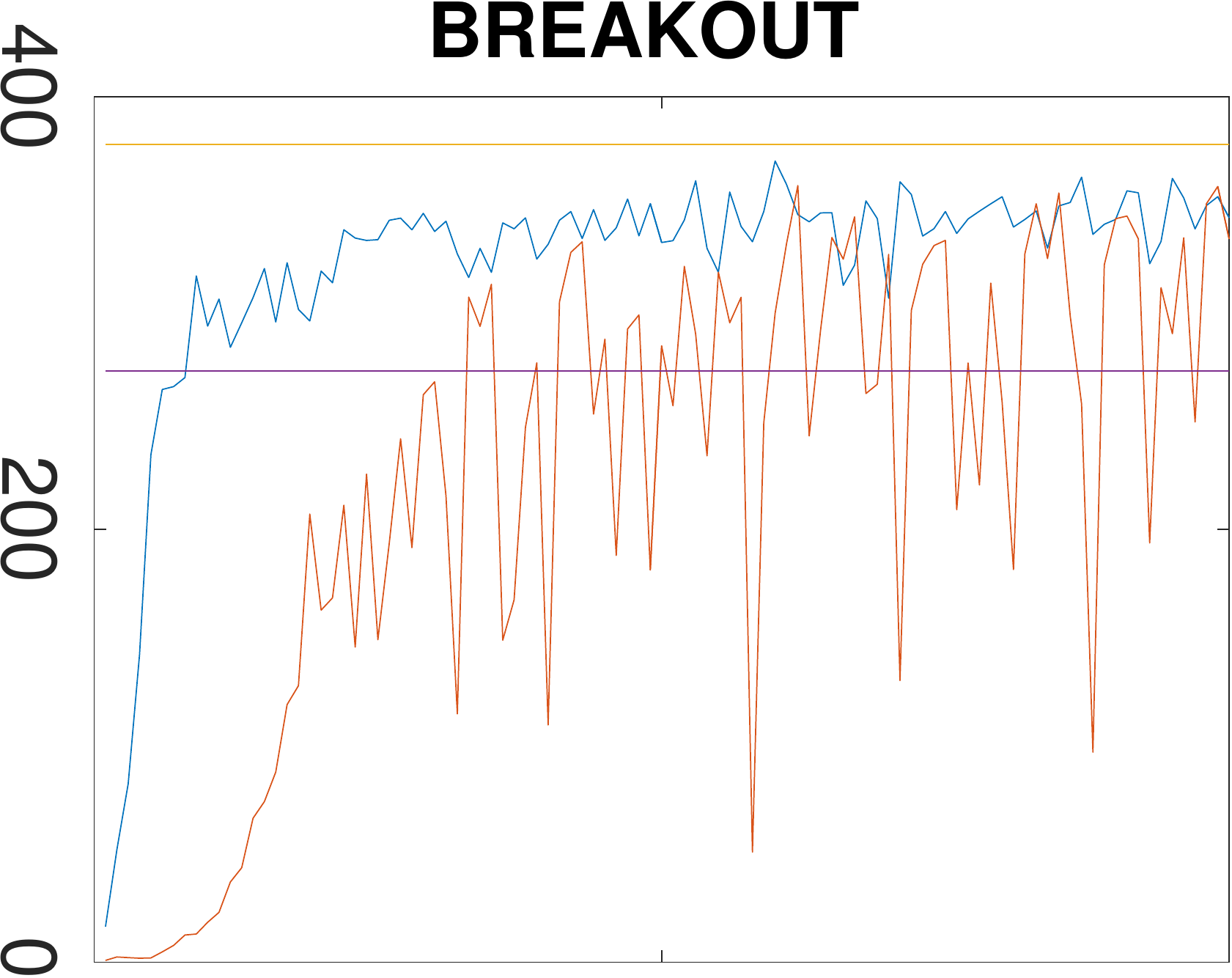} &
\includegraphics[width=0.225\linewidth]{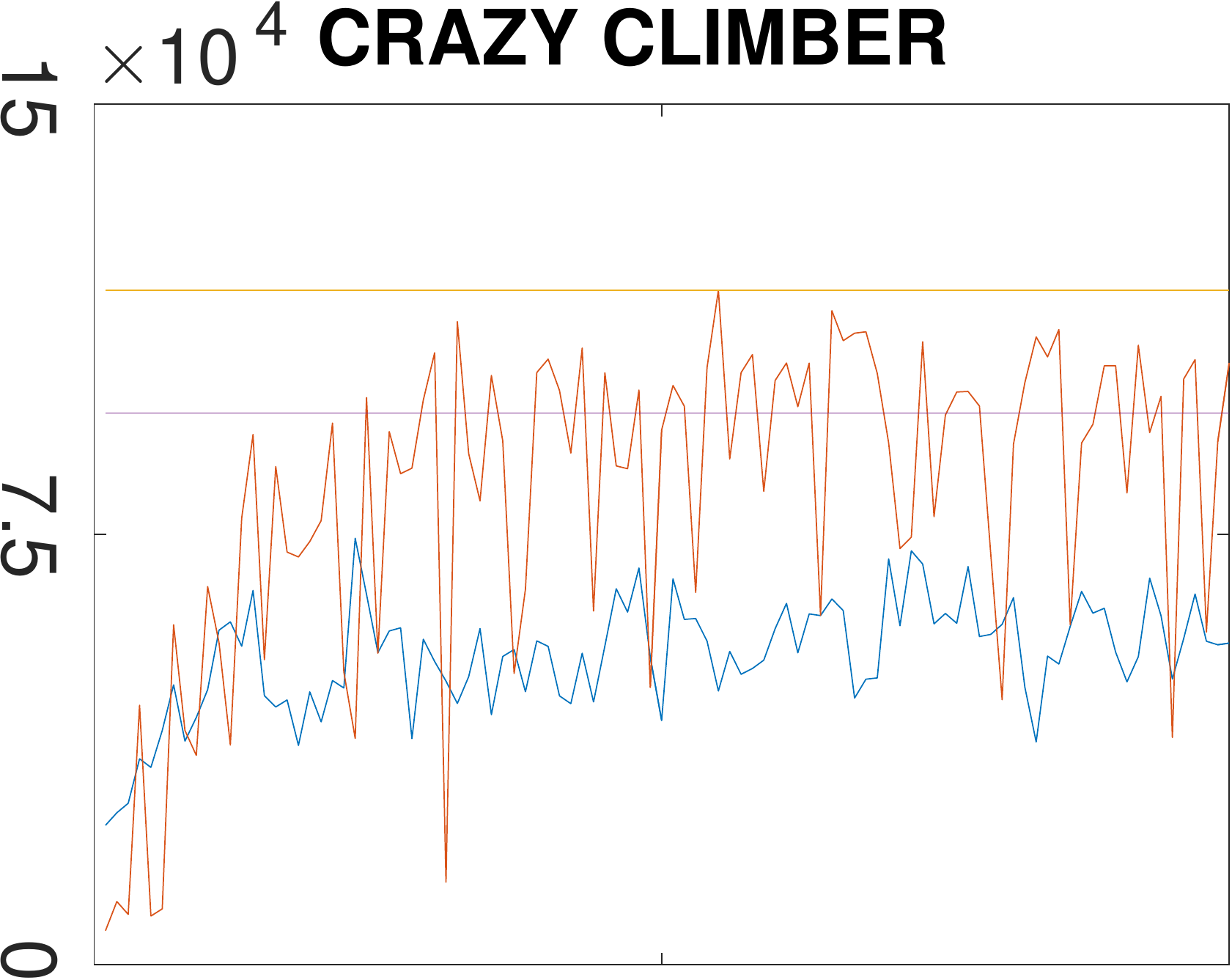} \\ 
\includegraphics[width=0.225\linewidth]{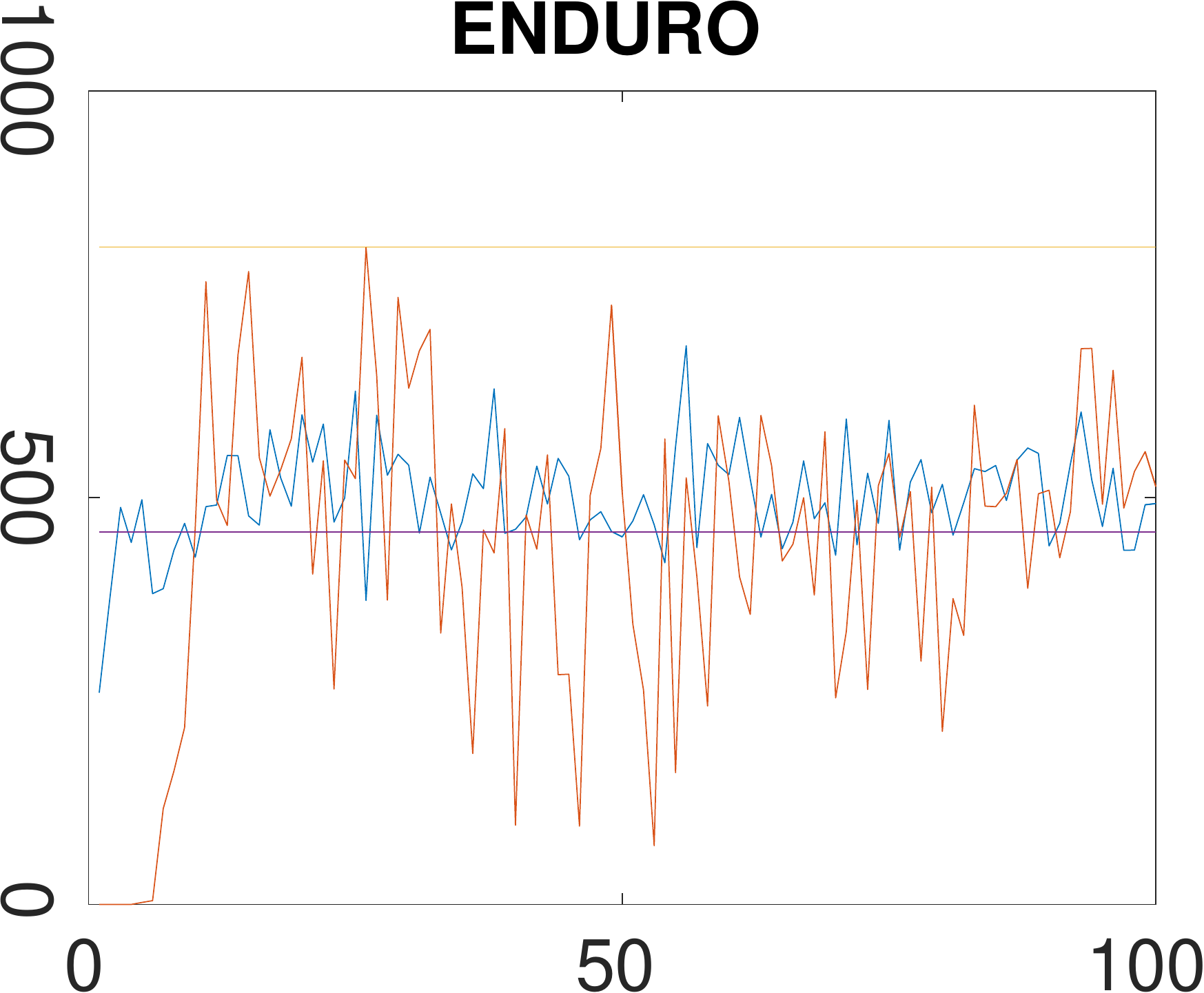} & 
\includegraphics[width=0.225\linewidth]{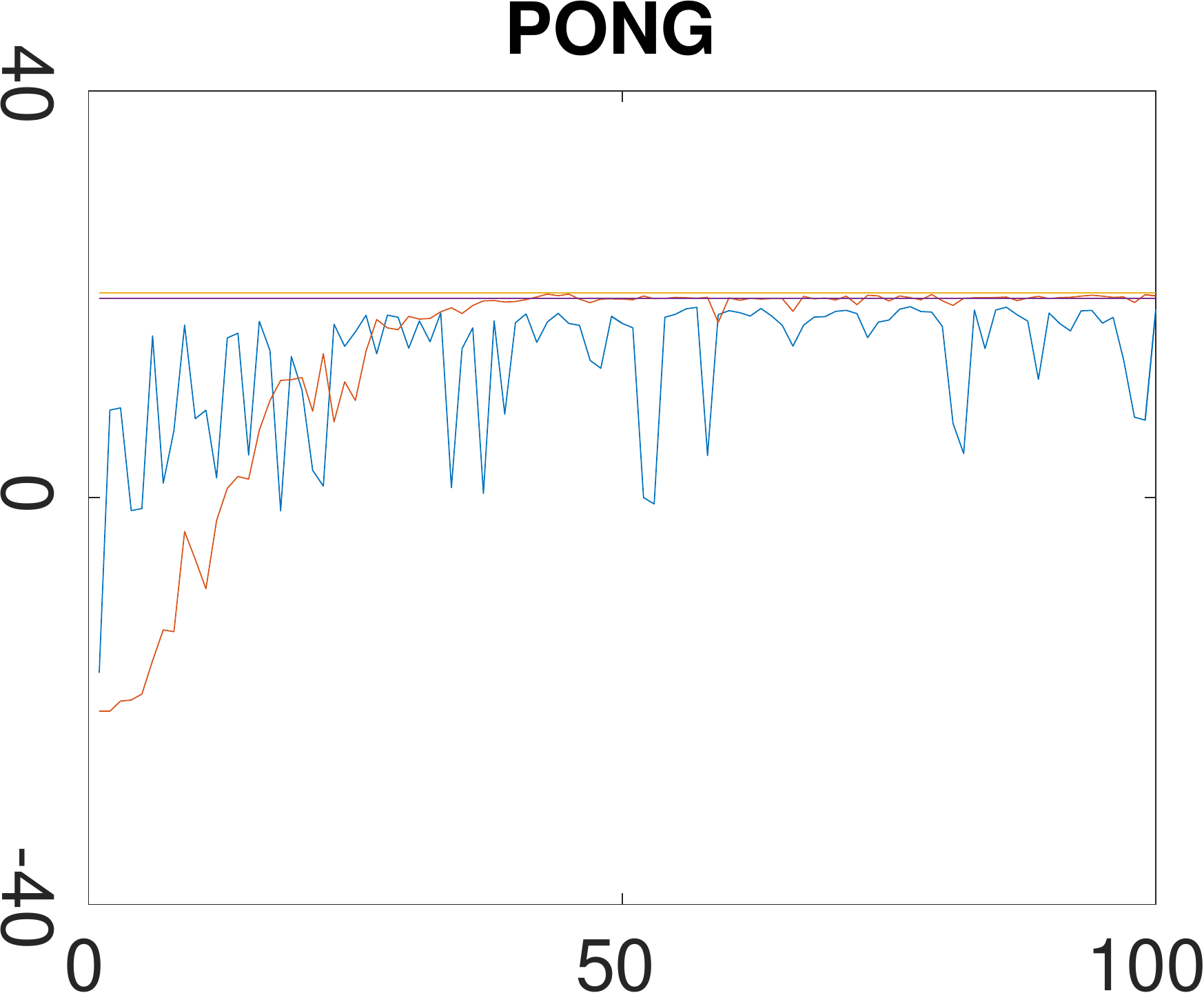} &
\includegraphics[width=0.225\linewidth]{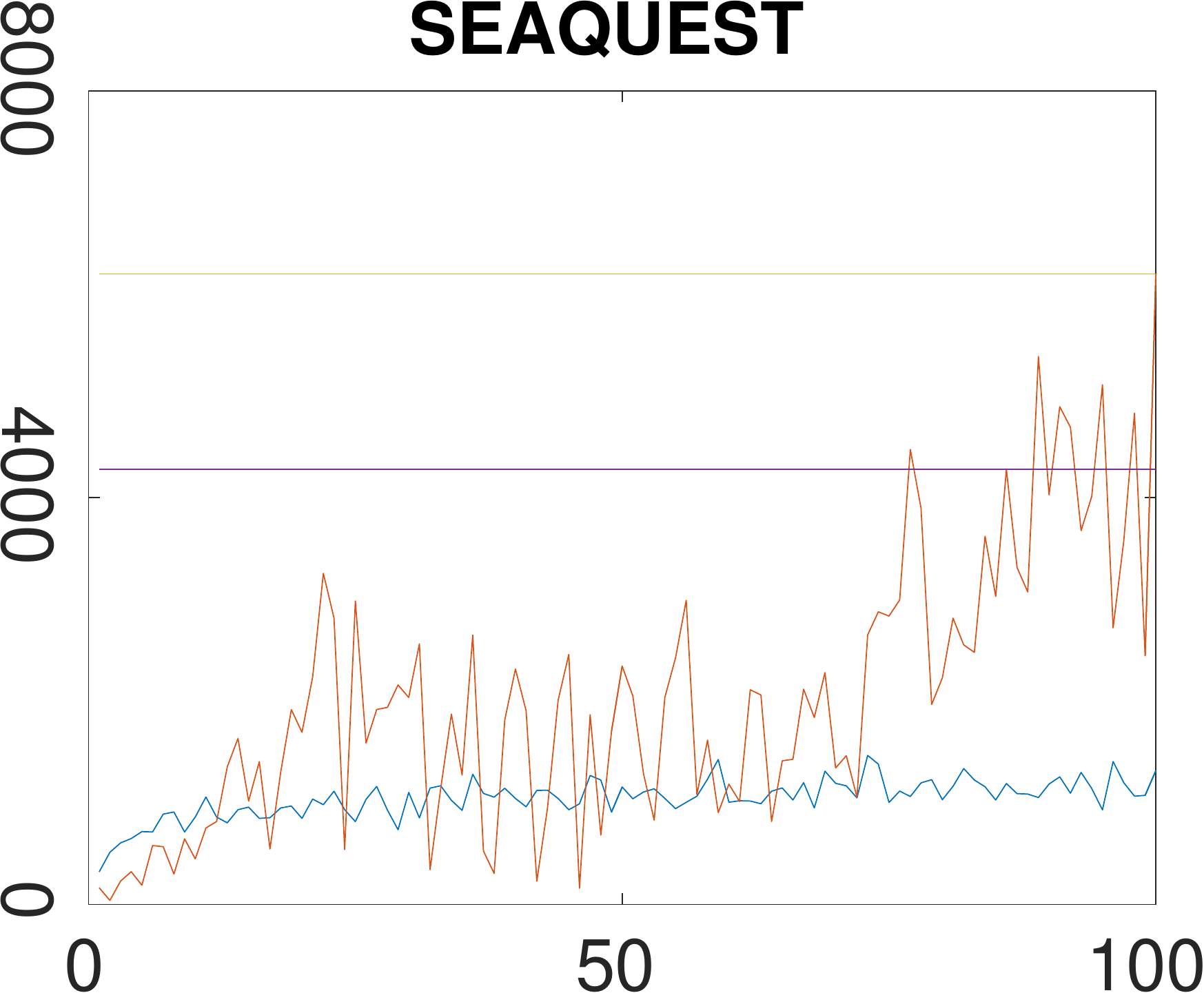} & 
\includegraphics[width=0.225\linewidth]{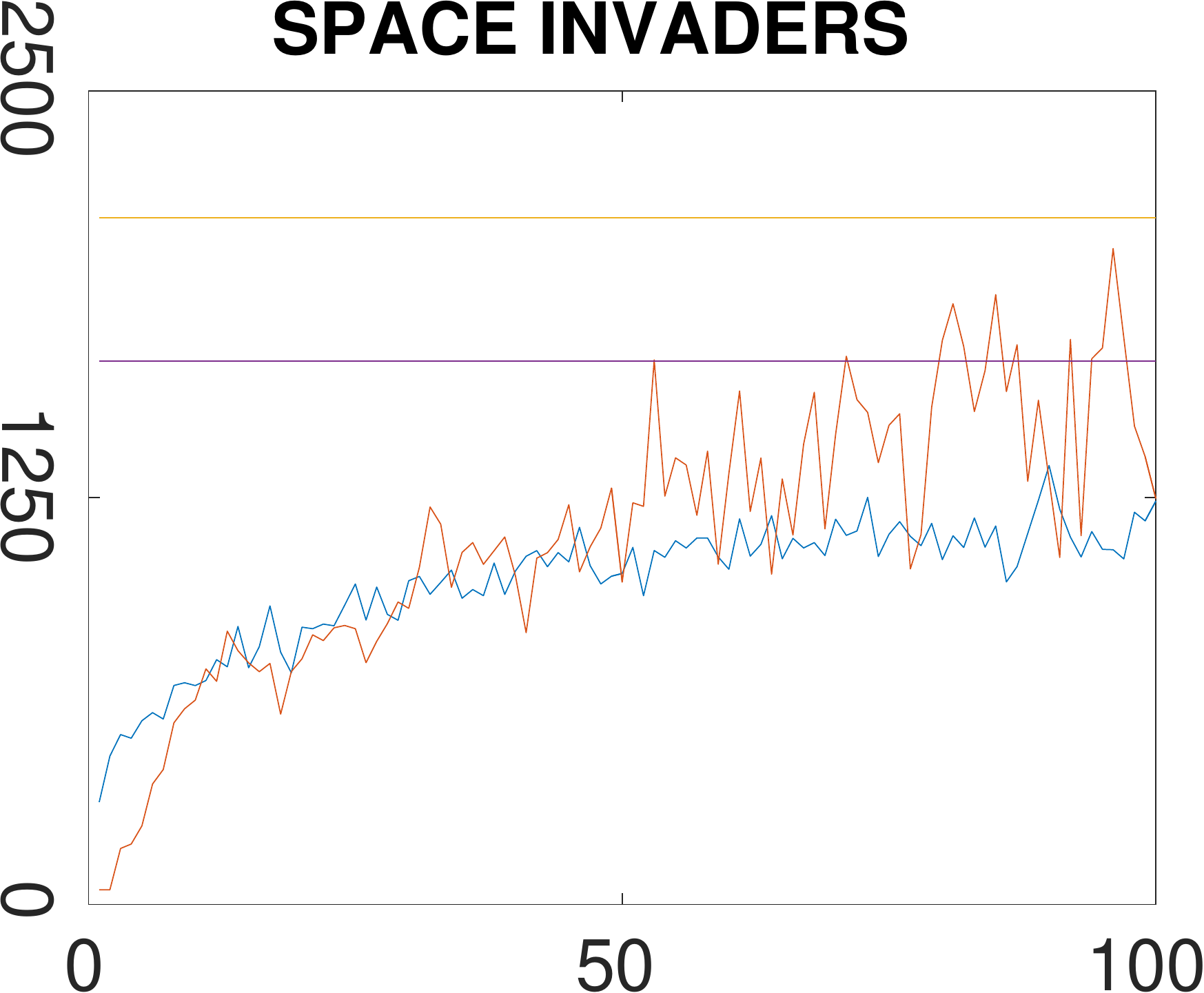} 
\end{tabular}
\vspace{-0.1in}
\caption{\small The Actor-Mimic and expert DQN training curves for 100 training epochs for each of the 8 games. A training epoch is 250,000 frames and for each training epoch we evaluate the networks with a testing epoch that lasts 125,000 frames.  We report AMN and expert DQN test reward for each testing epoch and the mean and max of DQN performance. The max is calculated over all testing epochs that the DQN experienced until convergence while the mean is calculated over the last ten epochs before the DQN training was stopped. In the testing epoch we use $\epsilon = 0.05$ in the $\epsilon$-greedy policy. The y-axis is the average unscaled episode reward during a testing epoch. The AMN results are averaged over 2 separately trained networks.}
\label{AMN_8games_figure}
\vspace{-0.2in}
\end{figure}

\begin{table}[t!]
\vspace{-0.5in}
\centering
\resizebox{\linewidth}{!}{%
\begin{tabular}{| c c | c | c | c | c | c | c | c | c |}
\hline
\multicolumn{2}{|c|}{Network} & Atlantis & Boxing & Breakout & Crazy Climber & Enduro & Pong & Seaquest & Space Invaders \\
\hline
\multirow{2}{*}{DQN} & Mean & 57279  & 81.47 & 273.15 & 96189  & 457.60 & 19.581 & 4278.9 & 1669.2 \\
                     & Max  & 541000 & 88.02 & 377.96 & 117593 & 808.00 & 20.140 & 6200.5 & 2109.7 \\
\hline
\multirow{2}{*}{AMN} & Mean  & 165065 & 76.264 & 347.01 & 57070 & 499.3  & 15.275 & 1177.3 & 1142.4 \\
                     & Max   & 584196 & 81.860 & 370.32 & 74342 & 686.77 & 18.780 & 1466.0 & 1349.0 \\
                               
\hline
\multirow{2}{*}{$100\%\times\frac{\textrm{AMN}}{\textrm{DQN}}$} 
                     & Mean & 288.2\% & 93.61\% & 127.0\% & 59.33\% & 109.1\% & 78.01\% & 27.51\% & 68.44\% \\
                     & Max  & 108.0\% & 93.00\% & 97.98\% & 63.22\% & 85.00\% & 93.25\% & 23.64\% & 63.94\% \\

\hline	
\end{tabular}}
\caption{\small Actor-Mimic results on a set of eight Atari games. We compare the AMN performance to that of the expert DQNs trained separately on each game. The expert DQNs were trained until convergence and the AMN was trained for 100 training epochs, which is equivalent to 25 million input frames per source game. For the AMN, we report maximum test reward ever achieved in epochs 1-100 and mean test reward in epochs 91-100. For the DQN, we report maximum test reward ever achieved until convergence and mean test reward in the last 10 epochs of DQN training. Additionally, at the last row of the table we report the percentage ratio of the AMN reward to the expert DQN reward for every game for both mean and max rewards. These percentage ratios are plotted in Figure~\ref{fig:amn_relative_improvement}. The AMN results are averaged over 2 separately trained networks.}
\label{AMN_8games_table}
\vspace{-0.2in} 
\end{table}

To first evaluate the actor-mimic objective on multitask learning, we demonstrate the effectiveness of training an AMN over multiple games simultaneously. In this particular case, since our focus is on multitask learning and not transfer learning, we disregard the feature regression objective and set $\beta$ to 0. Figure~\ref{AMN_8games_figure} and Table~\ref{AMN_8games_table} show the results of an AMN trained on 8 games simultaneously with the policy regression objective, compared to an expert DQN trained separately for each game. The AMN and every individual expert DQN in this case had the exact same network architecture. We can see that the AMN quickly reaches close-to-expert performance on 7 games out of 8, only taking around 20 epochs or 5 million training frames to settle to a stable behaviour. This is in comparison to the expert networks, which were trained for up to 50 million frames.

One result that was observed during training is that the AMN often becomes more consistent in its behaviour than the expert DQN, with a noticeably lower reward variance in every game except Atlantis and Pong. 
Another surprising result is that the AMN achieves a significantly higher mean reward in the game of Atlantis and relatively higher mean reward in the games of Breakout and Enduro. This is despite the fact that the AMN is not being optimized to improve reward over the expert but just replicate the expert's behaviour.
We also observed this increase in source task performance again when we later on increased the AMN model complexity for the transfer experiments (see Atlantis experiments in Appendix~\ref{app:transfer}). 
The AMN had the worst performance on the game of Seaquest, which was a game on which the expert DQN itself did not do very well. It is possible that a low quality expert policy has difficulty teaching the AMN to even replicate its own (poor) behaviour. We compare the performance of our AMN against a baseline of two different multitask DQN architectures in Appendix~\ref{app:mtbaseline}.

\vspace{-0.1in}
\subsection{Transfer}
\vspace{-0.1in}

We have found that although a small AMN can learn how to behave at a close-to-expert level on multiple source tasks, a larger AMN can more easily transfer knowledge to target tasks after being trained on the source tasks. For the transfer experiments, we therefore significantly increased the AMN model complexity relative to that of an expert. Using a larger network architecture also allowed us to scale up to playing 13 source games at once (see Appendix~\ref{app:transfer} for source task performance using the larger AMNs). We additionally found that using an AMN trained for too long on the source tasks hurt transfer, as it is likely overfitting. Therefore for the transfer experiments, we train the AMN on only 4 million frames for each of the source games.

\newcommand{\tb}[1]{\textbf{#1}}

\begin{table}[t!]
\vspace{-0.5in}
\centering
\resizebox{\linewidth}{!}{%
\begin{tabular}{| c | c | c | c | c | c | c | c | c | c | c |}
\hline
Breakout & 1 mil & 2 mil & 3 mil & 4 mil & 5 mil & 6 mil & 7 mil & 8 mil & 9 mil & 10 mil \\
\hline
Random      & 1.182 & 5.278 & 29.13 & 102.3 & 202.8 & 212.8 & 252.9 & 211.8 & 243.5 & 258.7 \\
AMN-policy  & \tb{18.35} & 102.1 & \tb{216.0} & \tb{271.1} & \tb{308.6} & \tb{286.3} & \tb{284.6} & \tb{318.8} & \tb{281.6} & \tb{311.3} \\
AMN-feature & 16.23 & \tb{119.0} & 153.7 & 191.8 & 172.6 & 233.9 & 248.5 & 178.8 & 235.6 & 225.5 \\
\hline
\hline
Gopher      & 1 mil & 2 mil & 3 mil & 4 mil & 5 mil & 6 mil & 7 mil & 8 mil & 8 mil & 10 mil \\
\hline
Random      & 294.0 & 578.9 & 1360  & \tb{1540}  & \tb{1820} & 1133  & 633.0 & \tb{1306}  & \tb{1758} & \tb{1539}  \\
AMN-policy  & \tb{715.0} & 612.7 & \tb{1362}  & 924.8 & 1029 & \tb{1186}  & \tb{1081}  & 936.7 & 1251 & 1142  \\
AMN-feature & 636.2 & \tb{1110}  & 918.8 & 1073  & 1028 & 810.1 & 1008  & 868.8 & 1054 & 982.4 \\
\hline
\hline
Krull & 1 mil & 2 mil & 3 mil & 4 mil & 5 mil & 6 mil & 7 mil & 8 mil & 9 mil & 10 mil \\
\hline
Random      & 4302 & 6193 & 6576 & 7030 & 6754 & 5294 & 5949 & 5557 & 5366 & 6005 \\
AMN-policy  & \tb{5827} & \tb{7279} & 6838 & 6971 & 7277 & 7129 & 7854 & \tb{8012} & 7244 & \tb{7835} \\
AMN-feature & 5033 & 7256 & \tb{7008} & \tb{7582} & \tb{7665} & \tb{8016} & \tb{8133} & 6536 & \tb{7832} & 6923 \\
\hline
\hline
Road Runner & 1 mil & 2 mil & 3 mil & 4 mil & 5 mil & 6 mil & 7 mil & 8 mil & 9 mil & 10 mil \\
\hline
Random      & 327.5 & 988.1 & 16263 & \tb{27183} & \tb{26639} & \tb{29488} & 33197 & 27683 & 25235 & \tb{31647} \\
AMN-policy  & \tb{1561}  & 5119  & \tb{19483} & 22132 & 23391 & 23813 & \tb{34673} & \tb{33476} & \tb{31967} & 31416 \\
AMN-feature & 1349  & \tb{6659}  & 18074 & 16858 & 18099 & 22985 & 27023 & 24149 & 28225 & 23342 \\
\hline
\hline
Robotank & 1 mil & 2 mil & 3 mil & 4 mil & 5 mil & 6 mil & 7 mil & 8 mil & 9 mil & 10 mil \\
\hline
Random      & \tb{4.830} & \tb{6.965} & 9.825 & 13.22 & \tb{21.07} & \tb{22.54} & \tb{31.94} & \tb{29.80} & \tb{37.12} & \tb{34.04} \\
AMN-policy  & 3.502 & 4.522 & 11.03 & 9.215 & 16.89 & 17.31 & 18.66 & 20.58 & 23.58 & 23.02 \\
AMN-feature & 3.550 & 6.162 & \tb{13.94} & \tb{17.58} & 17.57 & 20.72 & 20.13 & 21.13 & 26.14 & 23.29 \\
\hline
\hline
Star Gunner & 1 mil & 2 mil & 3 mil & 4 mil & 5 mil & 6 mil & 7 mil & 8 mil & 9 mil & 10 mil \\
\hline
Random      & 221.2 & 468.5 & 927.6 & 1084 & 1508 & 1626 & 3286 & 16017 & 36273 & 45322 \\
AMN-policy  & 274.3 & 302.0 & 978.4 & 1667 & 4000 & 14655 & 31588 & 45667 & 38738 & 53642 \\
AMN-feature & \tb{1405} & \tb{4570} & \tb{18111} & \tb{23406} & \tb{36070} & \tb{46811} & \tb{50667} & \tb{49579} & \tb{50440} & \tb{56839} \\
\hline
\hline
Video Pinball & 1 mil & 2 mil & 3 mil & 4 mil & 5 mil & 6 mil & 7 mil & 8 mil & 9 mil & 10 mil \\
\hline
Random      & 2323 & 8549  & 6780  & 5842   & 10383 & 11093  & 8468   & 5476  & 9964  & 11893 \\
AMN-policy  & \tb{2583} & \tb{25821} & \tb{95949} & \tb{143729} & \tb{57114} & \tb{106873} & \tb{111074} & \tb{73523} & \tb{34908} & \tb{123337} \\
AMN-feature & 1593 & 3958  & 21341 & 12421  & 15409 & 18992  & 15920  & 48690 & 24366 & 26379 \\
\hline
\end{tabular}}
\vspace{-0.1in}
\caption{\small Actor-Mimic transfer results for a set of 7 games. The 3 networks are trained as DQNs on the target task, with the only difference being the weight initialization. ``Random'' means random initial weights, ``AMN-policy'' means a weight initialization with an AMN trained using policy regression and ``AMN-feature'' means a weight initialization with an AMN trained using both policy and feature regression (see text for more details). We report the average test reward every 4 training epochs (equivalent to 1 million training frames), where the average is over 4 testing epochs that are evaluated immediately after each training epoch. For each game, we bold out the network results that have the highest average testing reward for that particular column.
}
\label{AMN_transfer_table}
\vspace{-0.2in}
\end{table}

To evaluate the Actor-Mimic objective on transfer learning, the previously described large AMNs will be used as a weight initialization for DQNs which are each trained on a different target task. We additionally independently evaluate the benefit of the feature regression objective during transfer by having one AMN trained with only the policy regression objective (AMN-policy) and another trained using both feature and policy regression (AMN-feature). The results are then compared to the baseline of a DQN that was initialized with random weights. 

The performance on a set of 7 target games is detailed in Table~\ref{AMN_transfer_table} (learning curves are plotted in Figure~\ref{fig:transfer_learn_curve}). We can see that the AMN pretraining provides a definite increase in learning speed for the 3 games of Breakout, Star Gunner and Video Pinball. 
The results in Breakout and Video Pinball demonstrate that the policy regression objective alone provides significant positive transfer in some target tasks. The reason for this large positive transfer might be due to the source game Pong having very similar mechanics to both Video Pinball and Breakout, where one must use a paddle to prevent a ball from falling off screen. The machinery used to detect the ball in Pong would likely be useful in detecting the ball for these two target tasks, given some fine-tuning.
Additionally, the feature regression objective causes a significant speed-up in the game of Star Gunner compared to both the random initialization and the network trained solely with policy regression. Therefore even though the feature regression objective can slightly hurt transfer in some source games, it can provide large benefits in others. The positive transfer in Breakout, Star Gunner and Video Pinball saves at least up to 5 million frames of training time in each game. Processing 5 million frames with the large model is equivalent to around 4 days of compute time on a NVIDIA GTX Titan.

On the other hand, for the games of Krull and Road Runner
(although the multitask pretraining does help learning at the start)  
the effect is not very pronounced. When running Krull we observed that the policy learnt by any DQN regardless of the initialization was a sort of unexpected local maximum. In Krull, the objective is to move between a set of varied minigames and complete each one. One of the minigames, where the player must traverse a spiderweb, gives extremely high reward by simply jumping quickly in a mostly random fashion. What the DQN does is it kills itself on purpose in the initial minigame, runs to the high reward spiderweb minigame, and then simply jumps in the corner of the spiderweb until it is terminated by the spider. Because it is relatively easy to get stuck in this local maximum, and very hard to get out of it (jumping in the minigame gives unproportionally high reward compared to the other minigames), transfer does not really help learning.

For the games of Gopher and Robotank, we can see that the multitask pretraining does not have any significant positive effect. In particular, multitask pretraining for Robotank even seems to slow down learning, providing an example of negative transfer. The task in Robotank is to control a tank turret in a 3D environment to destroy other tanks, so it's possible that this game is so significantly different from any source task (being the only first-person 3D game) that the multitask pretraining does not provide any useful prior knowledge.

\vspace{-0.1in}

\vspace{-0.05in}
\section{Related Work}
\vspace{-0.1in}
The idea of using expert networks to guide a single mimic network has been studied in the context of supervised learning, where it is known as model compression. The goal of model compression is to reduce the computational complexity of a large model (or ensemble of large models) to a single smaller mimic network while maintaining as high an accuracy as possible. To obtain high accuracy, the mimic network is trained using rich output targets provided by the experts. These output targets are either the final layer logits \citep{ba2014_compression} or the high-temperature softmax outputs of the experts \citep{hinton2015_distilling}. Our approach is most similar to the technique of
\citep{hinton2015_distilling}
which matches the high-temperature outputs of the mimic network with that of the expert network.
In addition, we also tried an objective that provides expert guidance at the feature level instead of only at the output level. A similar idea was also explored in the model compression case \citep{romero2014_fitnets}, where a deep and thin mimic network used a larger expert network's intermediate features as guiding hints during training. 
In contrast to these model compression techniques, our method is not concerned with decreasing test time computation but instead using experts to provide otherwise unavailable supervision to a mimic network on several distinct tasks. 

Actor-Mimic can also be considered as part of the larger Imitation Learning class of methods, which use expert guidance to teach an agent how to act. One such method, called DAGGER \citep{ross2011_dagger}, is similar to our approach in that it trains a policy to directly mimic an expert's behaviour while sampling actions from the mimic agent. Actor-Mimic can be considered as an extension of this work to the multitask case. In addition, using a deep neural network to parameterize the policy provides us with several advantages over the more general Imitation Learning framework. First, we can exploit the automatic feature construction ability of deep networks to transfer knowledge to new tasks, as long as the raw data between tasks is in the same form, i.e. pixel data with the same dimensions. Second, we can define objectives which take into account intermediate representations of the state and not just the policy outputs, for example the feature regression objective
which provides a richer training signal to the mimic network than just samples of the expert's action output.

Recent work has explored combining expert-guided Imitation Learning and deep neural networks in the single-task case. \cite{guo2014_mcts} use DAGGER with expert guidance provided by Monte-Carlo Tree Search (MCTS) policies to train a deep neural network that improves on the original DQN's performance. Some disadvantages of using MCTS experts as guidance are that they require both access to the (hidden) RAM state of the emulator as well as an environment model. 
Another related method is that of guided policy search \citep{levine2013_gps}, which combines a regularized importance-sampled policy gradient with guiding trajectory samples generated using differential dynamic programming. The goal in that work was to learn continuous control policies which improved upon the basic policy gradient method, which is prone to poor local minima.

A wide variety of methods have also been studied in the context of RL transfer learning (see
\cite{taylor2009_transfer} for a more comprehensive review). 
One related approach is to use a dual state representation with a set of task-specific and task-independent features known as ``problem-space'' and ``agent-space'' descriptors, respectively. For each source task, a task-specific value function is learnt on the problem-space descriptors and then these learnt value functions are transferred to a single value function over the agent-space descriptors. Because the agent-space value function is defined over features which maintain constant semantics across all tasks, this value function can be directly transferred to new tasks. 
\cite{banerjee2007_gametransfer} constructed agent-space features by first generating a fixed-depth game tree of the current state, classifying each future state in the tree as either $\{win, lose, draw, nonterminal\}$ and then coalescing all states which have the same class or subtree. To transfer the source tasks value functions to agent-space, they use a simple weighted average of the source task value functions, where the weight is proportional to the number of times that a specific agent-space descriptor has been seen during play in that source task. In a related method, \cite{konidaris2006_agentspace} transfer the value function to agent-space by using regression to predict every source task’s problem-space value function from the agent-space descriptors.
A drawback of these methods is that the agent- and problem-space descriptors are either hand-engineered or generated from a perfect environment model, thus requiring a significant amount of domain knowledge.

\vspace{-0.2in}
\section{Discussion}
\vspace{-0.1in}
In this paper we defined Actor-Mimic, a novel method for training a single deep policy network over a set of related source tasks. We have shown that a network trained using Actor-Mimic is capable of reaching expert performance on many games simultaneously, while having the same model complexity as a single expert. In addition, using Actor-Mimic as a multitask pretraining phase can significantly improve learning speed in a set of target tasks. This demonstrates that the features learnt over the source tasks can generalize to new target tasks, given a sufficient level of similarity between source and target tasks. A direction of future work is to develop methods that can enable a targeted knowledge transfer from source tasks by identifying related source tasks for the given target task. Using targeted knowledge transfer can potentially help in cases of negative transfer observed in our experiments.

{\small
{\bf Acknowledgments:}
This work was supported by Samsung and NSERC.
}

\newpage

\bibliography{actormimic_ref}
\bibliographystyle{actormimic_ref}

\newpage

\begin{appendices}
\section{Proof of Theorem \ref{thm:convergence}}
\label{app:convergence}
\begin{lemma}
For any two policies $\pi^1$,$\pi^2$, the stationary distributions over the states under the policies are bounded: $\|D_{\pi^1} - D_{\pi^2}\| \le c_D \|\pi^1 - \pi^2\|$, for some $c_D > 0$.
\end{lemma}
\begin{proof} 
Let $\transMat^1$ and $\transMat^2$ be the two transition matrices under the stationary distributions $\stationaryProb_{\policy^1}$, $\stationaryProb_{\policy^2}$. For any $ij$ elements $\transMat^1_{ij}$, $\transMat^2_{ij}$ in the transition matrices, we have,
\bea
\|\transMat^1_{ij} - \transMat^2_{ij}\| 
=&\left\|\sum_a p(s_i|a,s_j)\left(\policy^1(a|s_j)-\policy^2(a|s_j)\right)\right\|\\
 \le& |\actionSpace| \|\policy^1(a|s_j)-\policy^2(a|s_j)\|\\
 \le& |\actionSpace| \|\policy^1-\policy^2\|_\infty.
\eea
The above bound for any $ij^{th}$ elements implies the Euclidean distance of the transition matrices is also upper bounded $\|\transMat^1 - \transMat^2\|\le|\stateSpace||\actionSpace| \|\policy^1-\policy^2\|$. \cite{seneta1991sensitivity} has shown that $\|\stationaryProb_{\policy^1} - \stationaryProb_{\policy^2}\| \le \frac{1}{1-\lambda^1}\|\transMat^1 - \transMat^2\|_\infty$, where $\lambda^1$ is the largest eigenvalue of $\transMat^1$. Hence, there is a constant $c_D > 0$ such that $\|\stationaryProb_{\policy^1} - \stationaryProb_{\policy^2}\| \le c_D \|\policy^1 - \policy^2\|$.
\end{proof}

\begin{lemma}
For any two softmax policy $P_{\param^1}$, $P_{\param^2}$ matrices from the linear function approximator, $\|P_{\param^1} - P_{\param^2}\| \le c_J\| \Phi \param^1 - \Phi \param^2\|$, for some $c_J\ge0$. 
\end{lemma}
\begin{proof}
Note that the $i^{th}$ row $j^{th}$ column element $p(\action_j|\state_i)$ in a softmax policy matrix $\softmaxMat$ is computed by the softmax transformation on the Q function:
\be
p_{ij}=p(\action_j|\state_i) = softmax\bigg(Q(\state_i,\action_j)\bigg) = \frac{e^{Q(\state_i,\action_j)}}{\sum\limits_{k} e^{Q(\state_i,\action_k)}}.
\ee
Because the softmax function is a monotonically increasing element-wise function on matrices, the Euclidean distance of the softmax transformation is upper bounded by the largest Jacobian in the domain of the softmax function. Namely, for $c'_J = \max_{z\in\text{Dom } softmax}\|\pderiv{softmax(z)}{z}\|$,
\be
\|softmax(x^1) - softmax(x^2)\| \le c'_J\|x^1 - x^2\|, \forall x^1, x^2 \in \text{Dom } softmax.
\ee
By bounding the elements in $P$ matrix, it gives $\|P_{\param^1} - P_{\param^2}\| \le c_J\|\hat{Q}_{\param^1} - \hat{Q}_{\param^2}\| = c_J\| \Phi \param^1 - \Phi \param^2\|$.
\end{proof}

\setcounter{theorem}{0}
\begin{theorem}
Assume the Markov decision process is irreducible and aperiodic for any policy $\policy$ induced by the $\Gamma$ operator and $\Gamma$ is Lipschitz continuous with a constant $c_\epsilon$, the sequence of policies and model parameters generated by the iterative algorithm above converges almost surely to a unique solution $\policy^*$ and $\param^*$.
\end{theorem}
\begin{proof}
We follow a similar contraction argument made in 
\cite{perkins2002convergent}
, and show the iterative algorithm is a contraction process. Namely, for any two policies $\pi^1$ and $\pi^2$, the learning algorithm above produces new policies $\Gamma(\hat{Q}_{\param^1})$, $\Gamma(\hat{Q}_{\param^2})$ after one iteration, where $\| \Gamma(\hat{Q}_{\param^1}) - \Gamma(\hat{Q}_{\param^2})\| \le \beta\| \pi^1 - \pi^2 \|$. Here $\|\cdot\|$ is the Euclidean norm and $\beta \in [0, 1)$. 

By Lipschtiz continuity, 
\bea
\| \Gamma(\hat{Q}_{\param^1}) - \Gamma(\hat{Q}_{\param^2})\| \le &c_{\epsilon}\| \hat{Q}_{\param^1} - \hat{Q}_{\param^2}\|
= c_{\epsilon}\| \Phi{\param^1} - \Phi{\param^2}\| \\
\le &c_{\epsilon}\| \Phi \| \|{\param^1} - {\param^2}\|.
\eea
Let $\param^1$ and $\param^2$ be the stationary points of Eq. (\ref{eq:ev_obj}) under $\policy^1$ and $\policy^2$. That is, $\Delta \param^1 = \Delta \param^2 = 0$ respectively. Rearranging Eq. (\ref{eq:ev_grad}) gives,
\bea
\|\param^1 - \param^2\| = &\frac{1}{\lambda}\| \Phi^T D_{\pi^1} (P_{\param^1} - \Pi_e) - \Phi^T D_{\pi^2} (P_{\param^2} - \Pi_e) \| \\
= &\frac{1}{\lambda}\| \Phi^T (D_{\pi^2} - D_{\pi^1}) \Pi_e + \Phi^T D_{\pi^1}P_{\param^1} - \Phi^T D_{\pi^1}P_{\param^2} + \Phi^T D_{\pi^1}P_{\param^2} - \Phi^T D_{\pi^2}P_{\param^2} \| \\
= &\frac{1}{\lambda}\| \Phi^T (D_{\pi^2} - D_{\pi^1}) \Pi_e + \Phi^T D_{\pi^1}(P_{\param^1} - P_{\param^2}) + \Phi^T (D_{\pi^1} - D_{\pi^2})P_{\param^2} \| \\
\le &\frac{1}{\lambda}\bigg[\| \Phi^T\| \|D_{\pi^1} - D_{\pi^2}\| \|\Pi_e\| + \|\Phi^T\|\|D_{\pi^1}\|\|P_{\param^1} - P_{\param^2}\| + \|\Phi^T\|\|D_{\pi^1} - D_{\pi^2}\|\|P_{\param^2}\|\bigg] \\
\le & c\| \pi^1 - \pi^2\|.
\eea
The last inequality is given by Lemma 2 and 3 and the compactness of $\Phi$.  For a Lipschtiz constant $c_\epsilon \ge c$, there exists a $\beta$ such that $\| \Gamma(\hat{Q}_{\param^1}) - \Gamma(\hat{Q}_{\param^2})\| \le \beta\| \pi^1 - \pi^2 \|$. Hence, the sequence of policies generated by the algorithm converges almost surely to a unique fixed point $\pi^*$ from Lemma \ref{lemma:fixed_policy} and the Contraction Mapping Theorem 
\cite{bertsekas1995dynamic}.
Furthermore, the model parameters converge w.p. 1 to a stationary point $\param^*$ under the fixed point policy $\pi^*$.
\end{proof}

\end{appendices}

\begin{appendices}

\section{AMN training details}
\label{app:trainingdetails}

All of our Actor-Mimic Networks (AMNs) were trained using the
Adam \citep{kingma2015_adam} optimization algorithm.
The AMNs have a single 18-unit output, with each output corresponding to one of the 18 possible Atari player actions. Having the full 18-action output simplifies the multitask case when each game has a different subset of valid actions. While playing a certain game, we mask out AMN action outputs that are not valid for that game and take the softmax over only the subset of valid actions. We use a replay memory for each game to reduce correlations between successive frames and stabilize network training. Because the memory requirements of having the standard replay memory size of 1,000,000 frames for each game are prohibitive when we are training over many source games, for AMNs we use a per-game 100,000 frame replay memory. AMN training was stable even with only a per-game equivalent of a tenth of the replay memory size of the DQN experts. For the transfer experiments with the feature regression objective, we set the scaling parameter $\beta$ to 0.01 and the feature prediction network $f_i$ was set to a linear projection from the AMN features to the $i^{th}$ expert features. For the policy regression objective, we use a softmax temperature of 1 in all cases. Additionally, during training for all AMNs we use an $\epsilon$-greedy policy with $\epsilon$ set to a constant 0.1. Annealing $\epsilon$ from 1 did not provide any noticeable benefit. During training, we choose actions based on the AMN and not the expert DQN. We do not use weight decay during AMN training as we empirically found that it did not provide any large benefits.

For the experiments using the DQN algorithm, we optimize the networks with RMSProp.
Since the DQNs are trained on a single game their output layers only contain the player actions that are valid in the particular game that they are trained on. The experts guiding the AMNs used the same architecture, hyperparameters and training procedure as that of \cite{mnih2015_humanlevel}. We use the full 1,000,000 frame replay memory when training any DQN.

\end{appendices}


\begin{appendices}

\section{Multitask DQN baseline results}
\label{app:mtbaseline}

\begin{figure}[t!]
\vspace{-0.5in}
\begin{tabular}{cccc}
\includegraphics[width=0.225\linewidth]{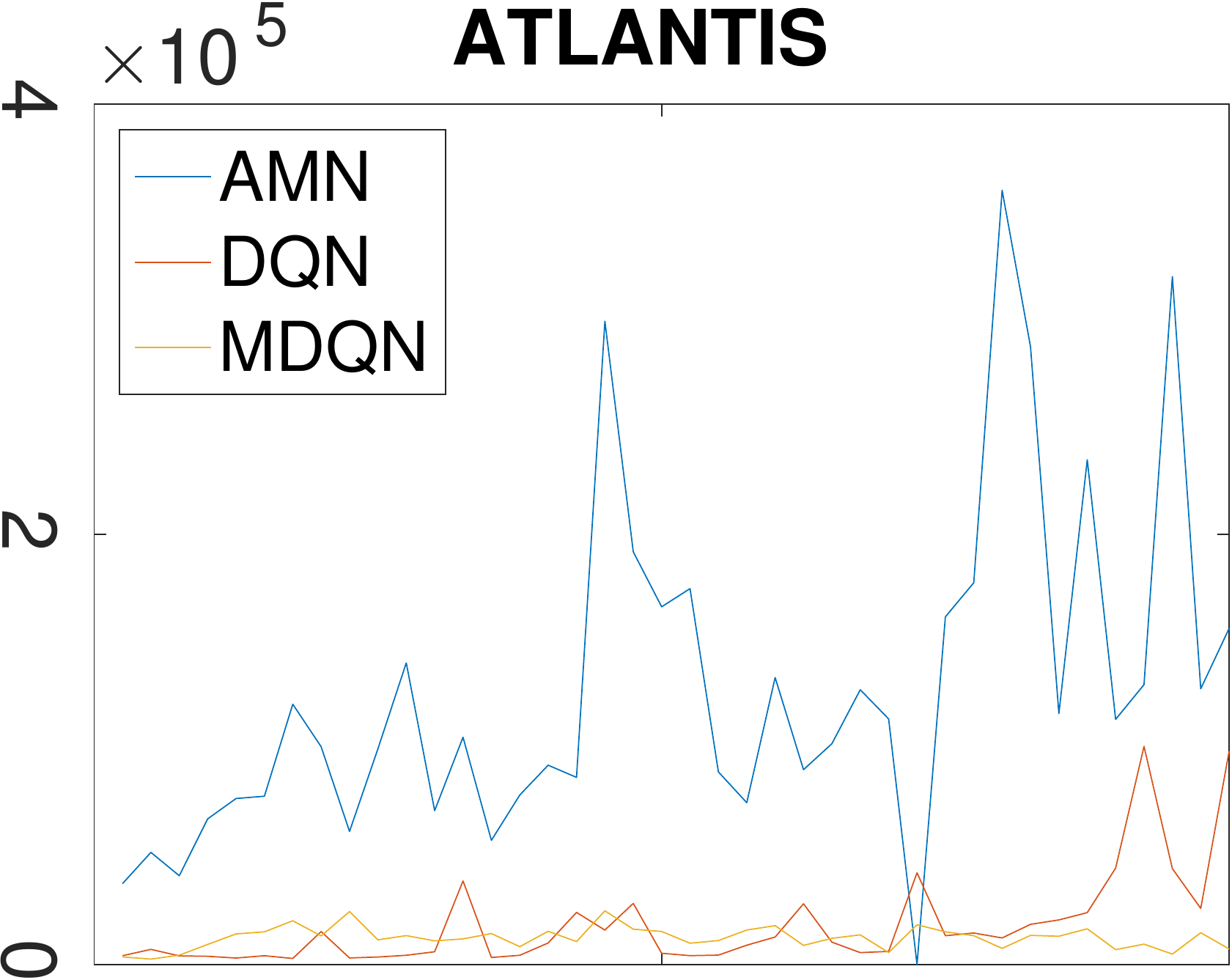} &
\includegraphics[width=0.225\linewidth]{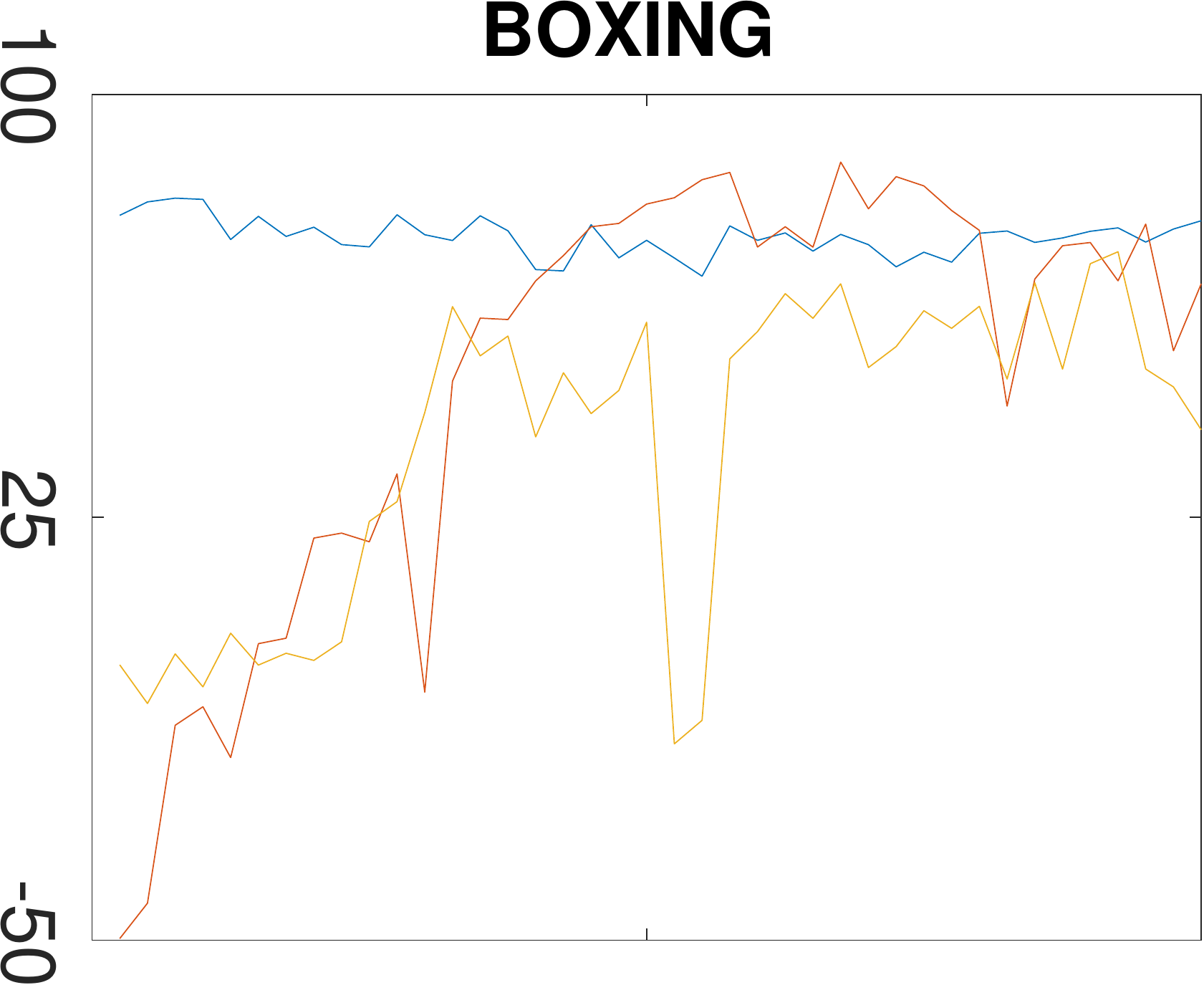} &
\includegraphics[width=0.225\linewidth]{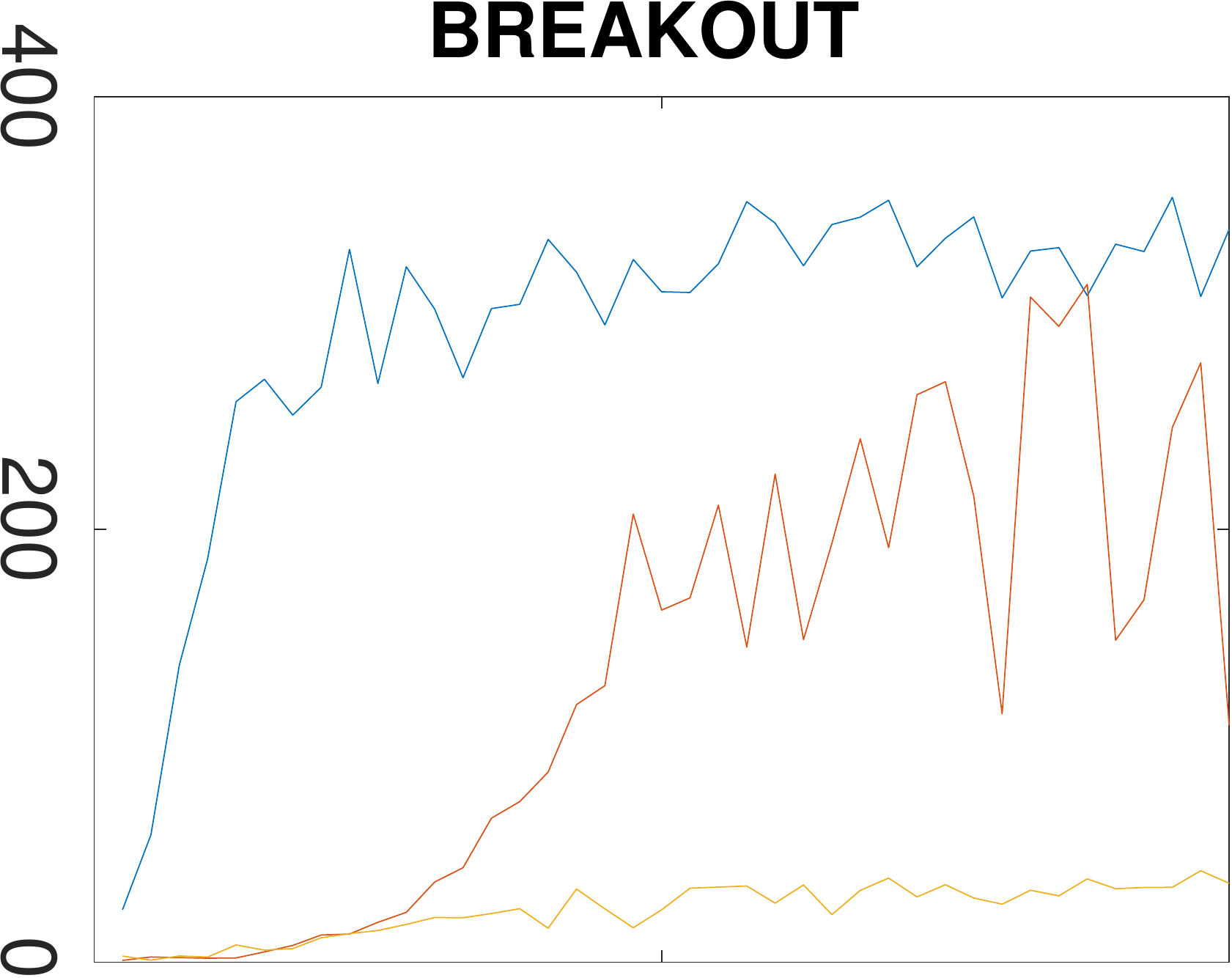} &
\includegraphics[width=0.225\linewidth]{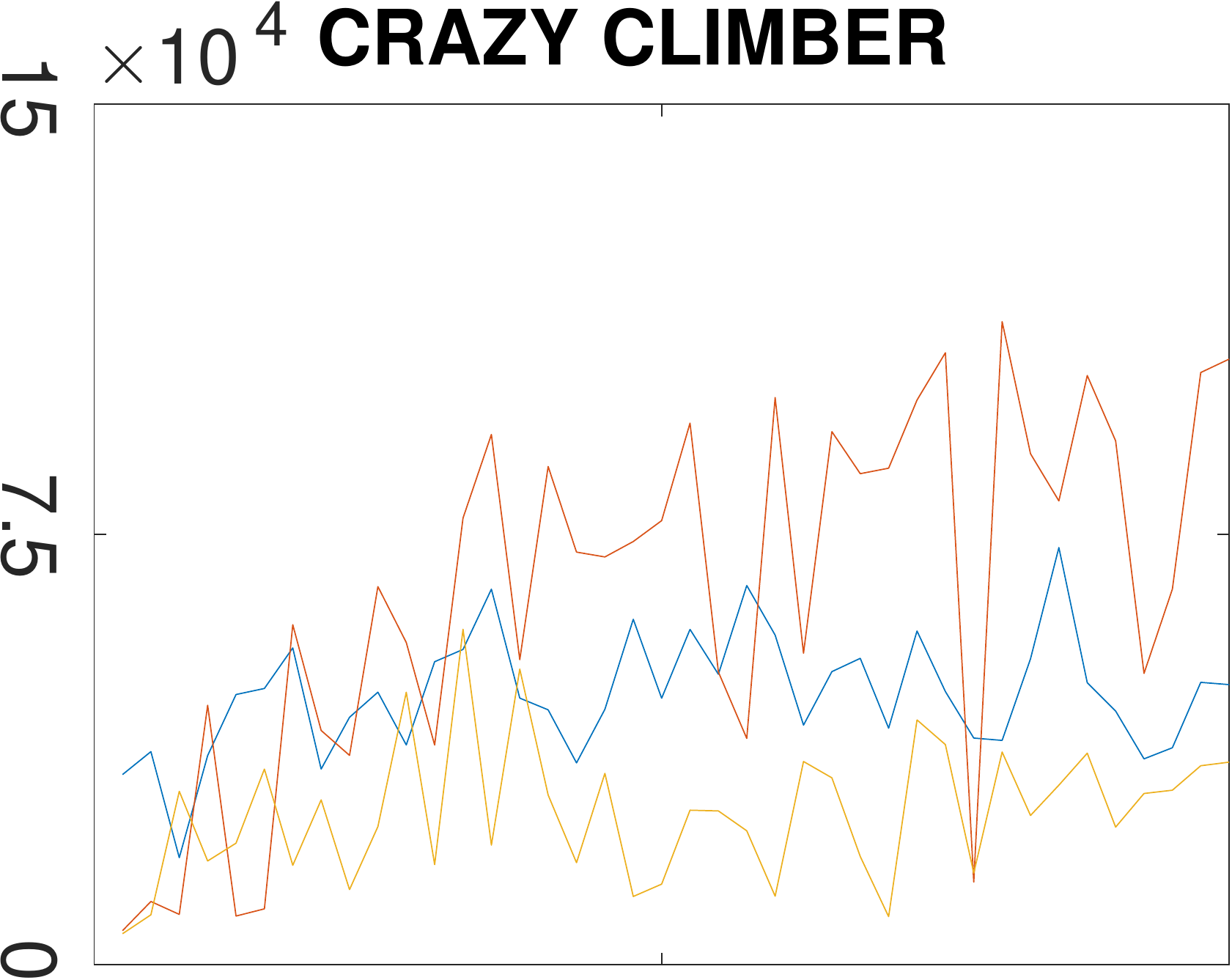} \\ 
\includegraphics[width=0.225\linewidth]{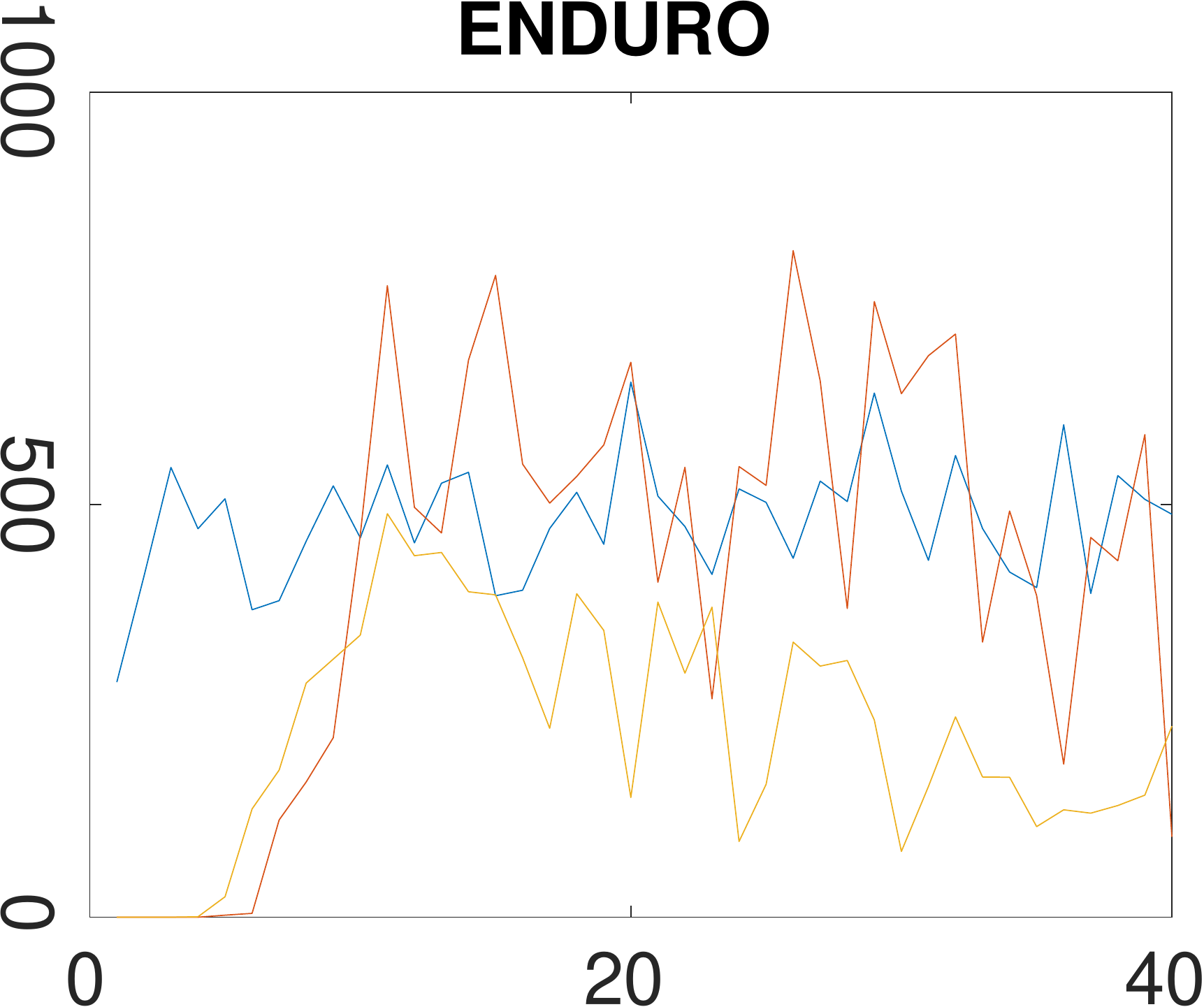} & 
\includegraphics[width=0.225\linewidth]{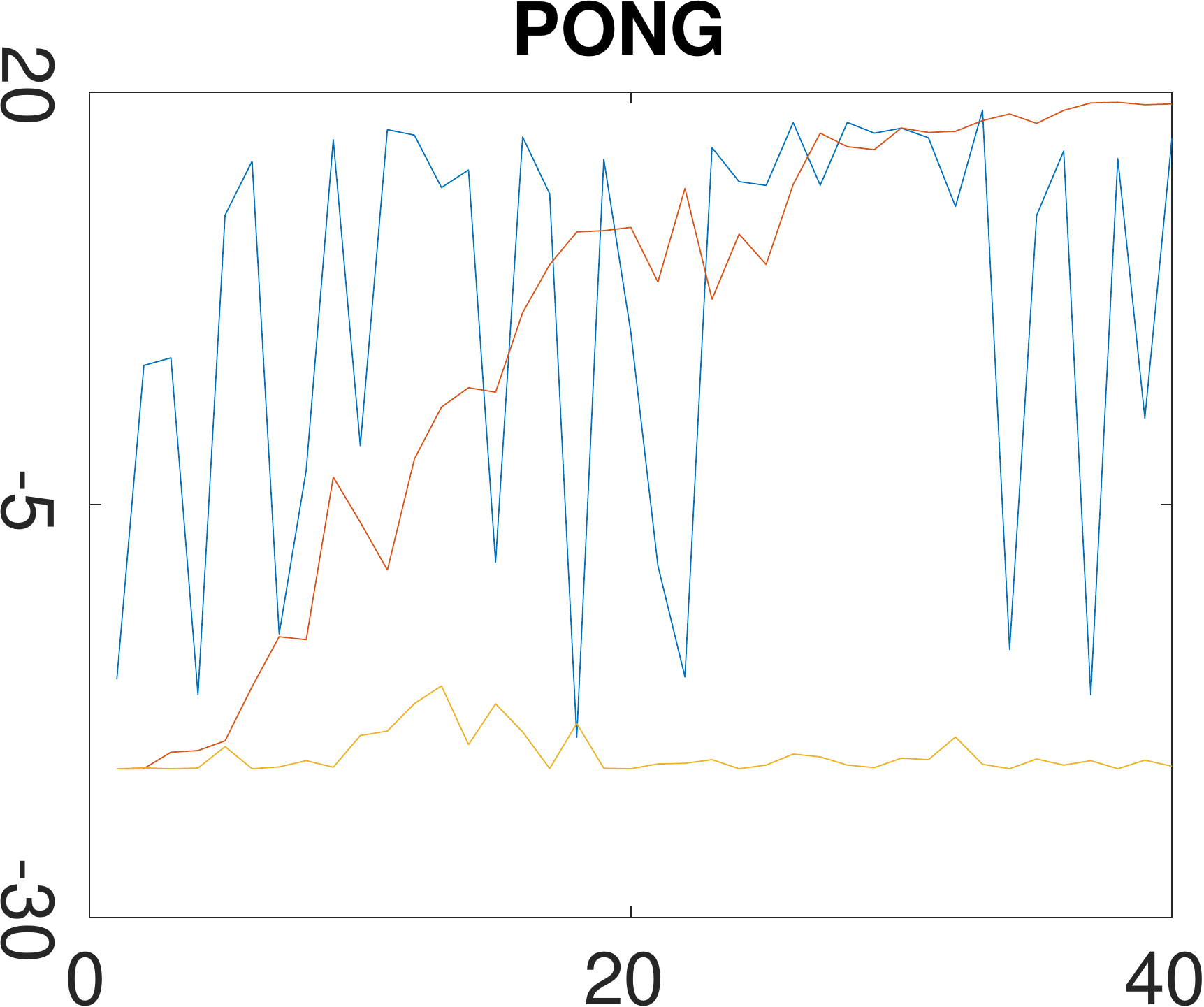} &
\includegraphics[width=0.225\linewidth]{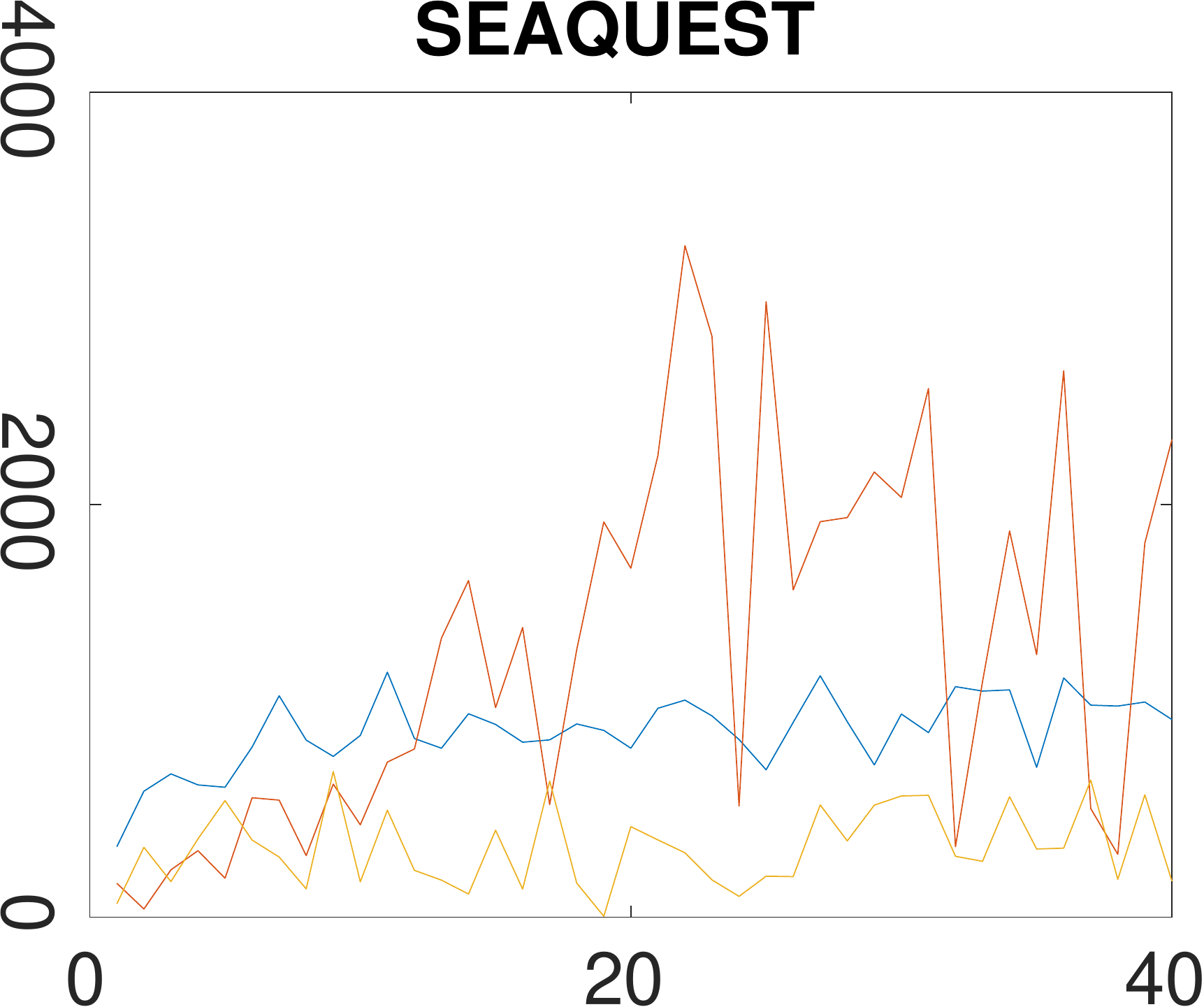} & 
\includegraphics[width=0.225\linewidth]{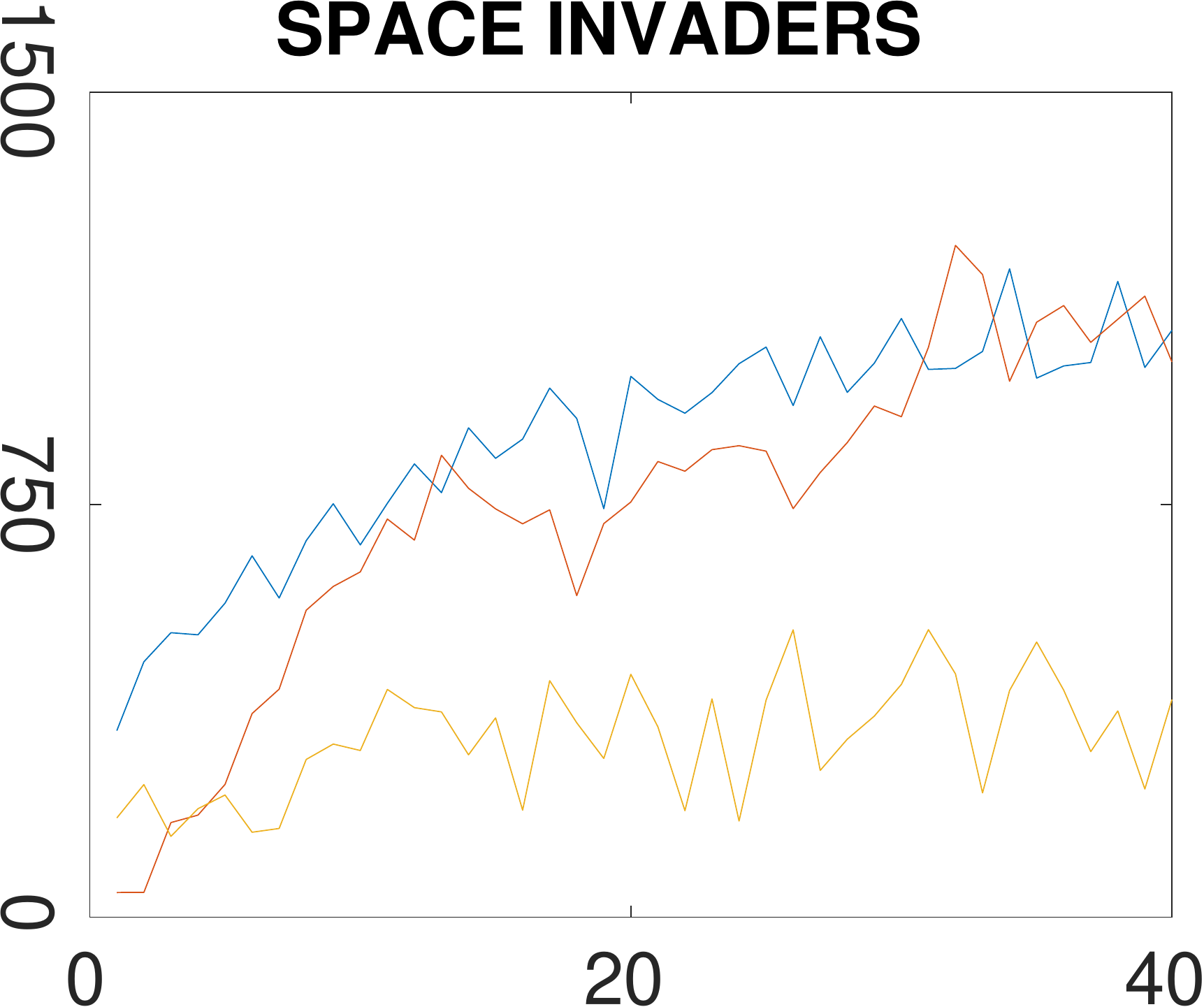} 
\end{tabular}
\vspace{-0.1in}
\caption{\small The Actor-Mimic, expert DQN, and Multitask DQN (MDQN) training curves for 40 training epochs for each of the 8 games. A training epoch is 250,000 frames and for each training epoch we evaluate the networks with a testing epoch that lasts 125,000 frames.  We report AMN, expert DQN and MDQN test reward for each testing epoch. In the testing epoch we use $\epsilon = 0.05$ in the $\epsilon$-greedy policy. The y-axis is the average unscaled episode reward during a testing epoch.}
\label{fig:MDQN_results}
\end{figure}

\begin{figure}[t!]
\begin{tabular}{cccc}
\includegraphics[width=0.225\linewidth]{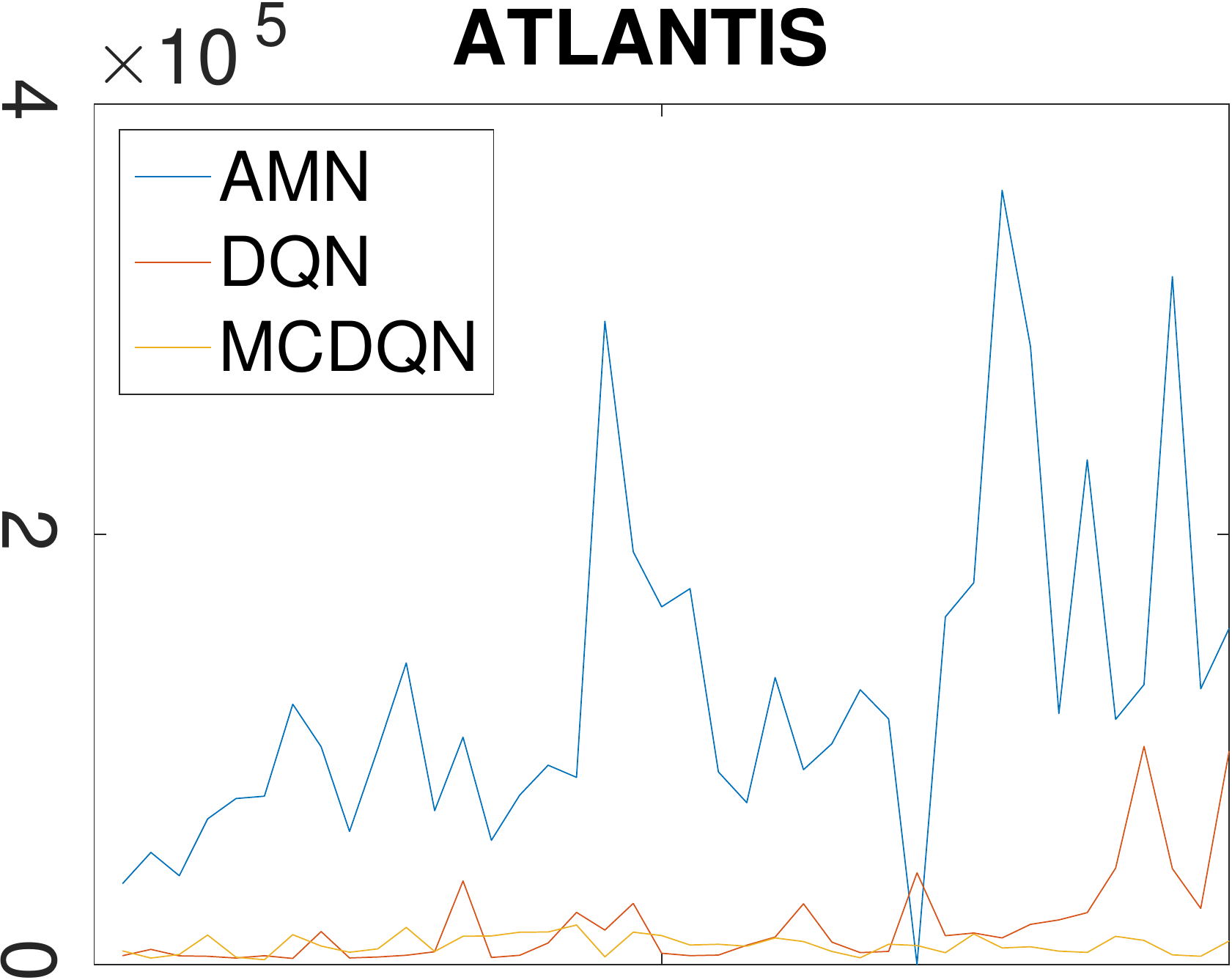} &
\includegraphics[width=0.225\linewidth]{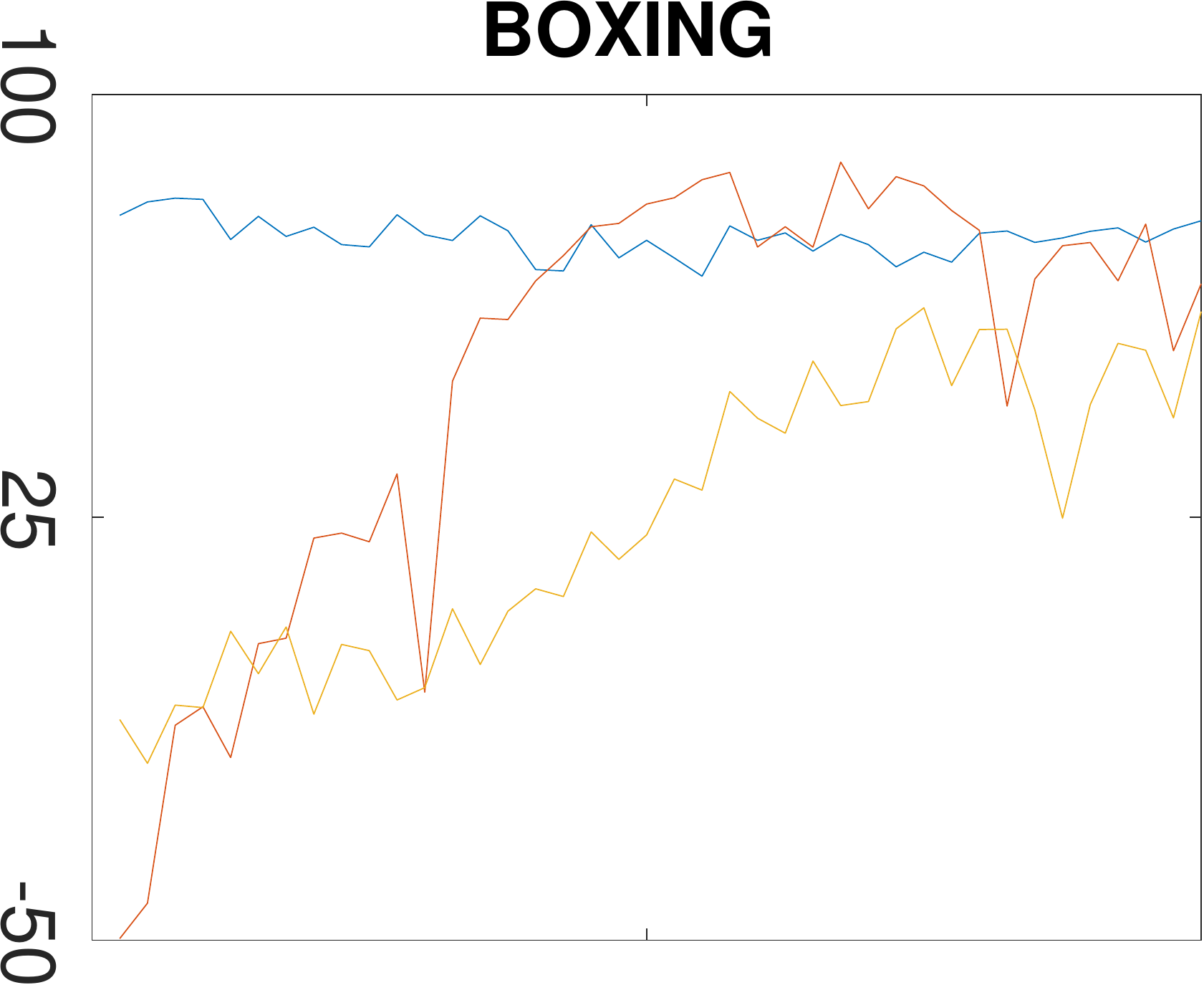} &
\includegraphics[width=0.225\linewidth]{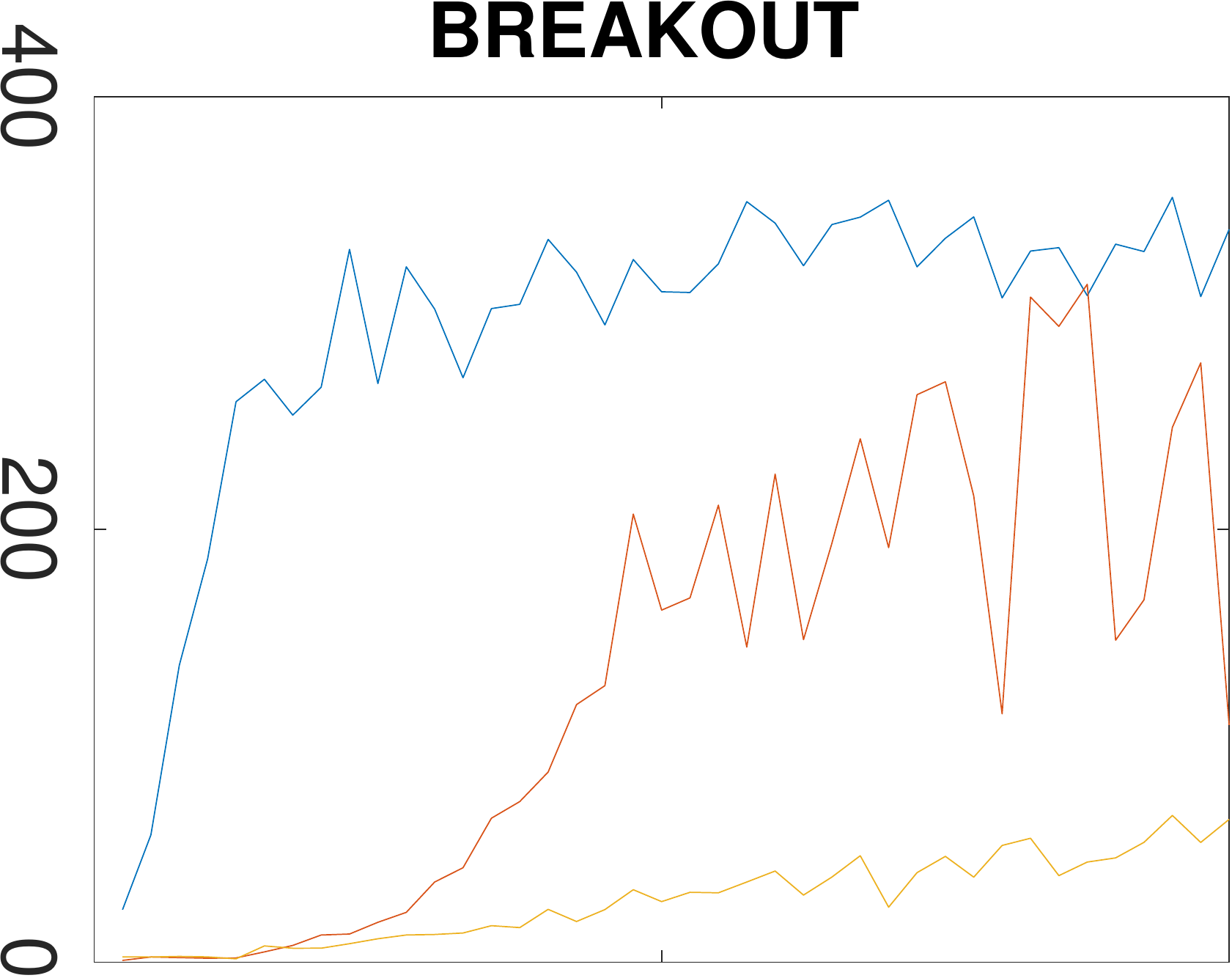} &
\includegraphics[width=0.225\linewidth]{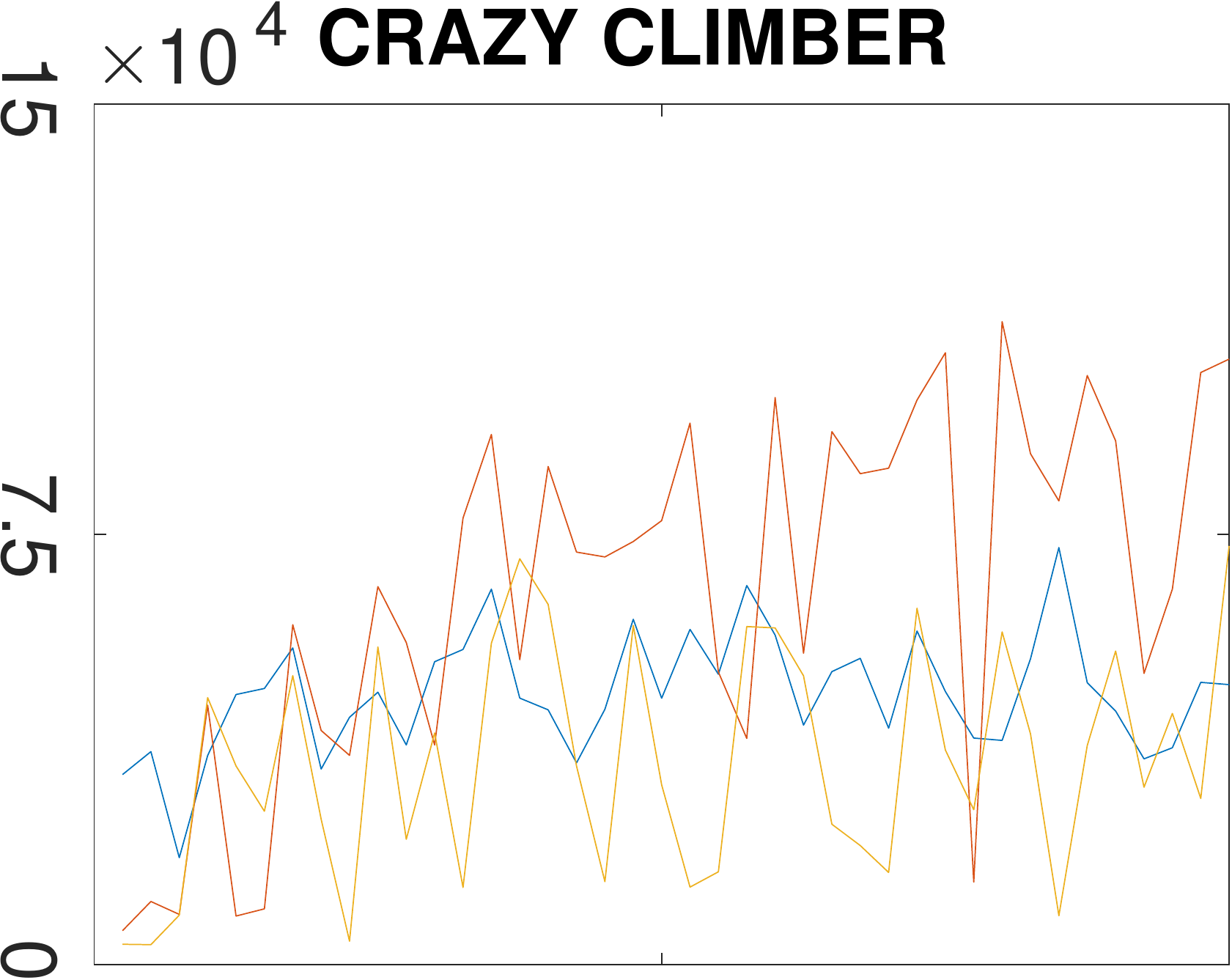} \\ 
\includegraphics[width=0.225\linewidth]{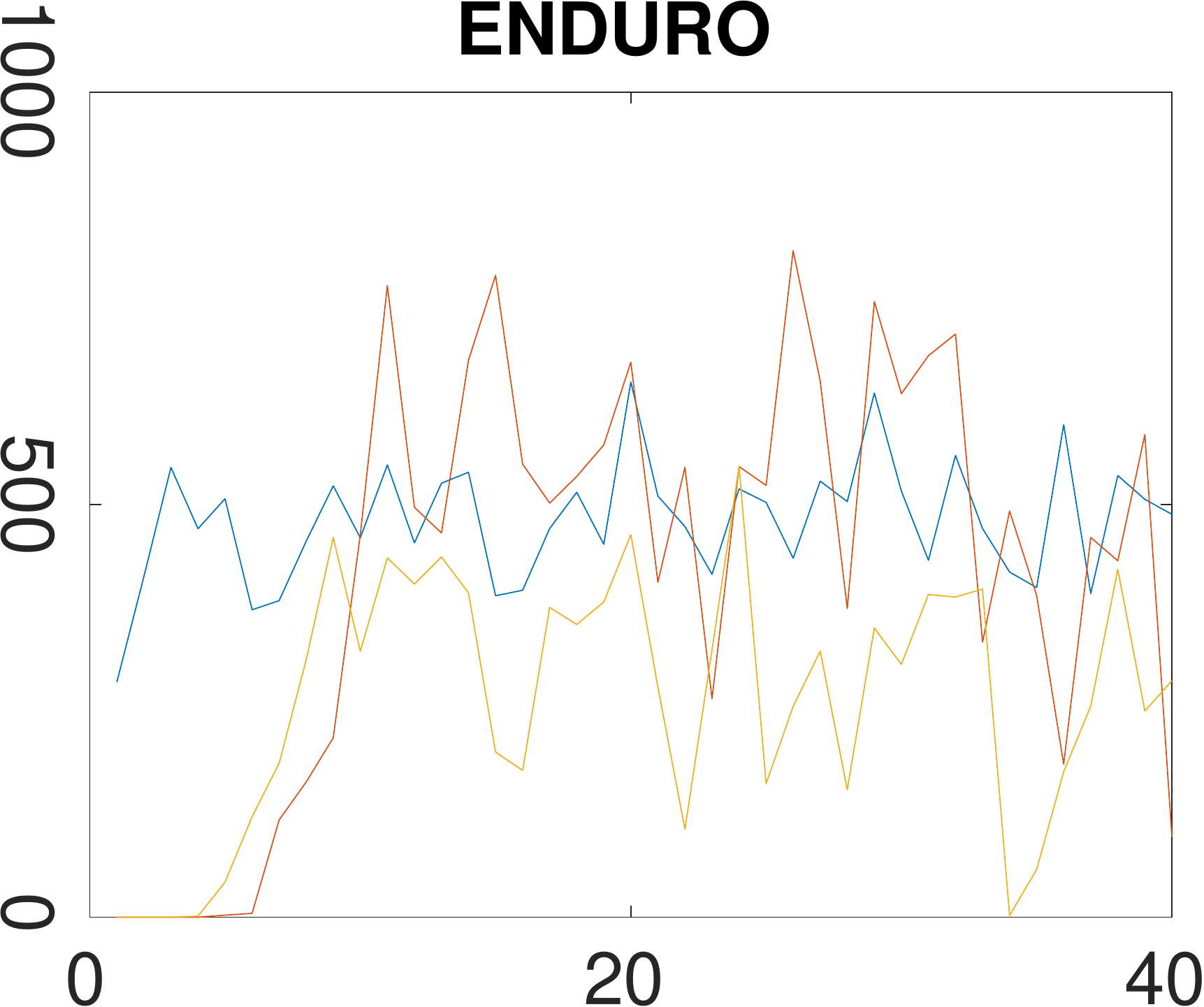} & 
\includegraphics[width=0.225\linewidth]{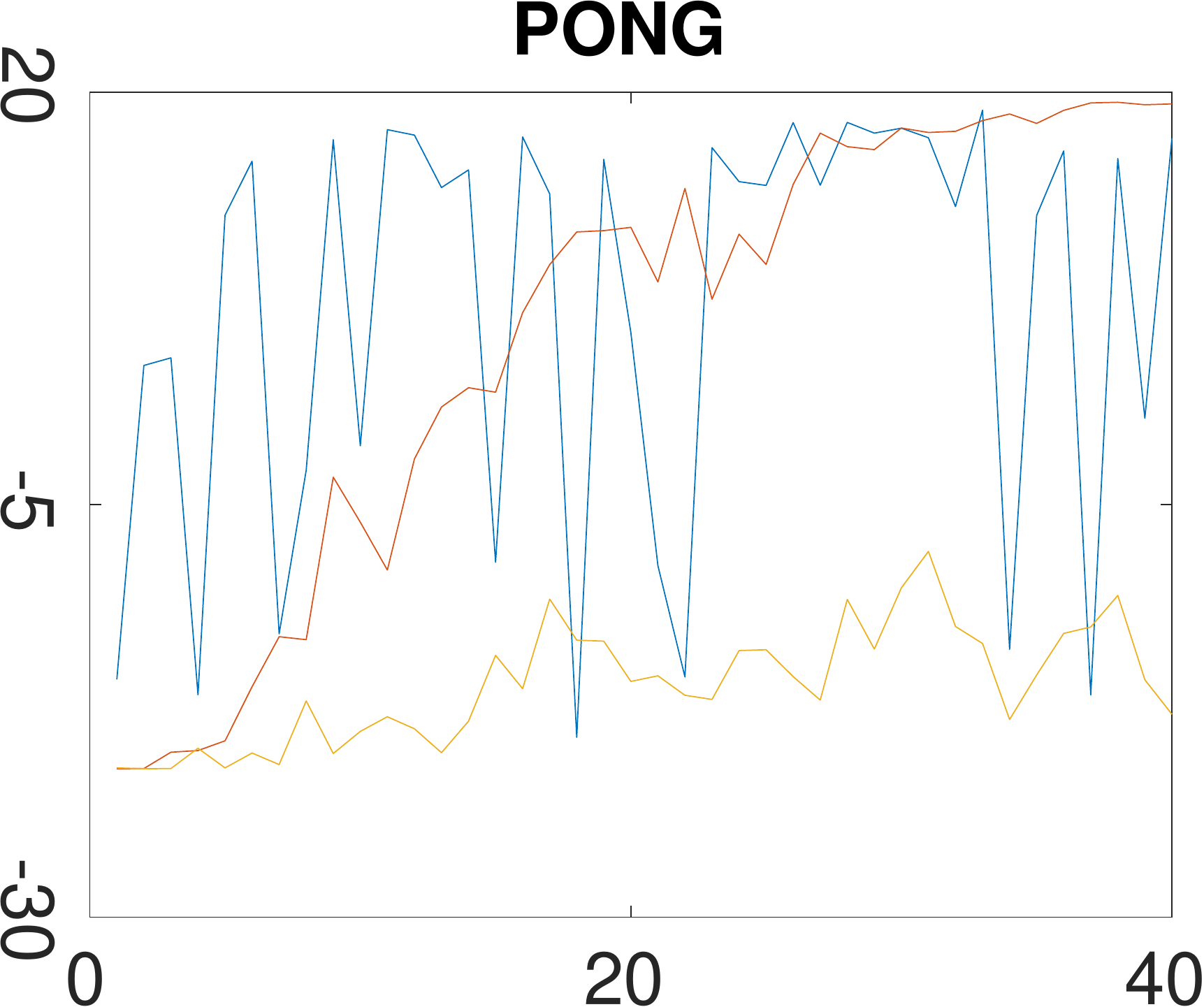} &
\includegraphics[width=0.225\linewidth]{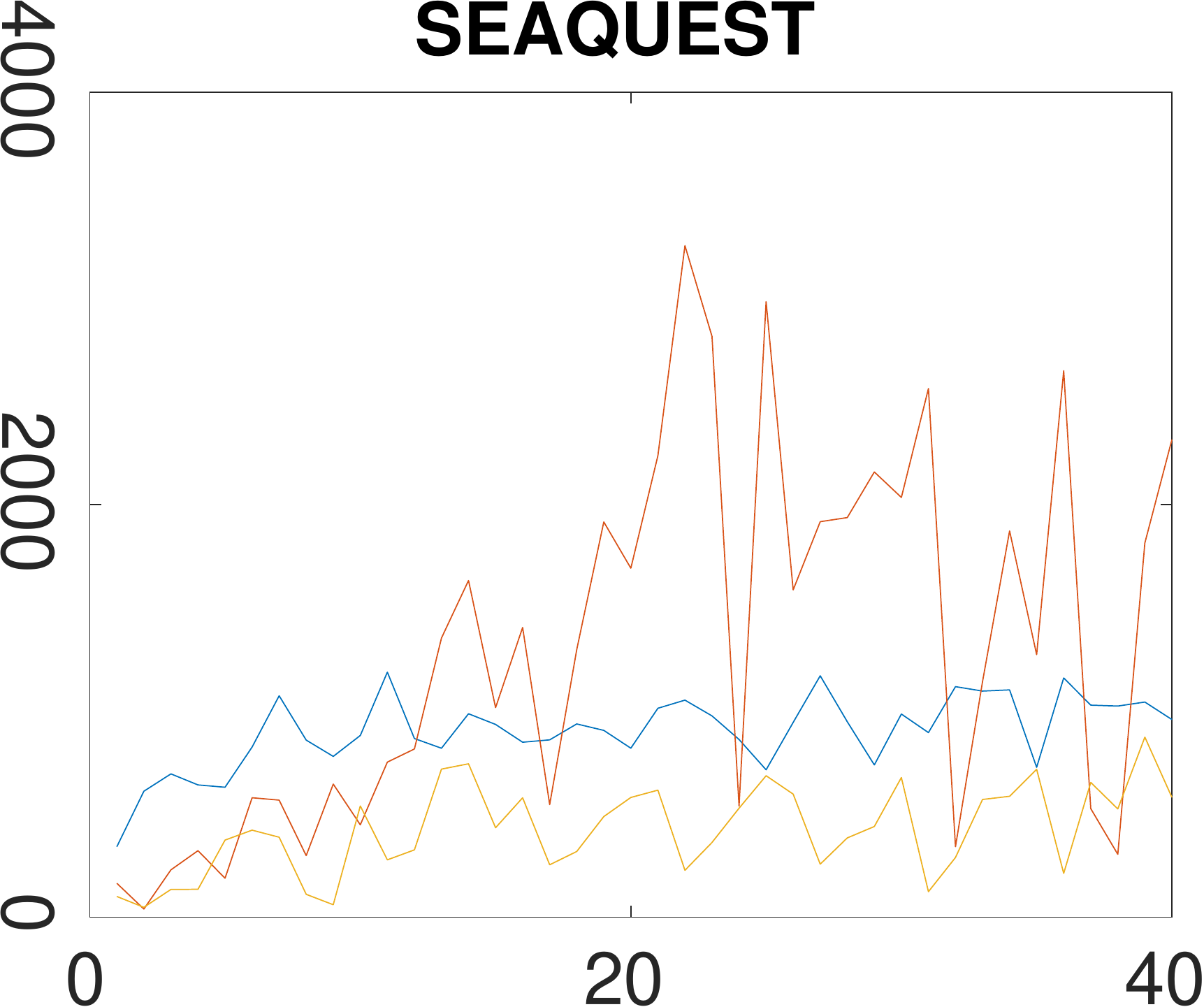} & 
\includegraphics[width=0.225\linewidth]{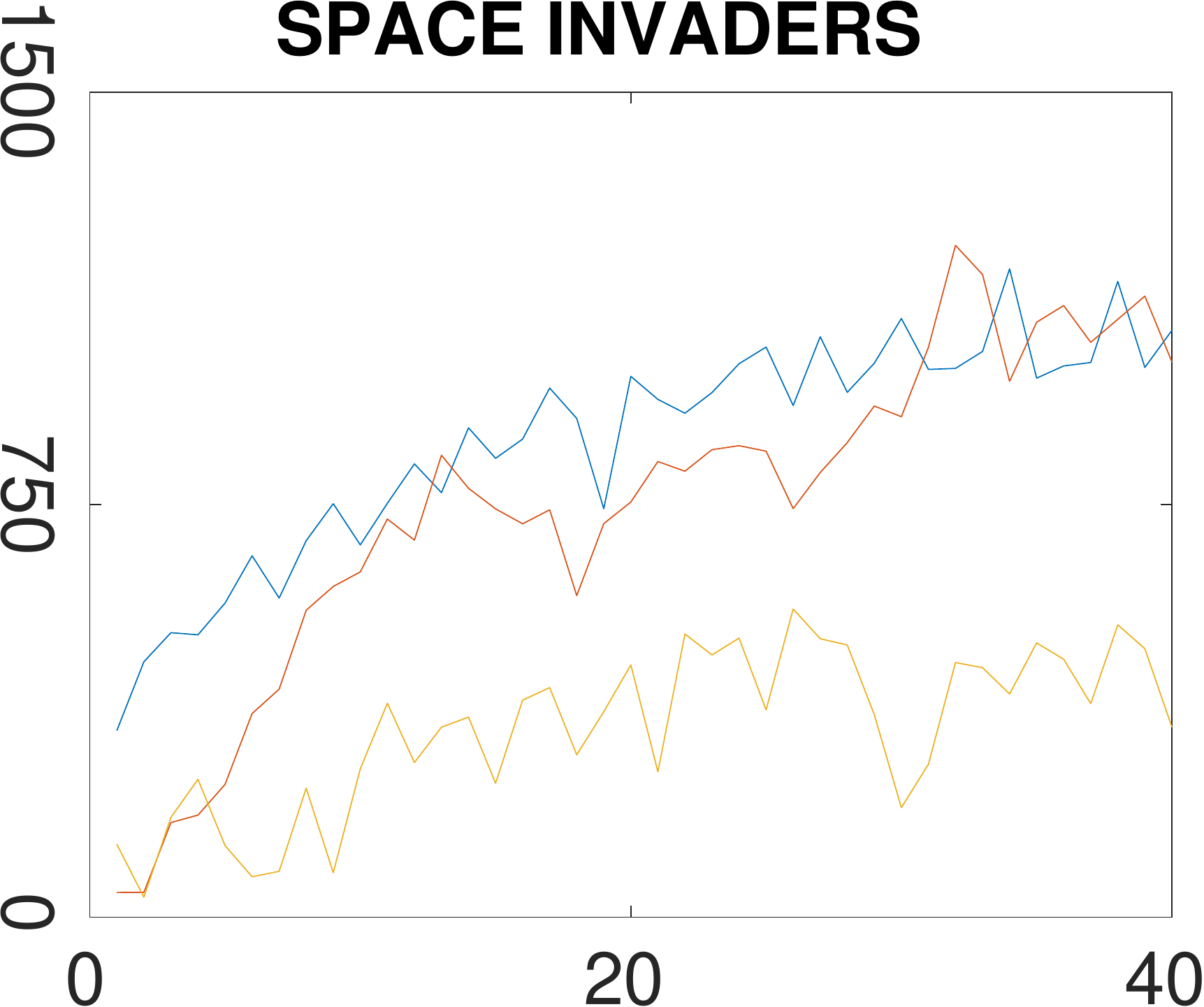} 
\end{tabular}
\vspace{-0.1in}
\caption{\small The Actor-Mimic, expert DQN, and Multitask Convolutions DQN (MCDQN) training curves for 40 training epochs for each of the 8 games. A training epoch is 250,000 frames and for each training epoch we evaluate the networks with a testing epoch that lasts 125,000 frames.  We report AMN, expert DQN and MCDQN test reward for each testing epoch. In the testing epoch we use $\epsilon = 0.05$ in the $\epsilon$-greedy policy. The y-axis is the average unscaled episode reward during a testing epoch.}
\label{fig:MCDQN_results}
\end{figure}

As a baseline, we trained DQN networks over 8 games simultaneously to test their performance against the Actor-Mimic method. We tried two different architectures, the first is using the basic DQN procedure on all 8 games. This network has a single 18 action output shared by all games, but when we train or test in a particular game, we mask out and ignore the action values from actions that are invalid for that particular game. This architecture is denoted the Multitask DQN (MDQN). The second architecture is a DQN but where each game has a separate fully-connected feature layer and action output. In this architecture only the convolutions are shared between games, and thus the features and action values are completely separate. This was to try to mitigate the destabilizing effect that the different value scales of each game had during learning. This architecture is denoted the Multitask Convolutions DQN (MCDQN). 

The results for the MDQN and MCDQN are shown in Figures~\ref{fig:MDQN_results} and~\ref{fig:MCDQN_results}, respectively. From the figures, we can observe that the AMN is far more stable during training as well as being consistently higher in performance than either the MDQN or MCDQN methods. In addition, it can be seen that the MDQN and MCDQN will often focus on performing reasonably well on a small subset of the source games, such as on Boxing and Enduro, while making little to no progress in others, such as Breakout or Pong. Between the MDQN and MCDQN, we can see that the MCDQN hardly improves results even though it has significantly larger computational cost that scales linearly with the number of source games. 

For the specific details of the architectures we tested, for the MDQN the architecture was: 8x8x4x32-4 \footnote{\label{fn:conv} Here we represent convolutional layers as WxWxCxN-S, where W is the width of the (square) convolution kernel, C is the number of input images, N is the number of filter maps and S is the convolution stride.} $\rightarrow$ 4x4x32x64-2 $\rightarrow$ 3x3x64x64-1 $\rightarrow$ 512 fully-connected units $\rightarrow$ 18 actions. This is exactly the same network architecture as used for the 8 game AMN in Section~\ref{Actor_Mimic_exp}. For the MCDQN, the bottom convolutional layers were the same as the MDQN, except there are 8 parallel subnetworks on top of the convolutional layers. These game-specific subnetworks had the architecture: 512 fully-connected units $\rightarrow$ 18 actions. All layers except the action outputs were followed with a rectifier non-linearity.

\end{appendices}

\begin{appendices}

\section{Actor-Mimic Network Multitask Results for Transfer Pretraining}
\label{app:transfer}

The network used for transfer consisted of the following architecture: 8x8x4x256-4~\textsuperscript{\ref{fn:conv}} $\rightarrow$ 4x4x256x512-2 $\rightarrow$ 3x3x512x512-1 $\rightarrow$ 3x3x512x512-1 $\rightarrow$ 2048 fully-connected units $\rightarrow$ 1024 fully-connected units $\rightarrow$ 18 actions. All layers except the final one were followed with a rectifier non-linearity. 

\begin{figure}[hb!]
\begin{tabular}{ccc}
\includegraphics[width=0.3\linewidth]{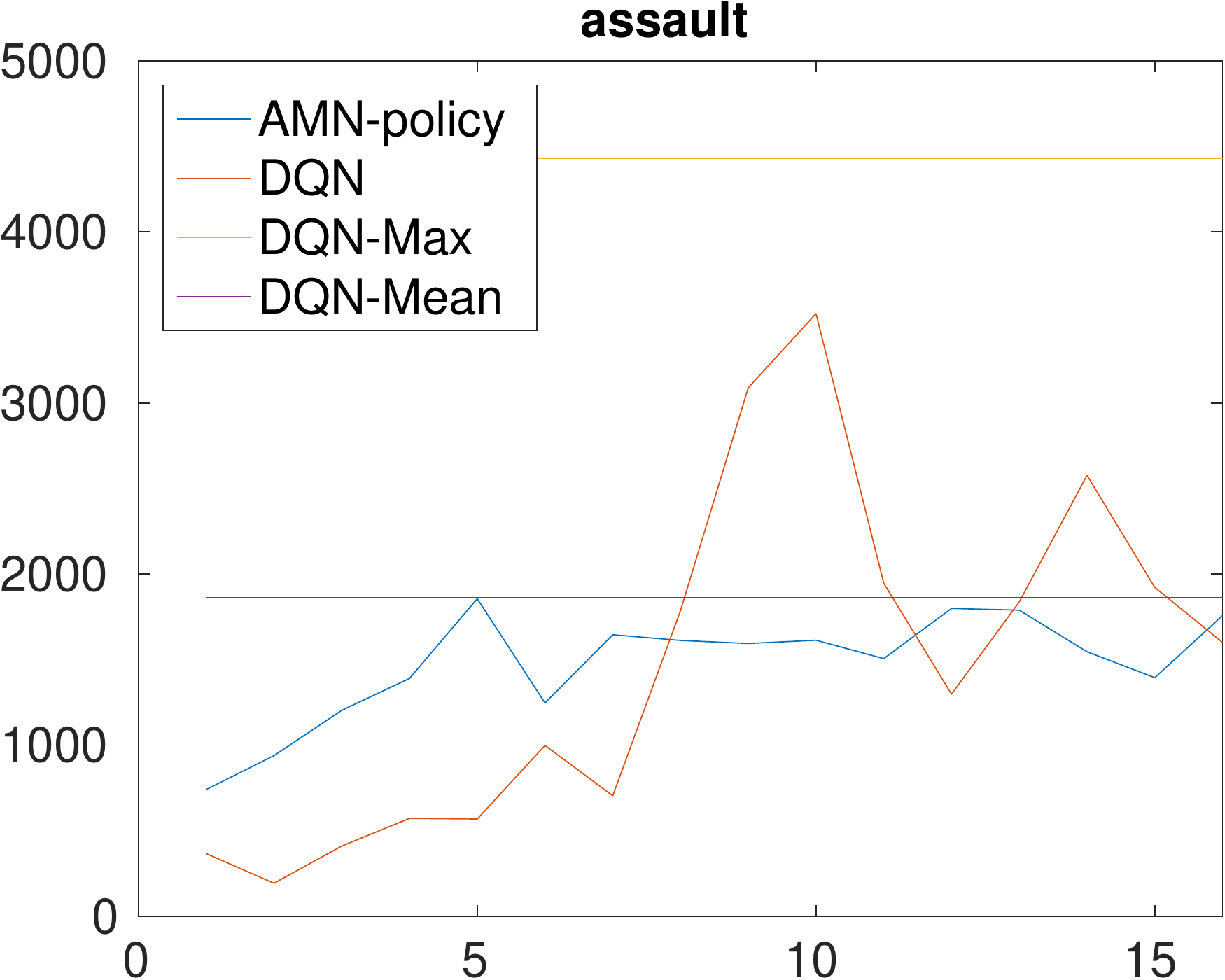} &
\includegraphics[width=0.3\linewidth]{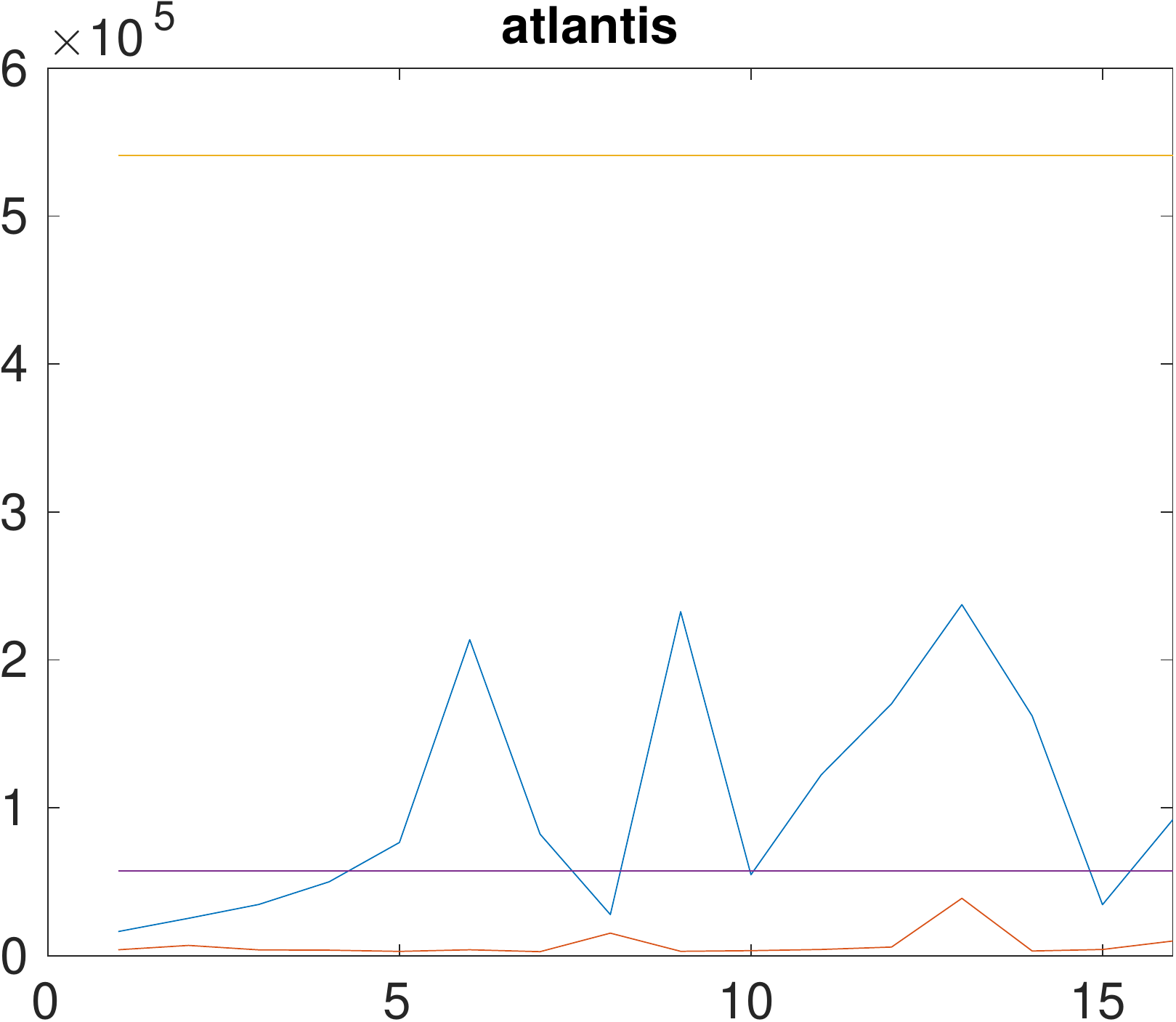} &
\includegraphics[width=0.3\linewidth]{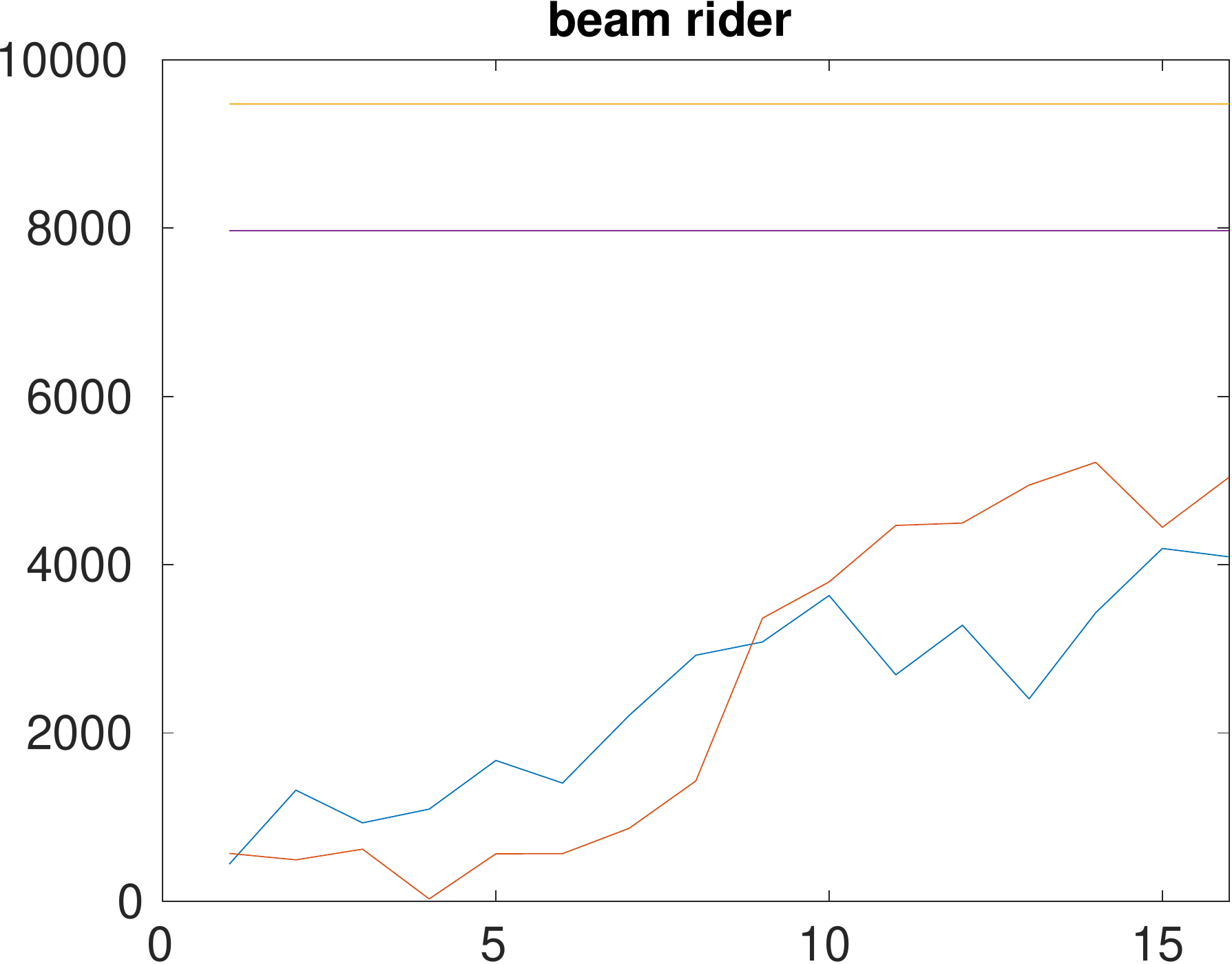} \\
\includegraphics[width=0.3\linewidth]{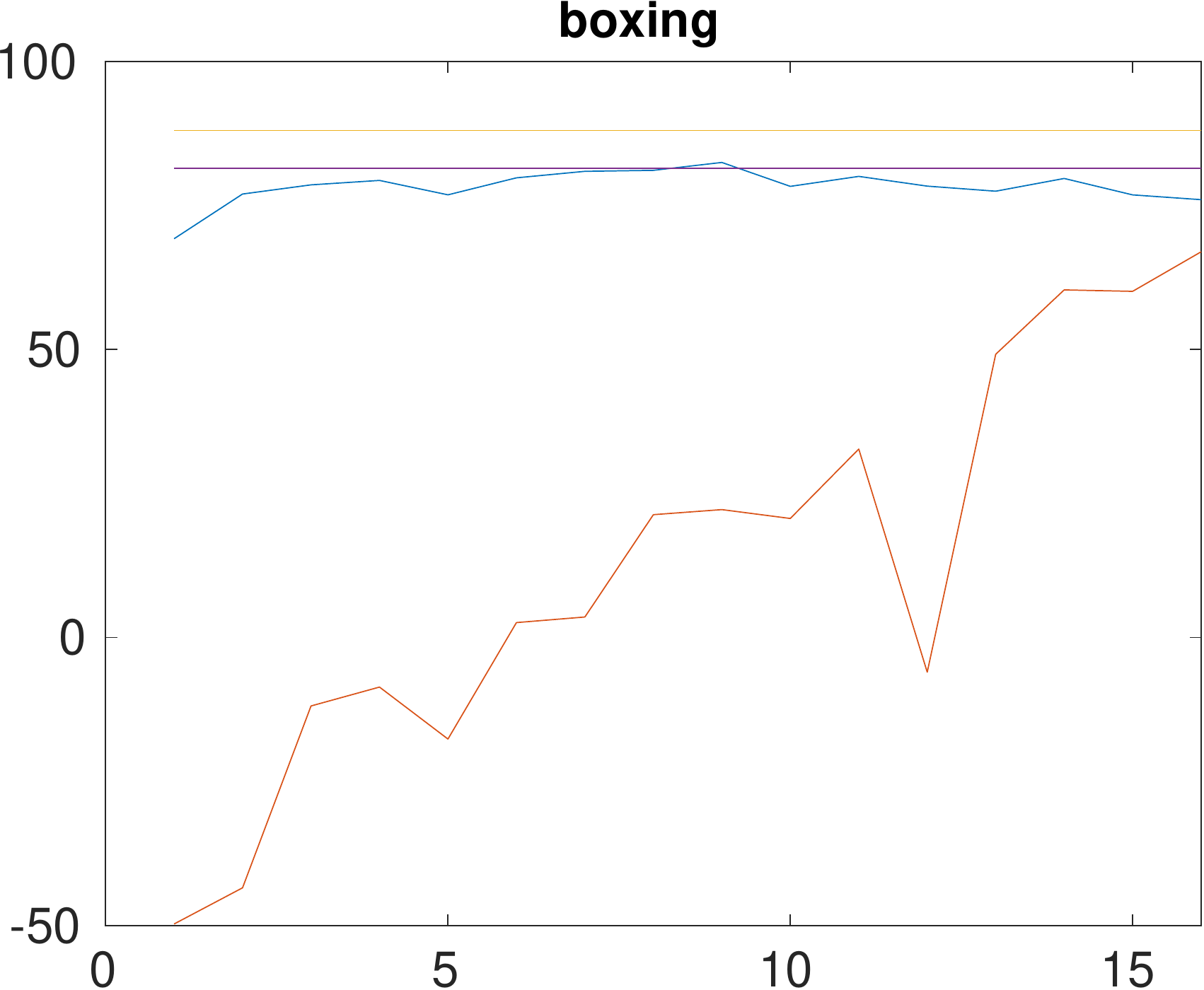} &
\includegraphics[width=0.3\linewidth]{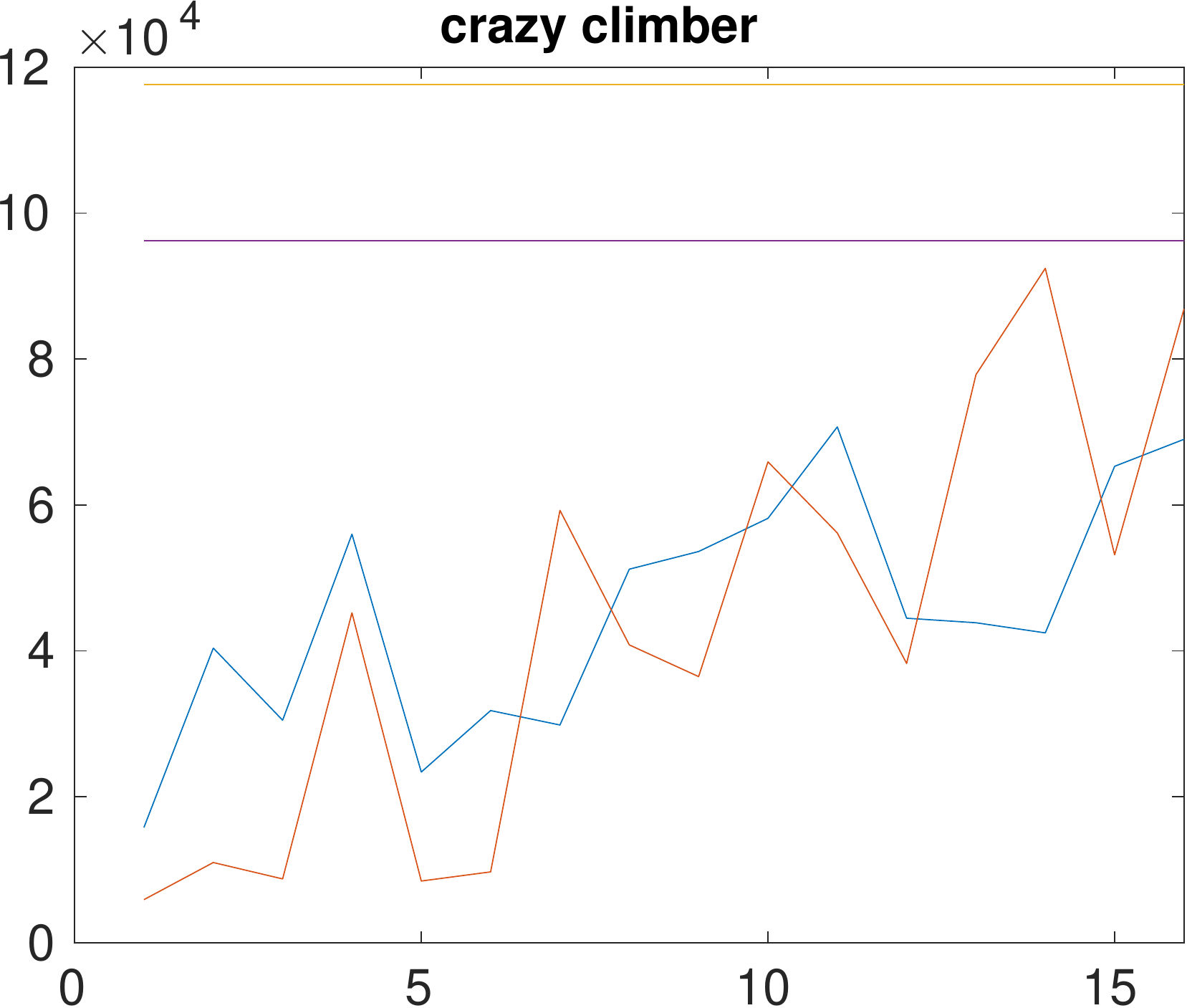} &
\includegraphics[width=0.3\linewidth]{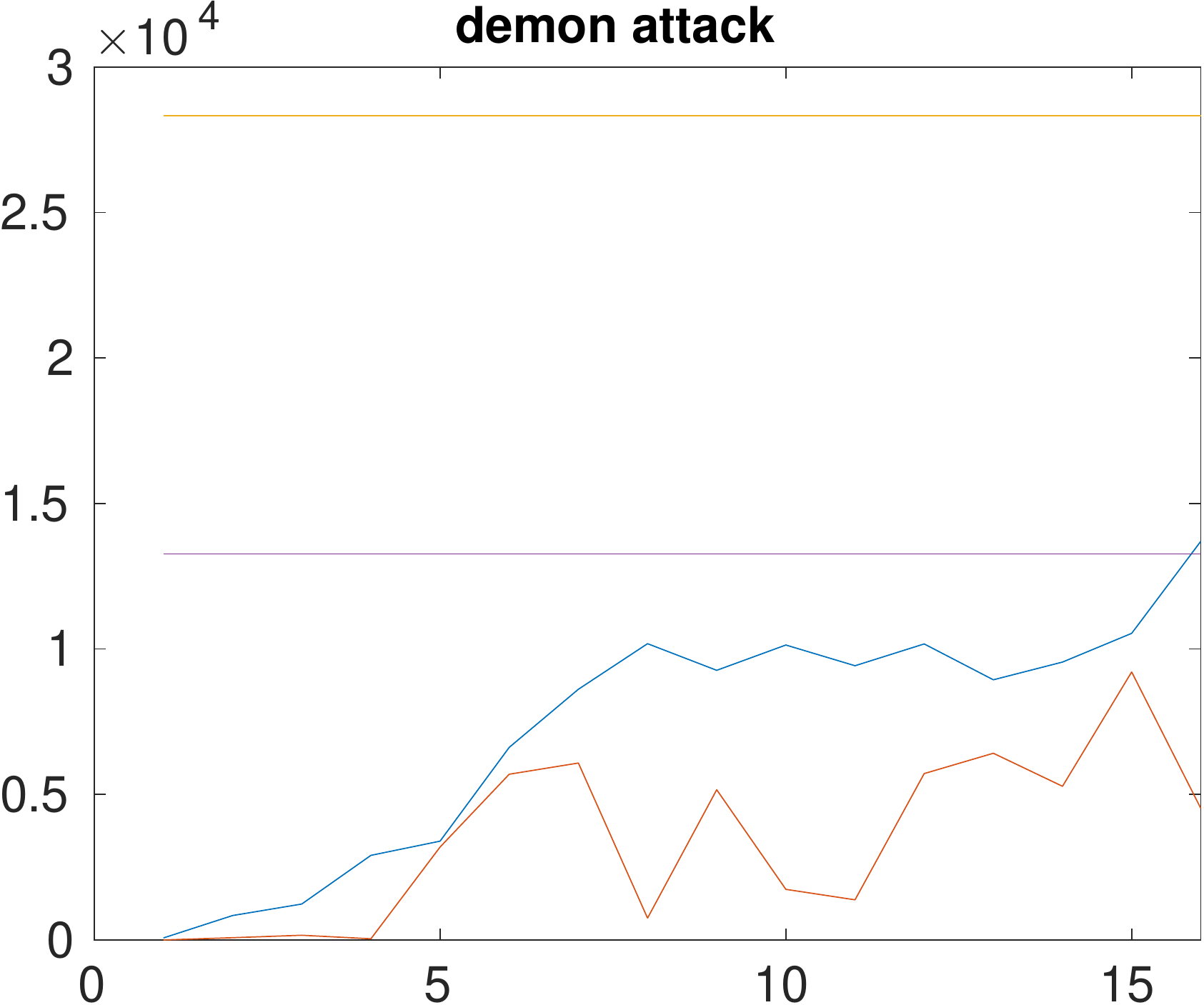} \\
\includegraphics[width=0.3\linewidth]{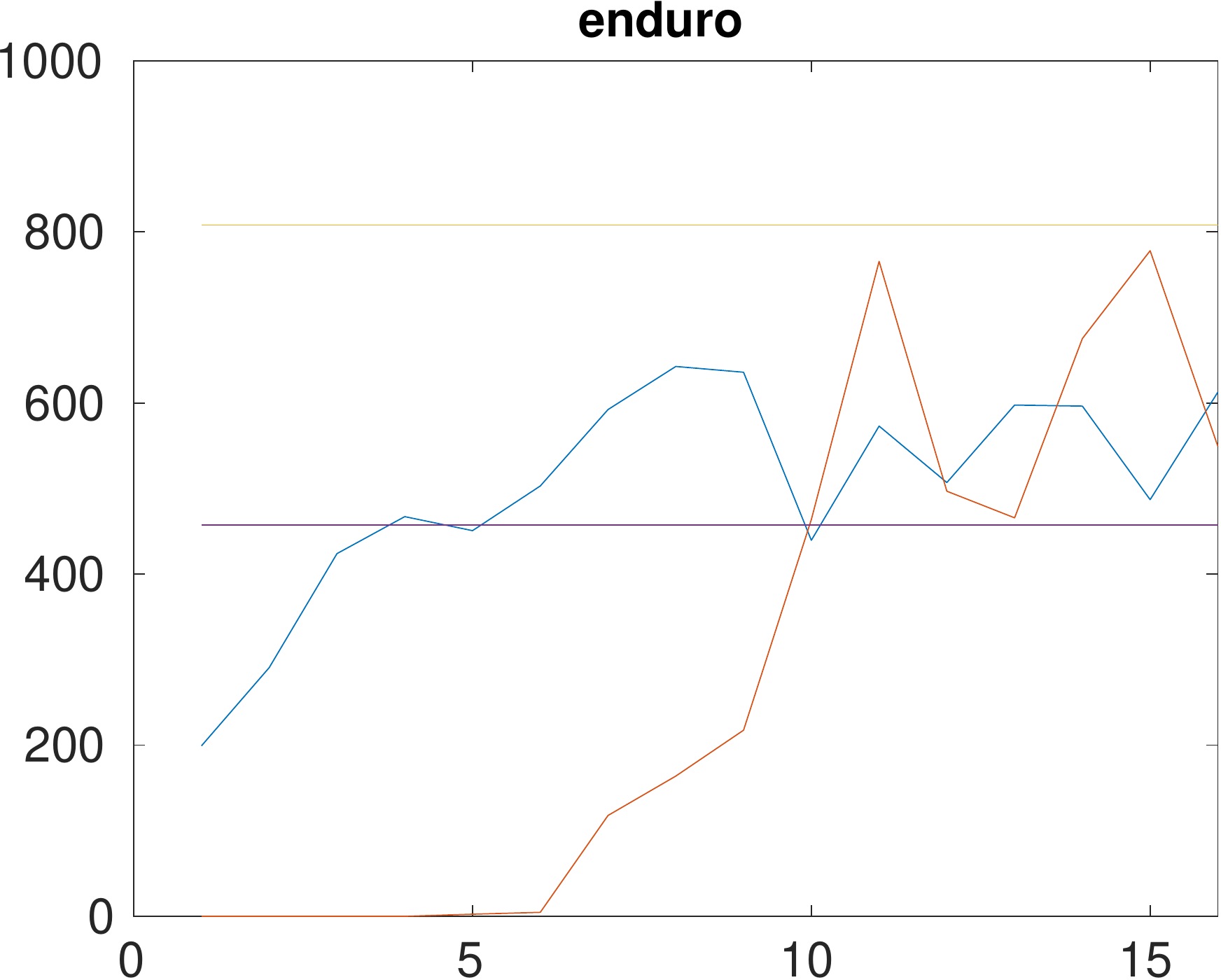} &
\includegraphics[width=0.3\linewidth]{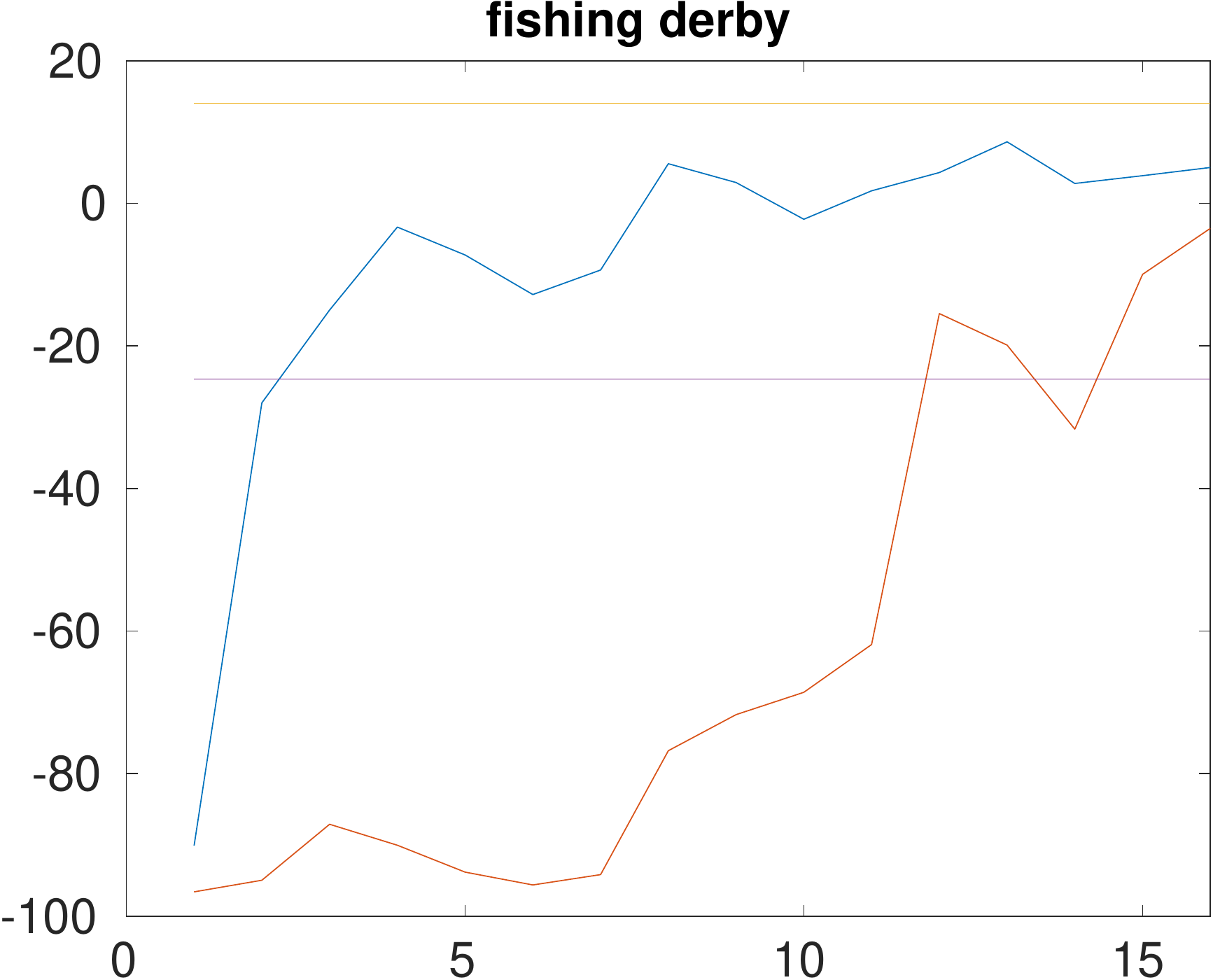} &
\includegraphics[width=0.3\linewidth]{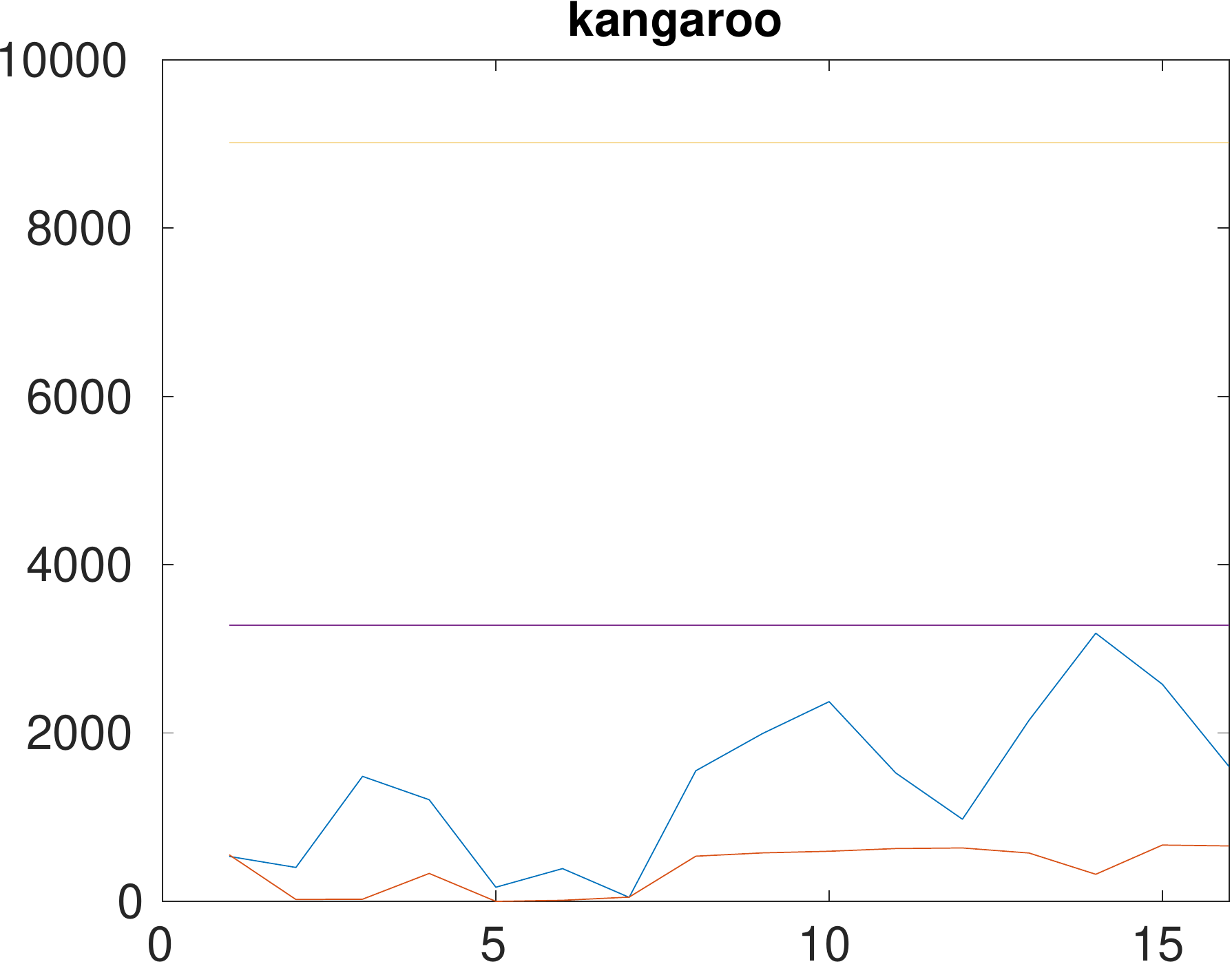} \\
\includegraphics[width=0.3\linewidth]{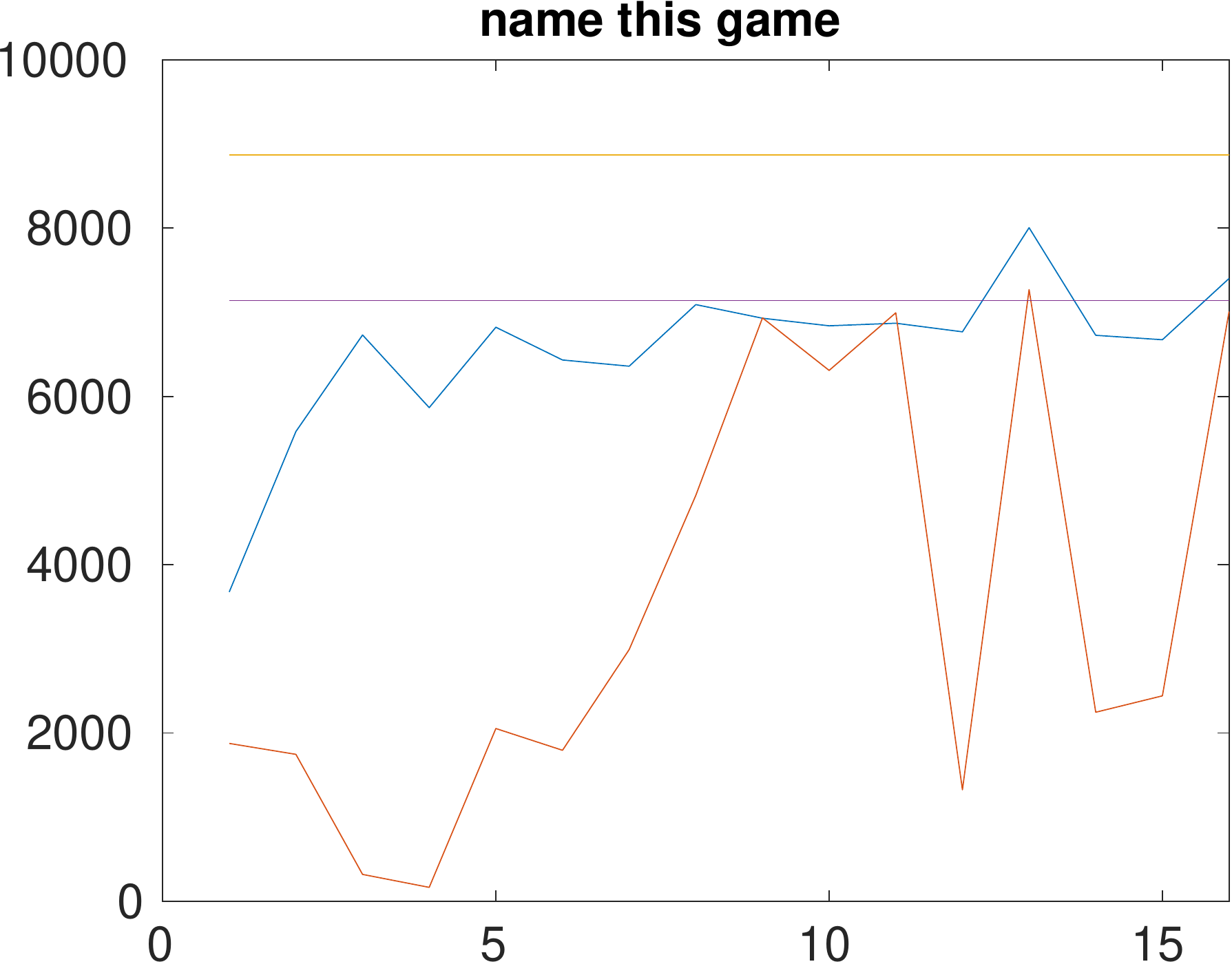} &
\includegraphics[width=0.3\linewidth]{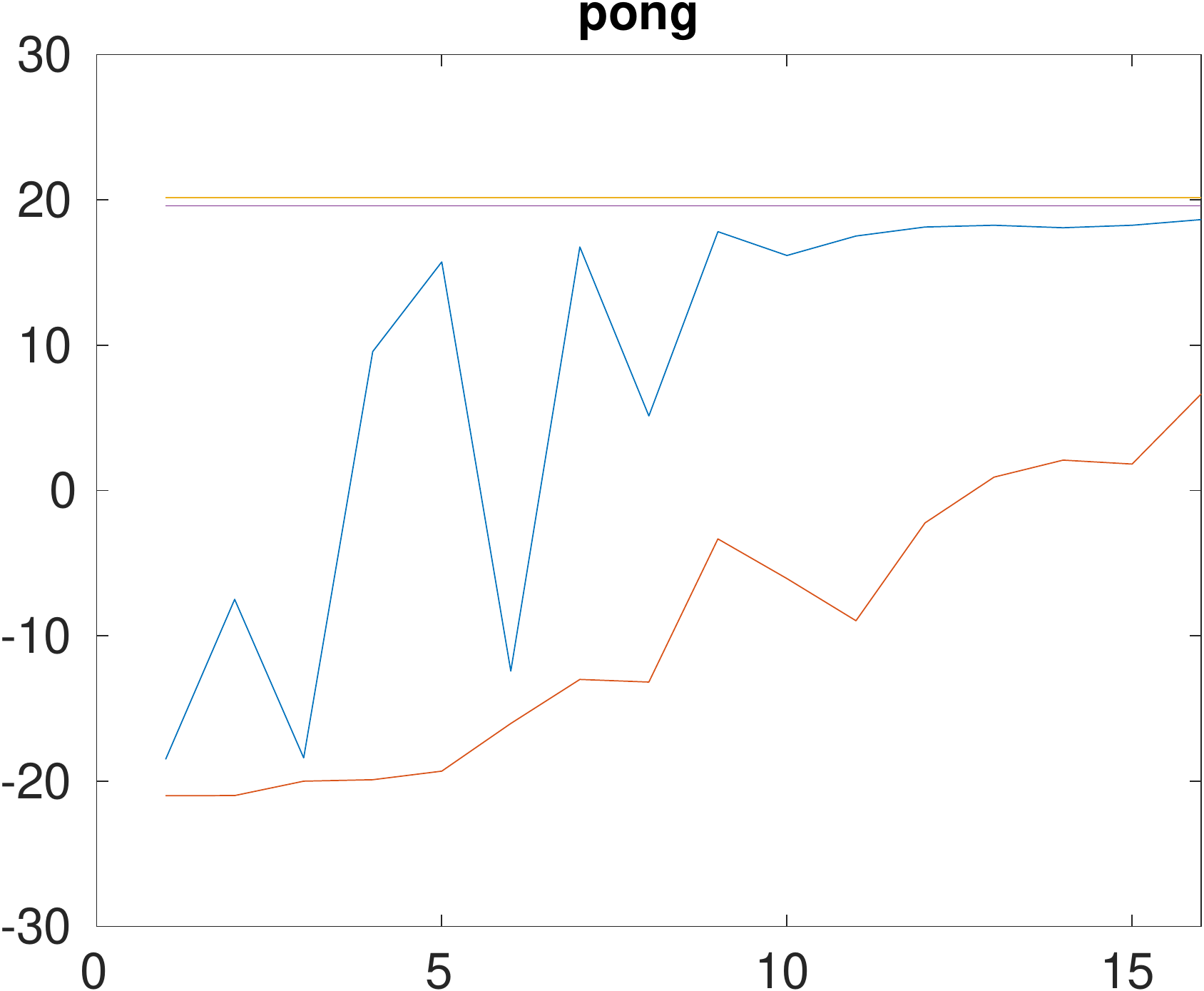} &
\includegraphics[width=0.3\linewidth]{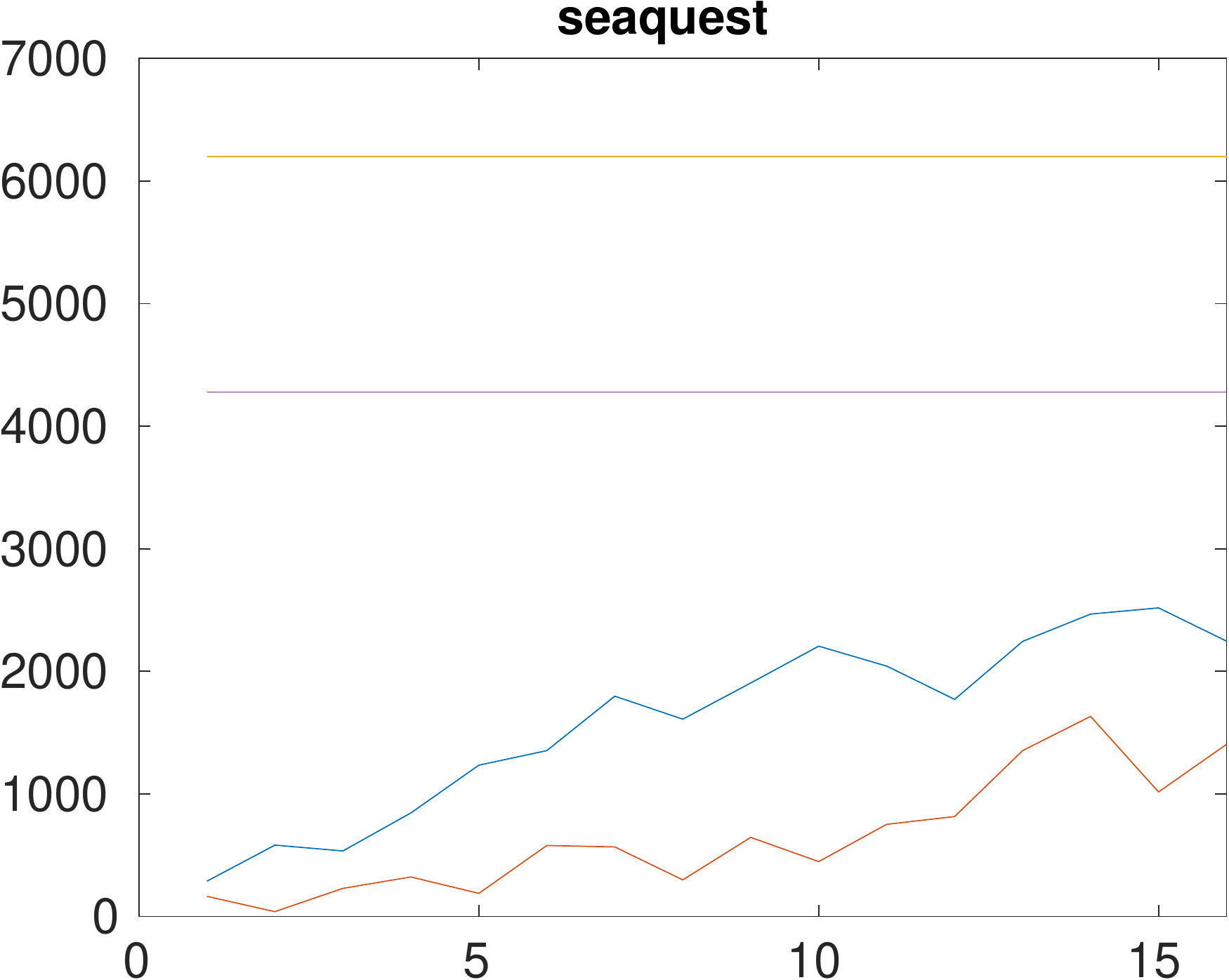} \\
\multicolumn{3}{c}{\includegraphics[width=0.3\linewidth]{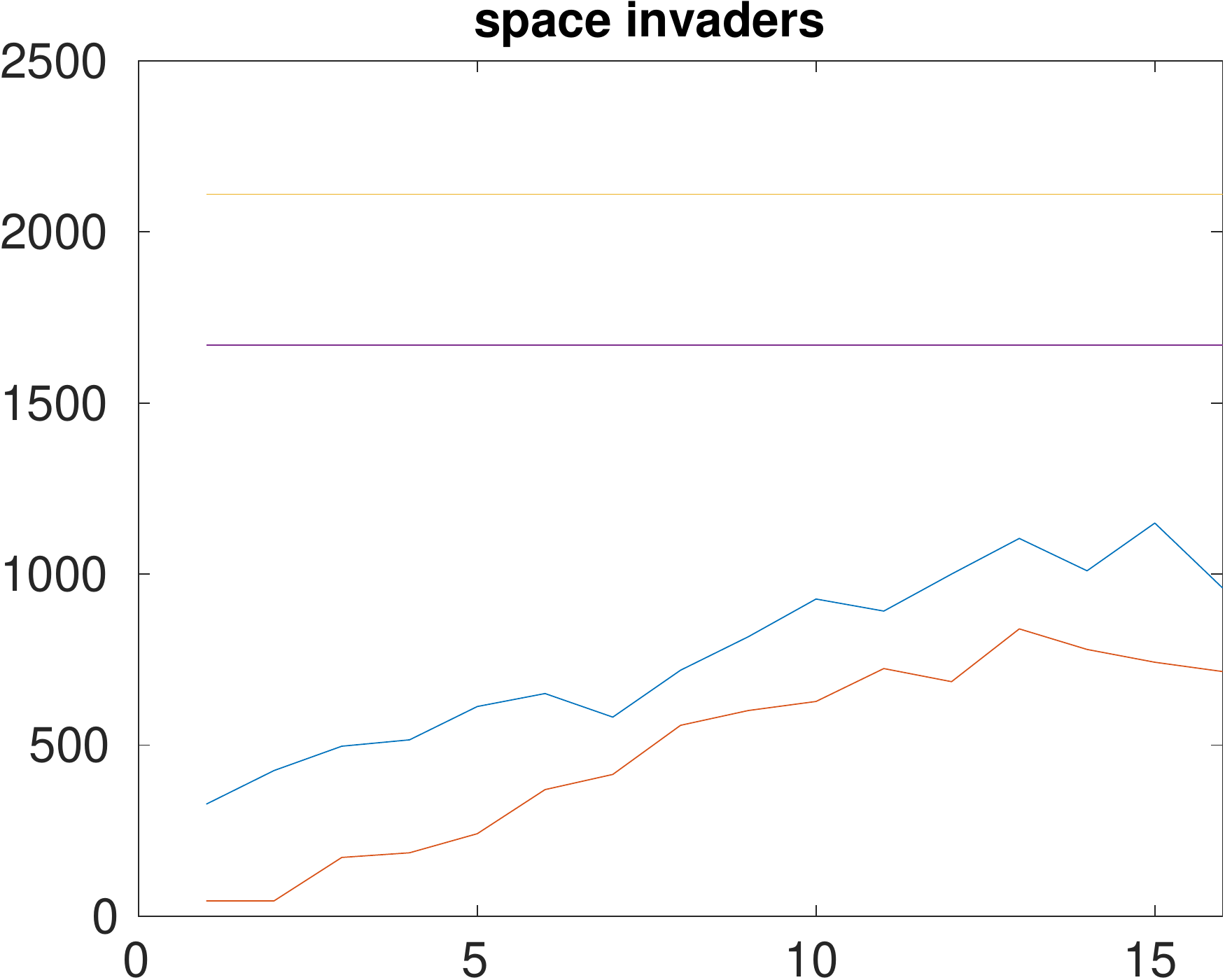}}
\end{tabular}
\caption{\small The Actor-Mimic training curves for the network trained solely with the policy regression objective (AMN-policy). The AMN-policy is trained for 16 epochs, or 4 million frames per game. We compare against the (smaller network) expert DQNs, which are trained until convergence. We also report the maximum test reward the expert DQN achieved over all training epochs, as well as the mean testing reward achieved over the last 10 epochs.}
\end{figure} 

\begin{figure}[hb!]
\begin{tabular}{ccc}
\includegraphics[width=0.3\linewidth]{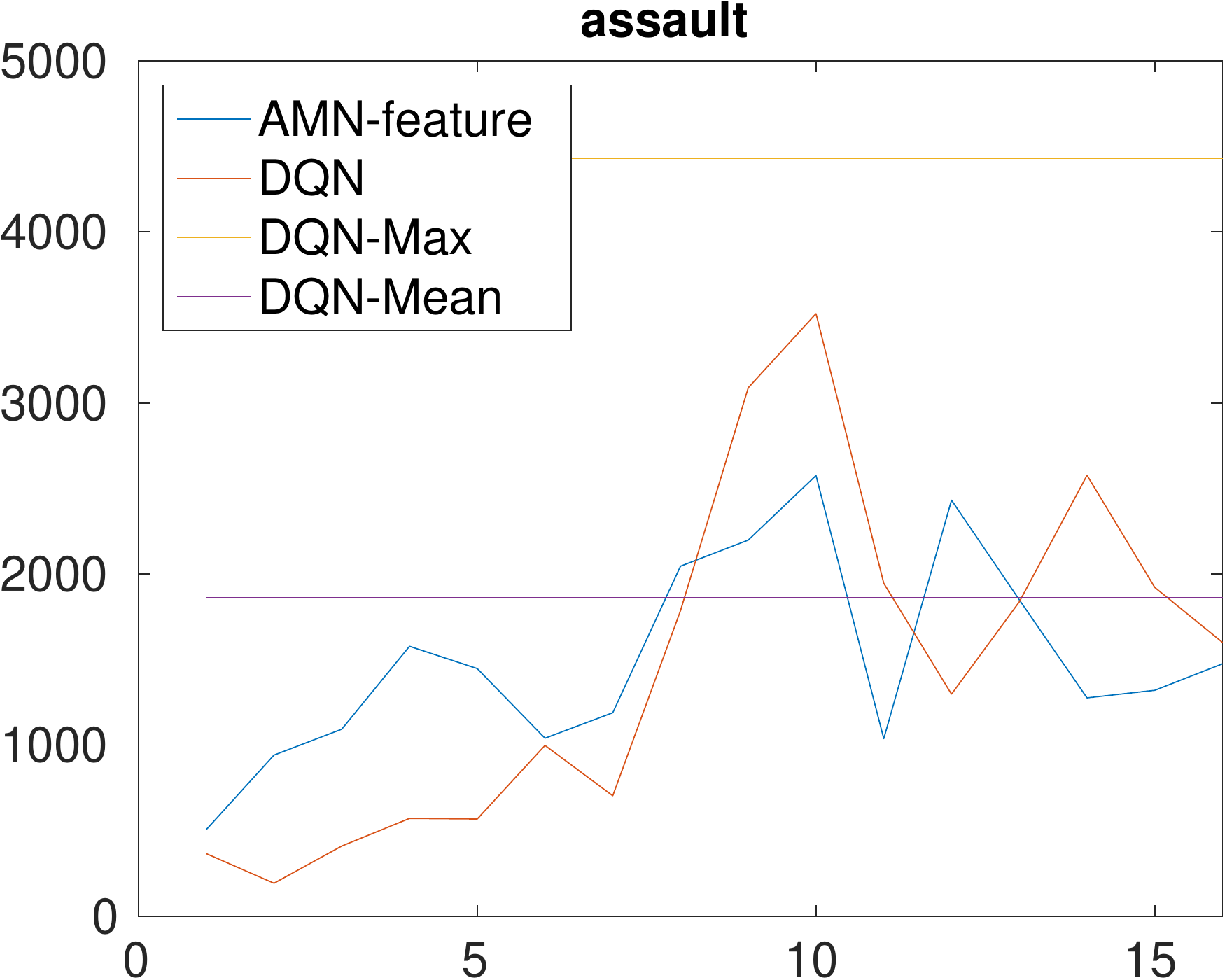} &
\includegraphics[width=0.3\linewidth]{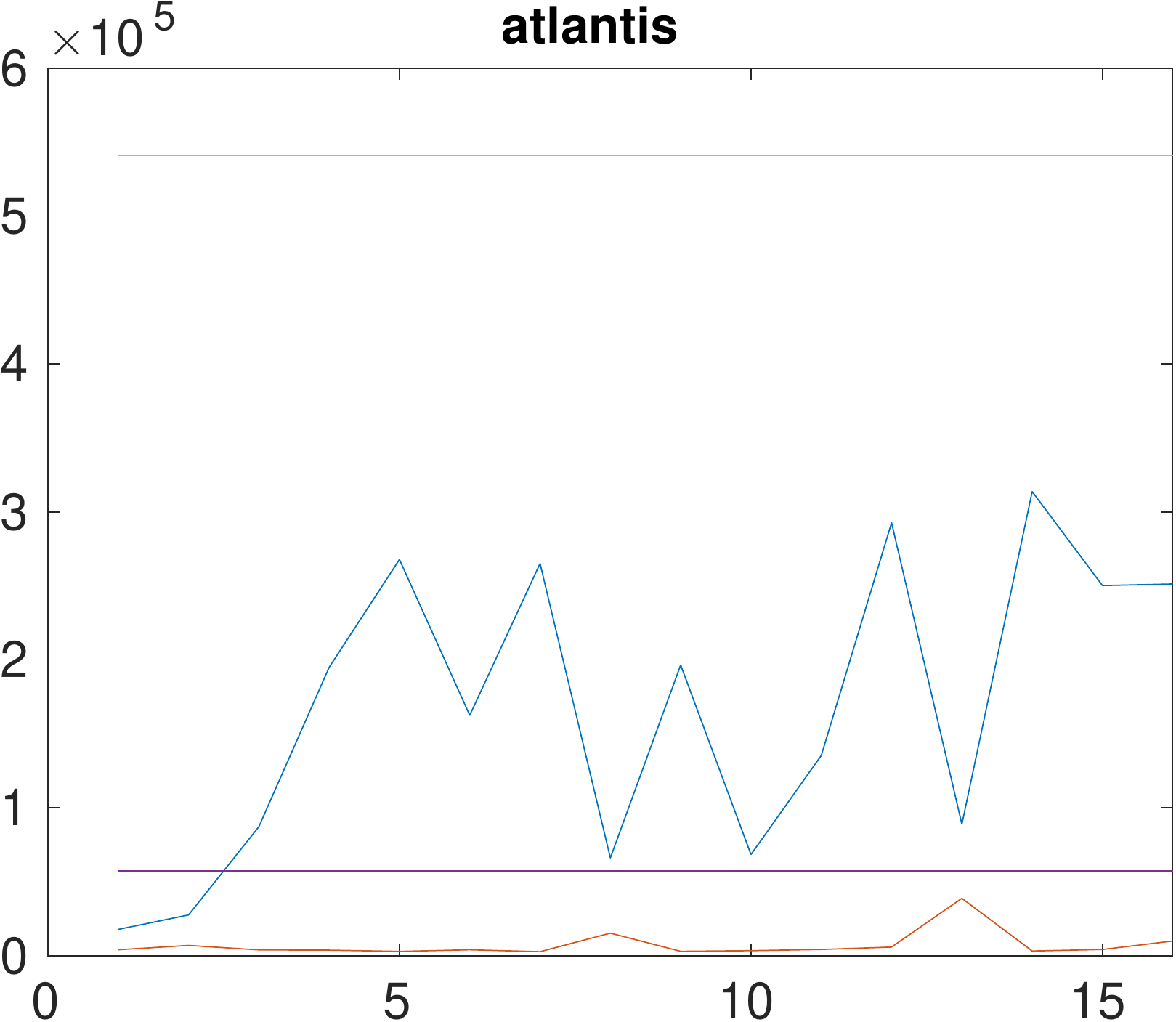} &
\includegraphics[width=0.3\linewidth]{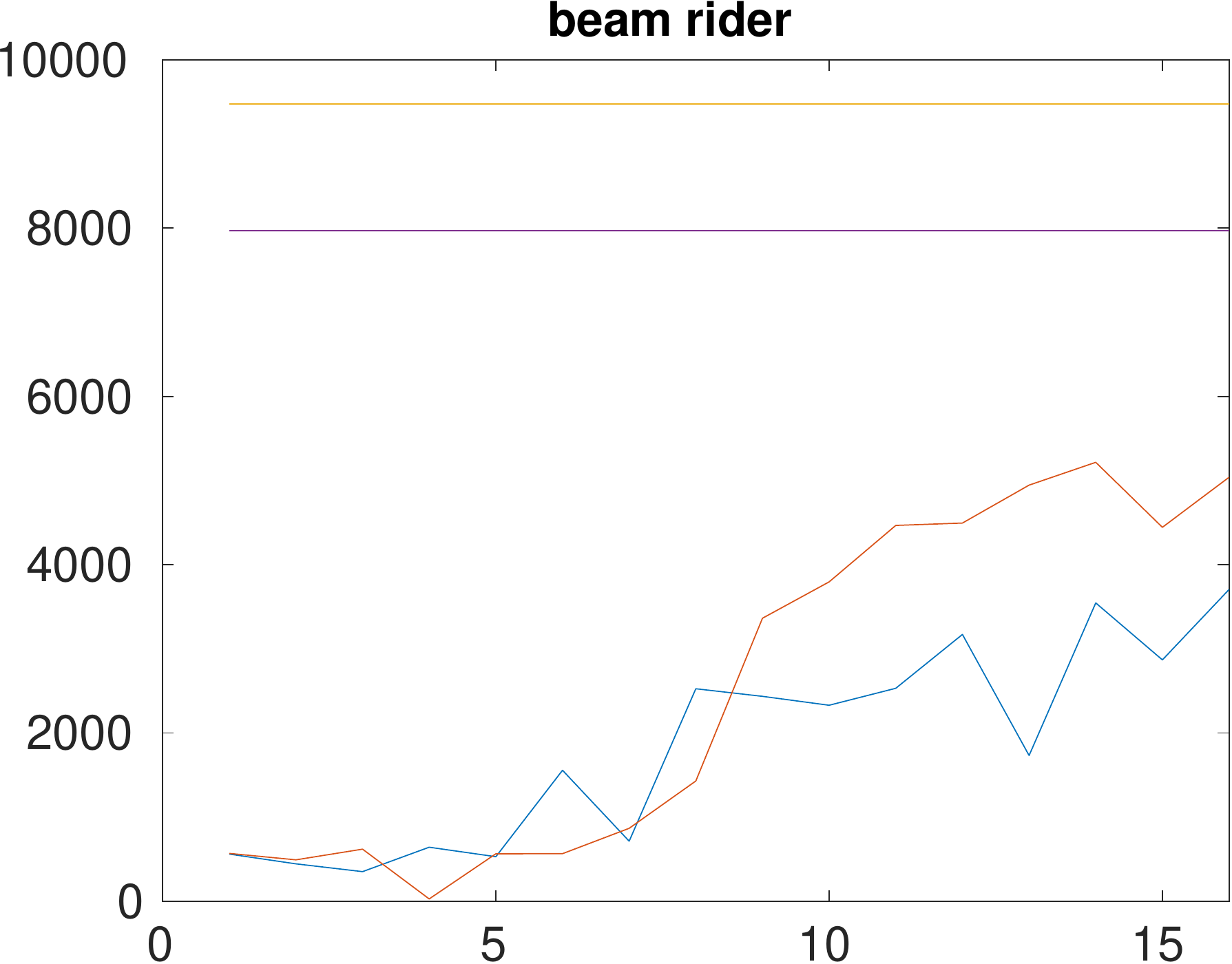} \\
\includegraphics[width=0.3\linewidth]{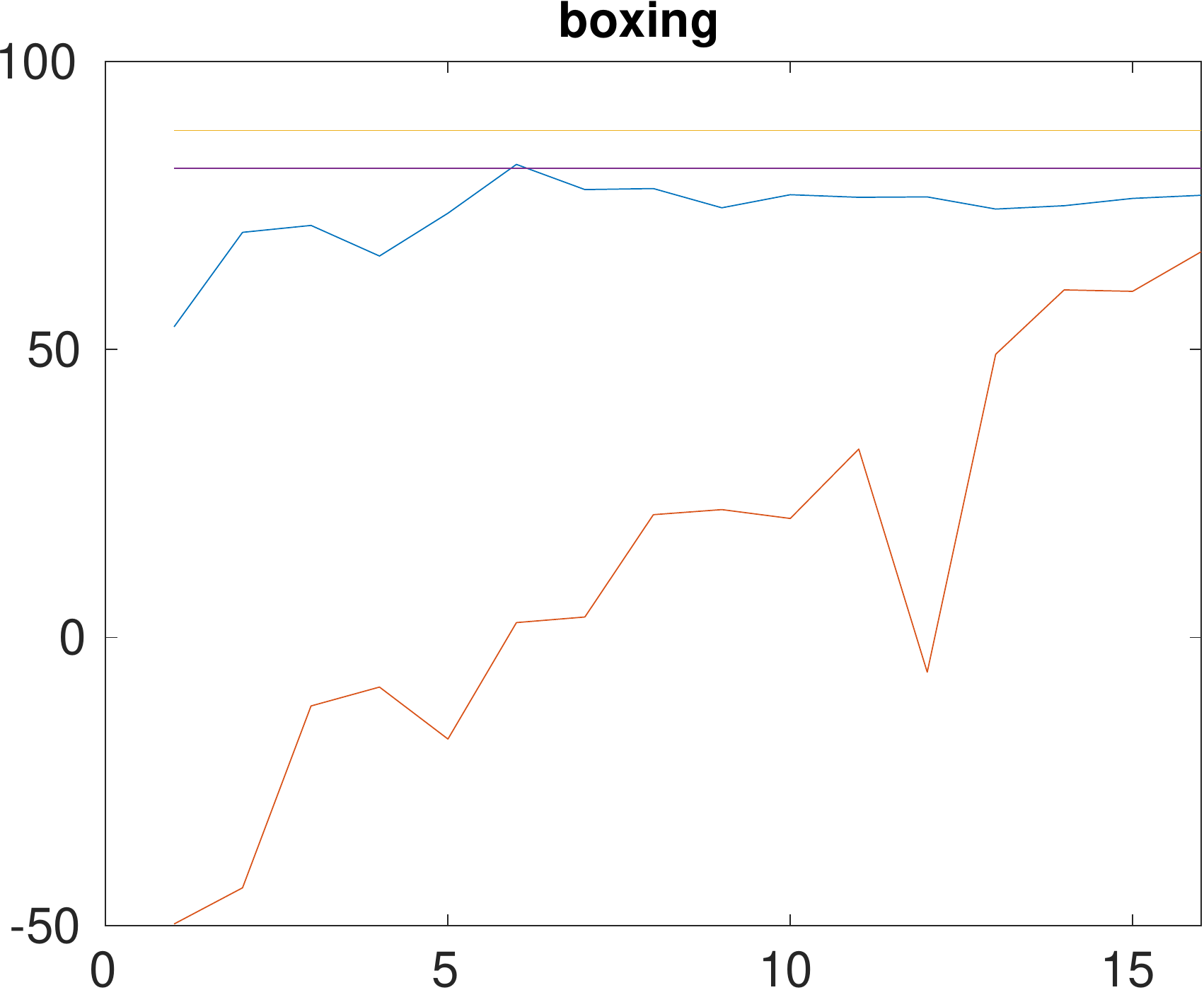} &
\includegraphics[width=0.3\linewidth]{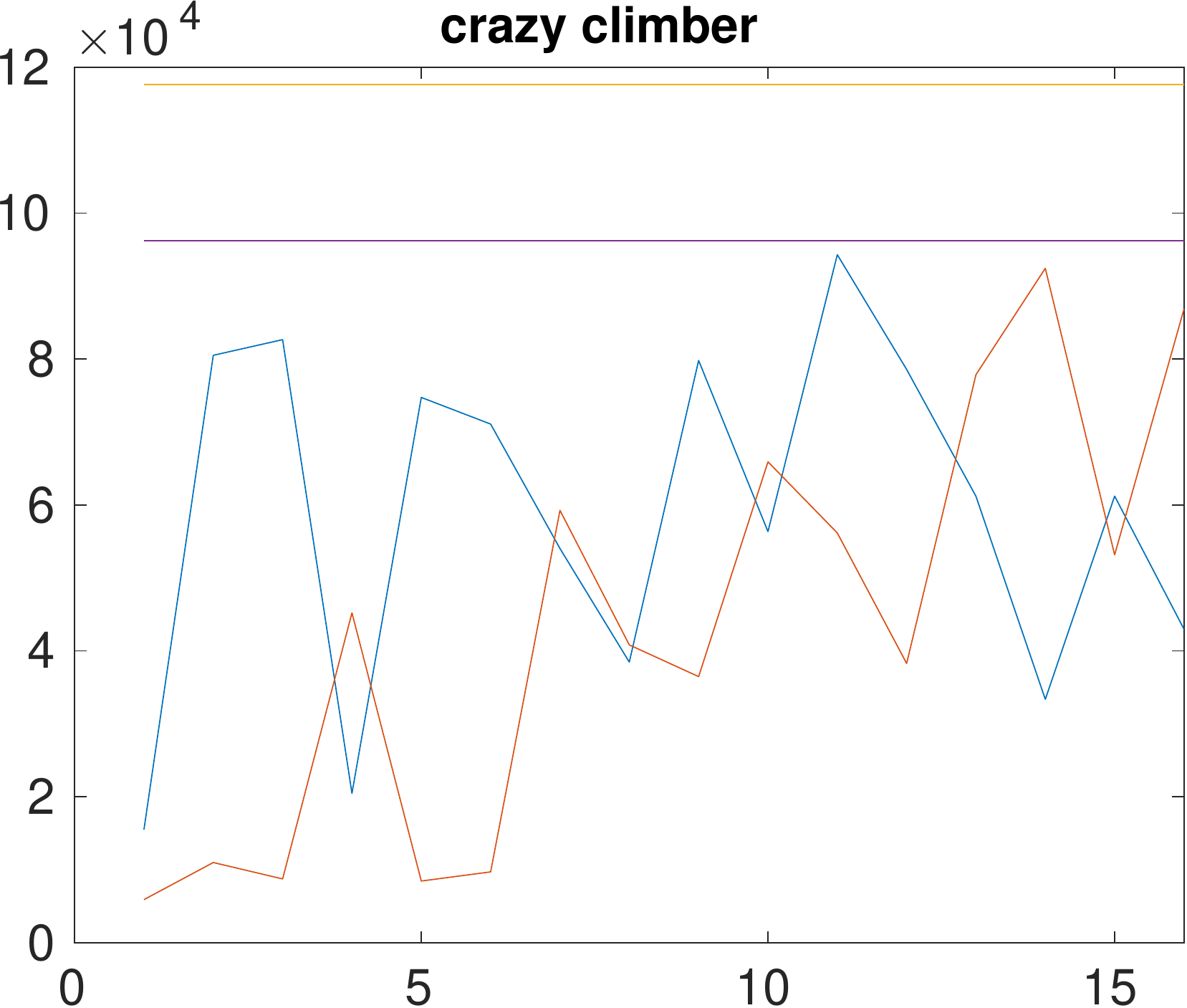} &
\includegraphics[width=0.3\linewidth]{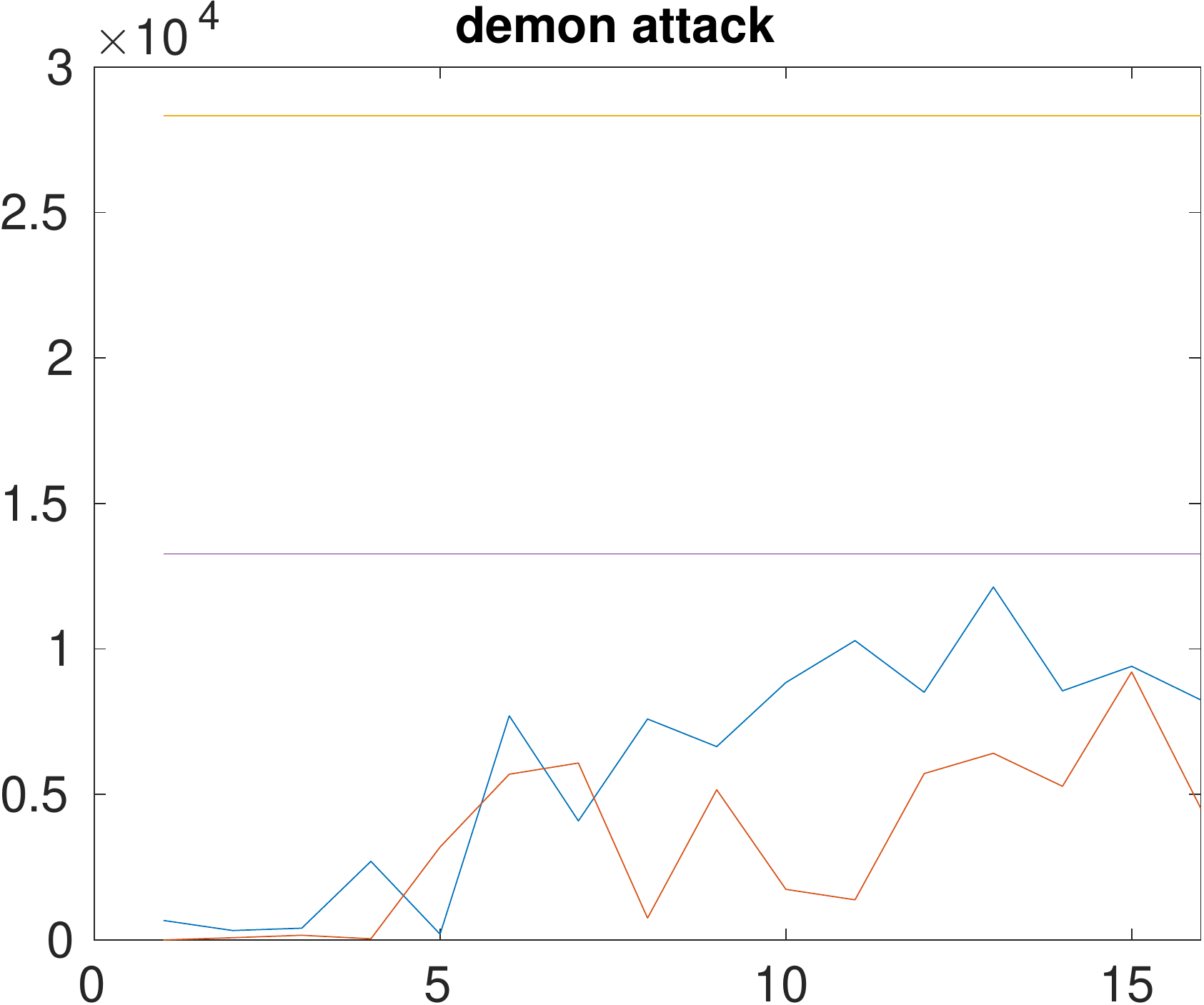} \\
\includegraphics[width=0.3\linewidth]{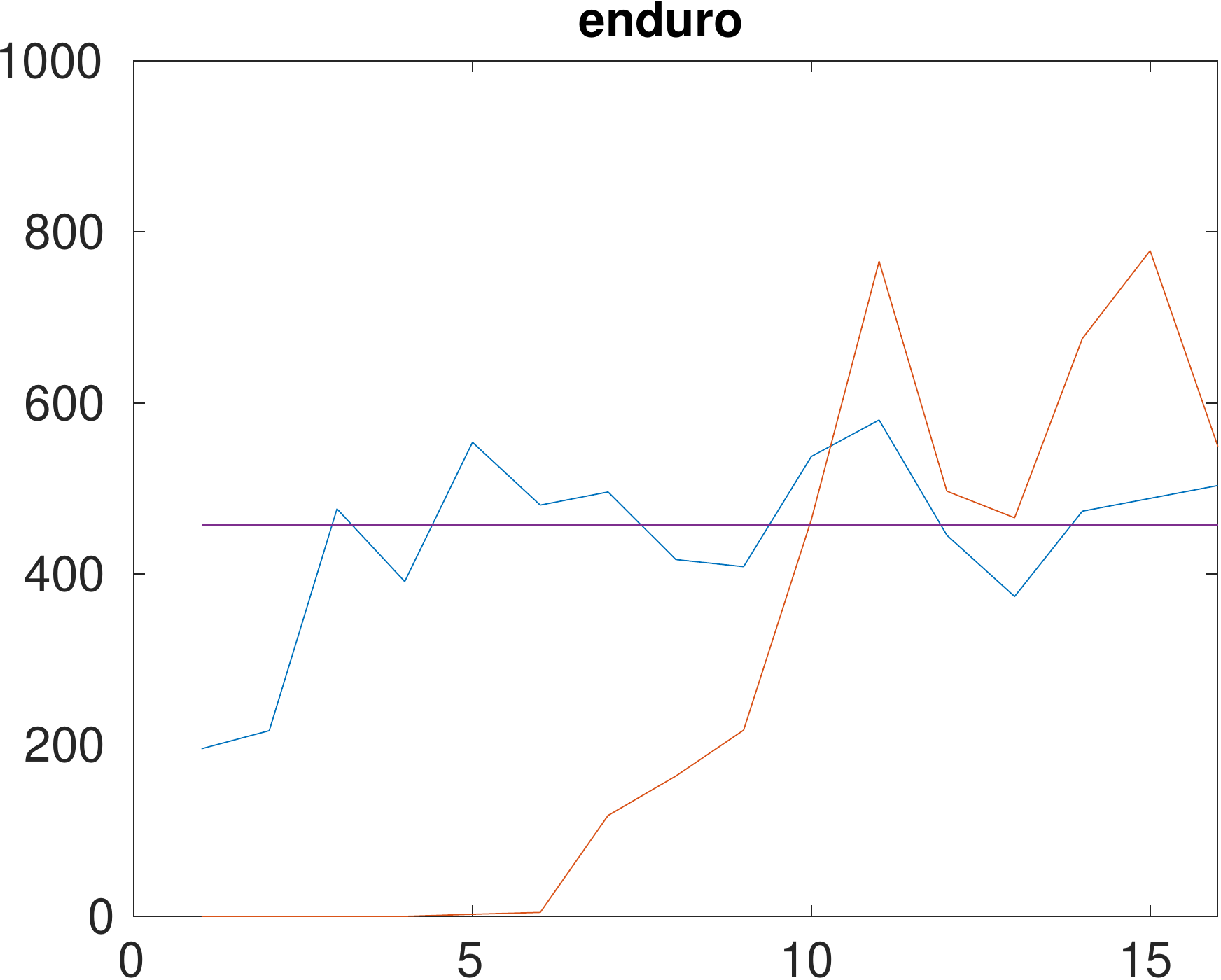} &
\includegraphics[width=0.3\linewidth]{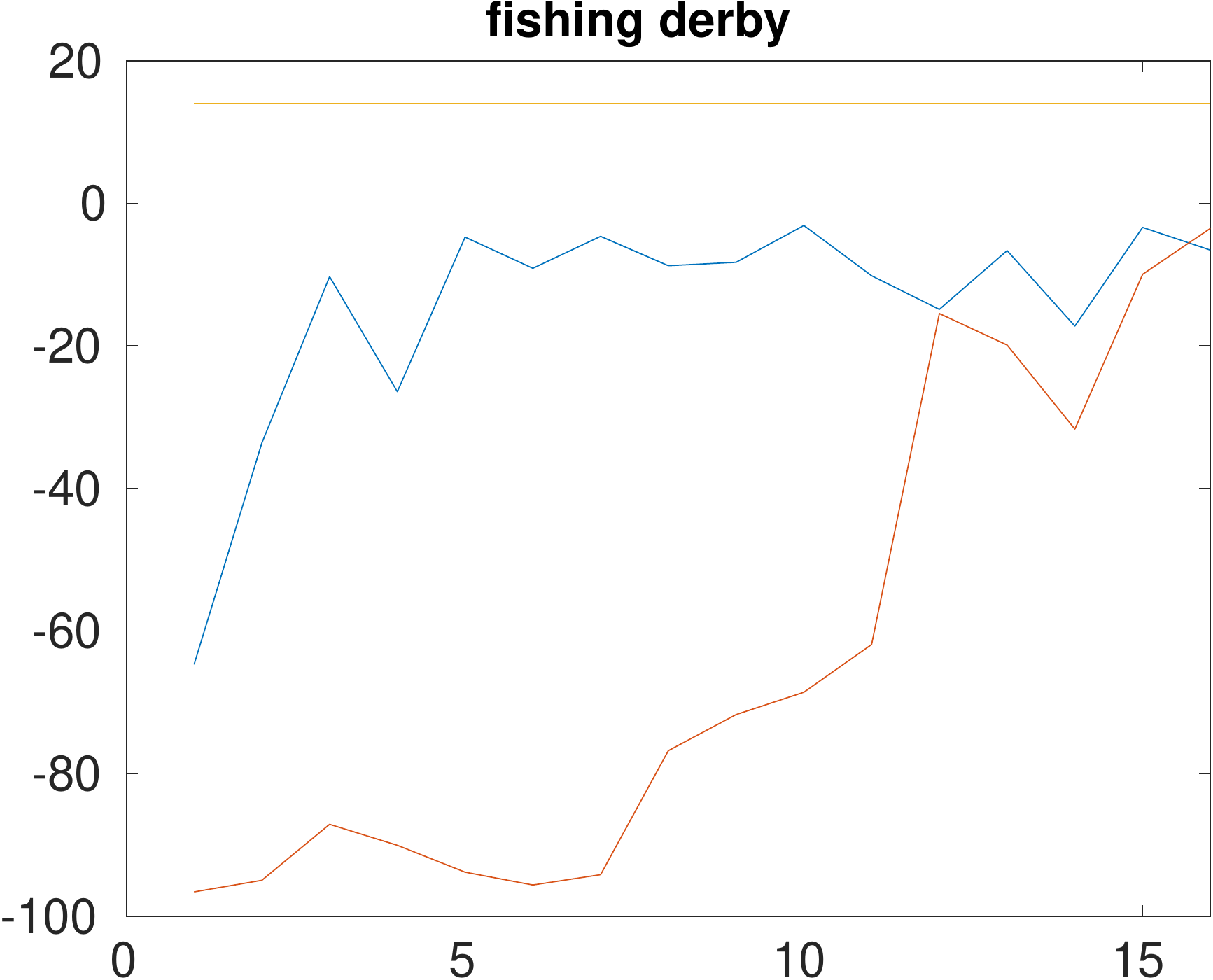} &
\includegraphics[width=0.3\linewidth]{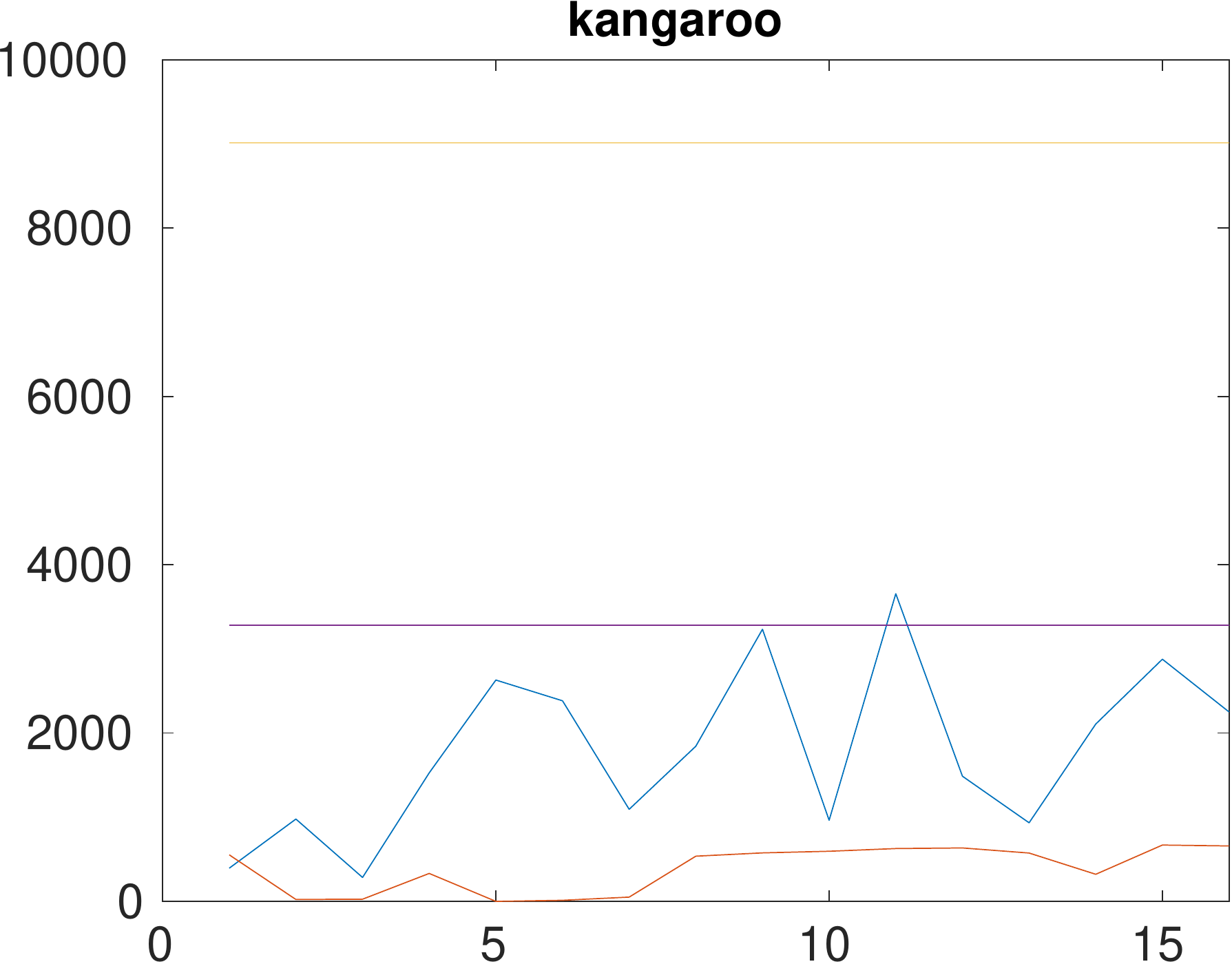} \\
\includegraphics[width=0.3\linewidth]{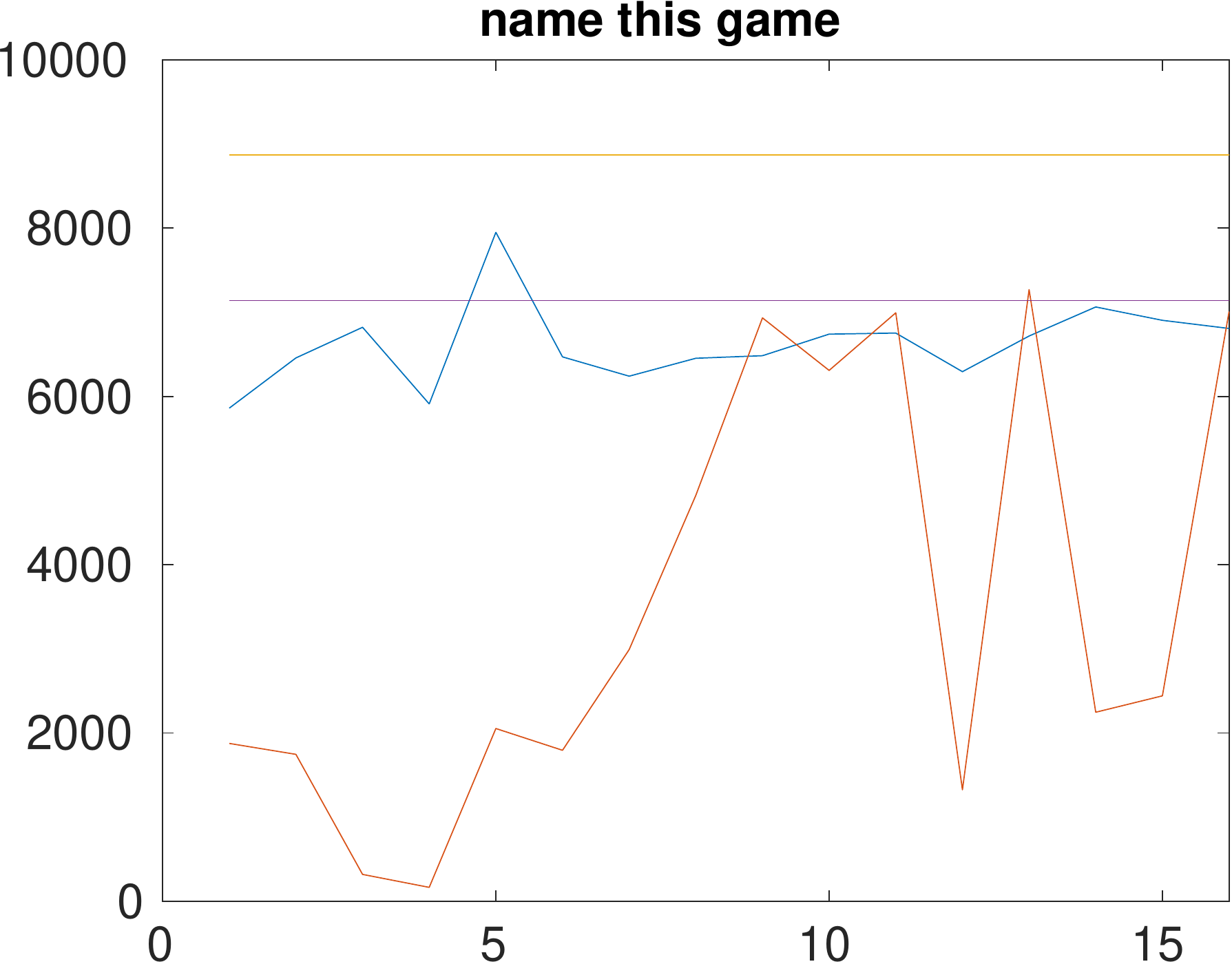} &
\includegraphics[width=0.3\linewidth]{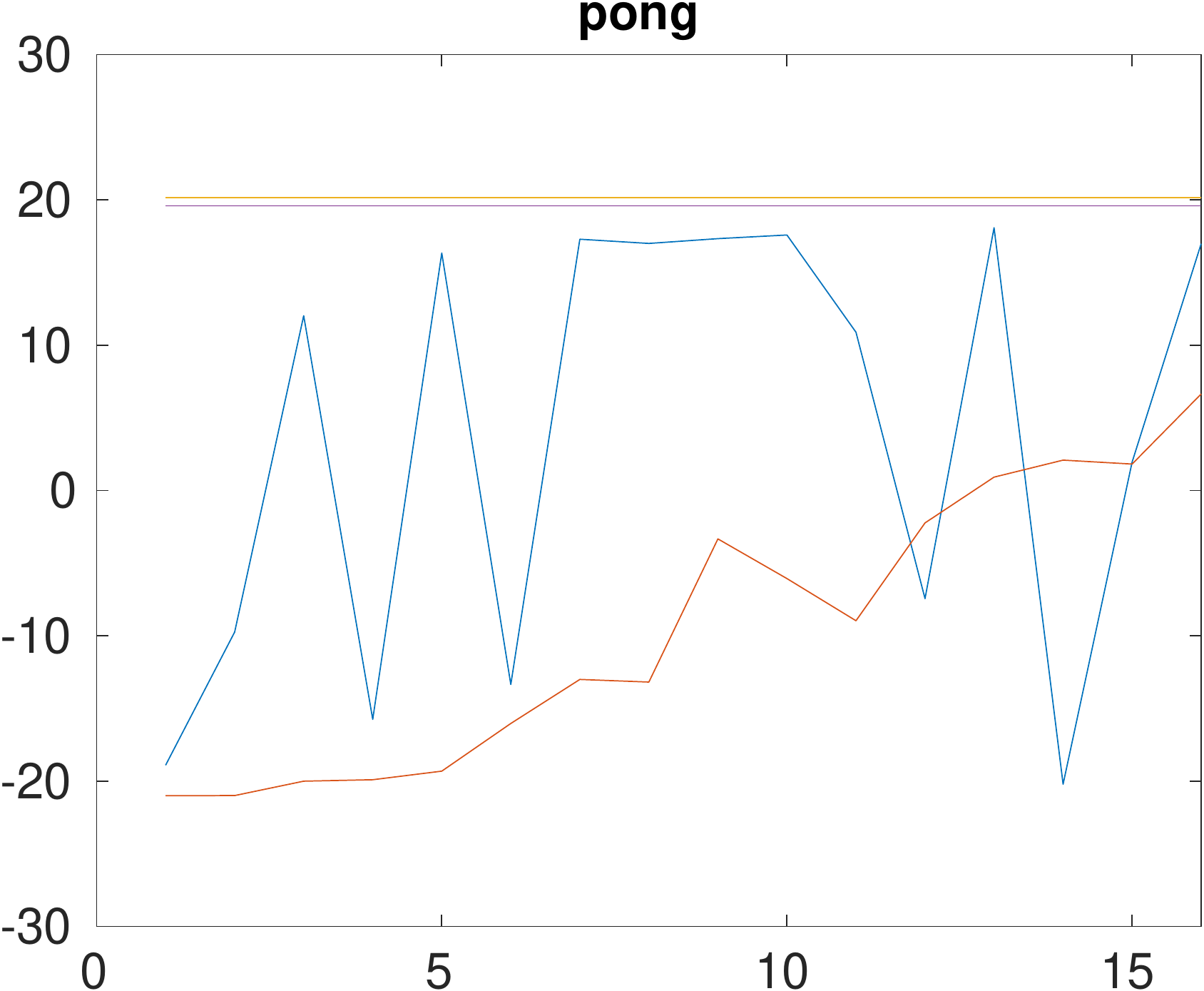} &
\includegraphics[width=0.3\linewidth]{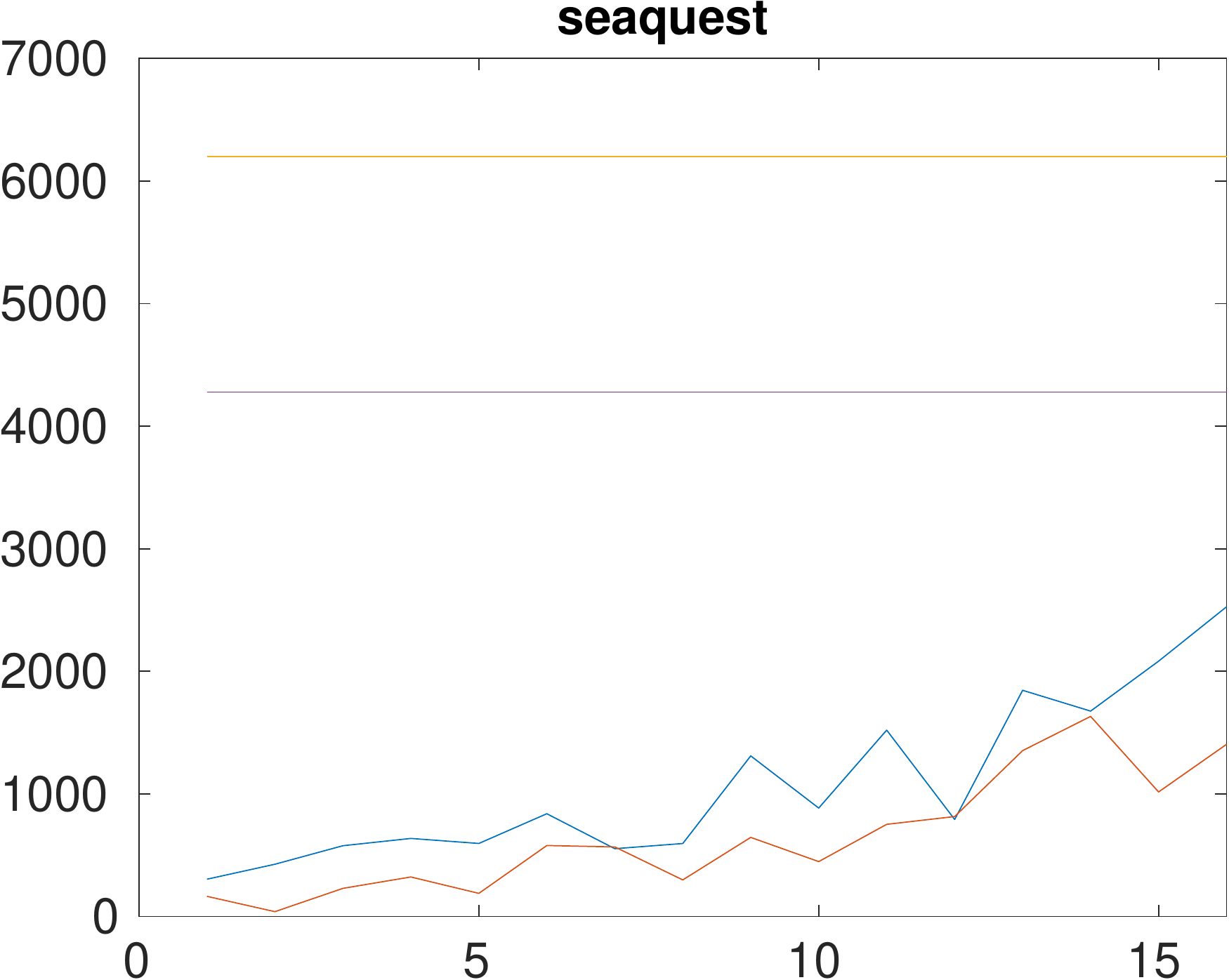} \\
\multicolumn{3}{c}{\includegraphics[width=0.3\linewidth]{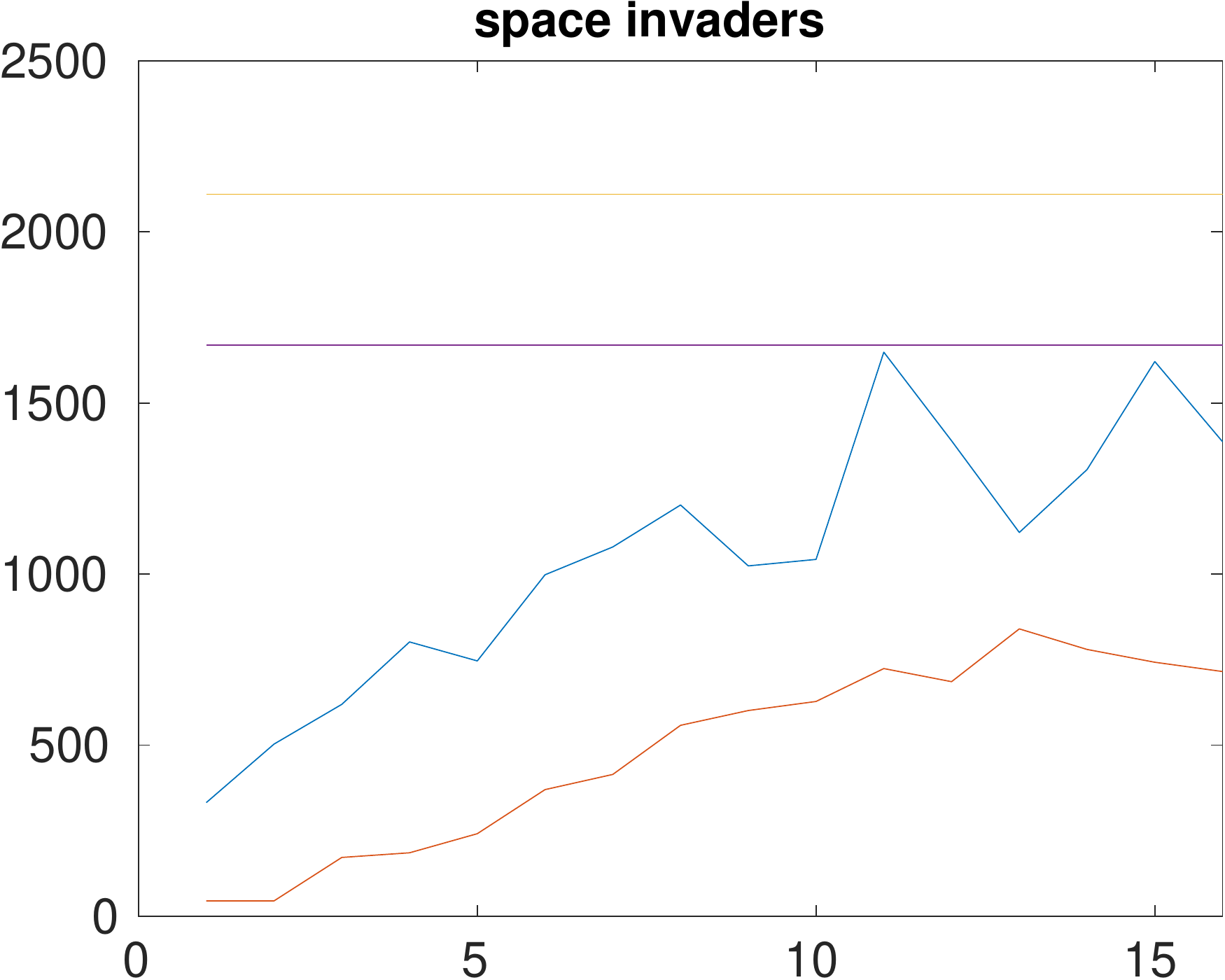}}
\end{tabular}
\caption{\small The Actor-Mimic training curves for the network trained with both the feature and policy regression objective (AMN-feature). The AMN-feature is trained for 16 epochs, or 4 million frames per game. We compare against the (smaller network) expert DQNs, which are trained until convergence. We also report the maximum test reward the expert DQN achieved over all training epochs, as well as the mean testing reward achieved over the last 10 epochs.}
\end{figure} 

\end{appendices}

\begin{appendices}

\section{Table 1 Barplot}
\label{app:barplot}

\begin{figure}[htb!]
\begin{tabular}{cc}
\includegraphics[width=0.45\linewidth]{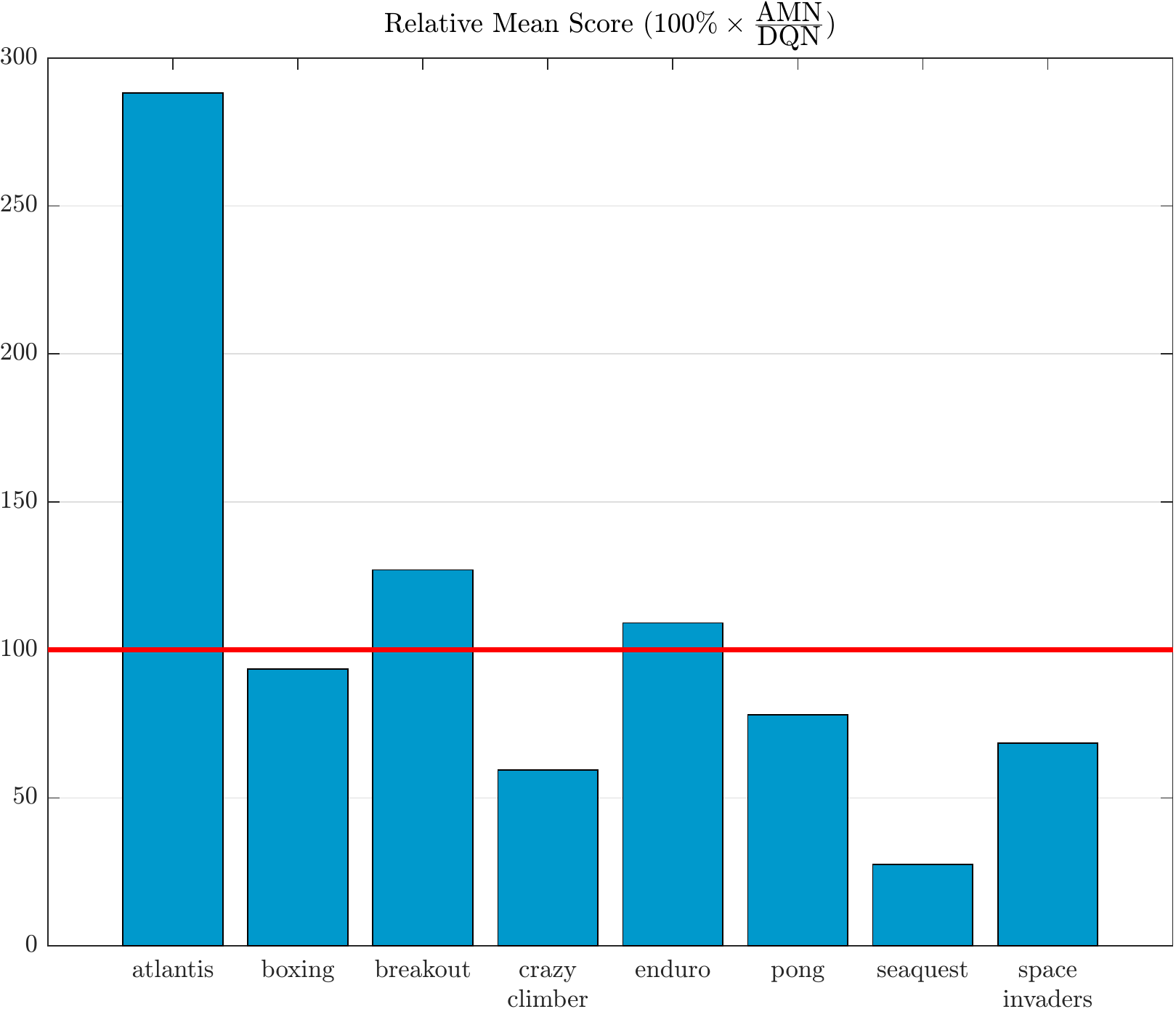} &
\includegraphics[width=0.45\linewidth]{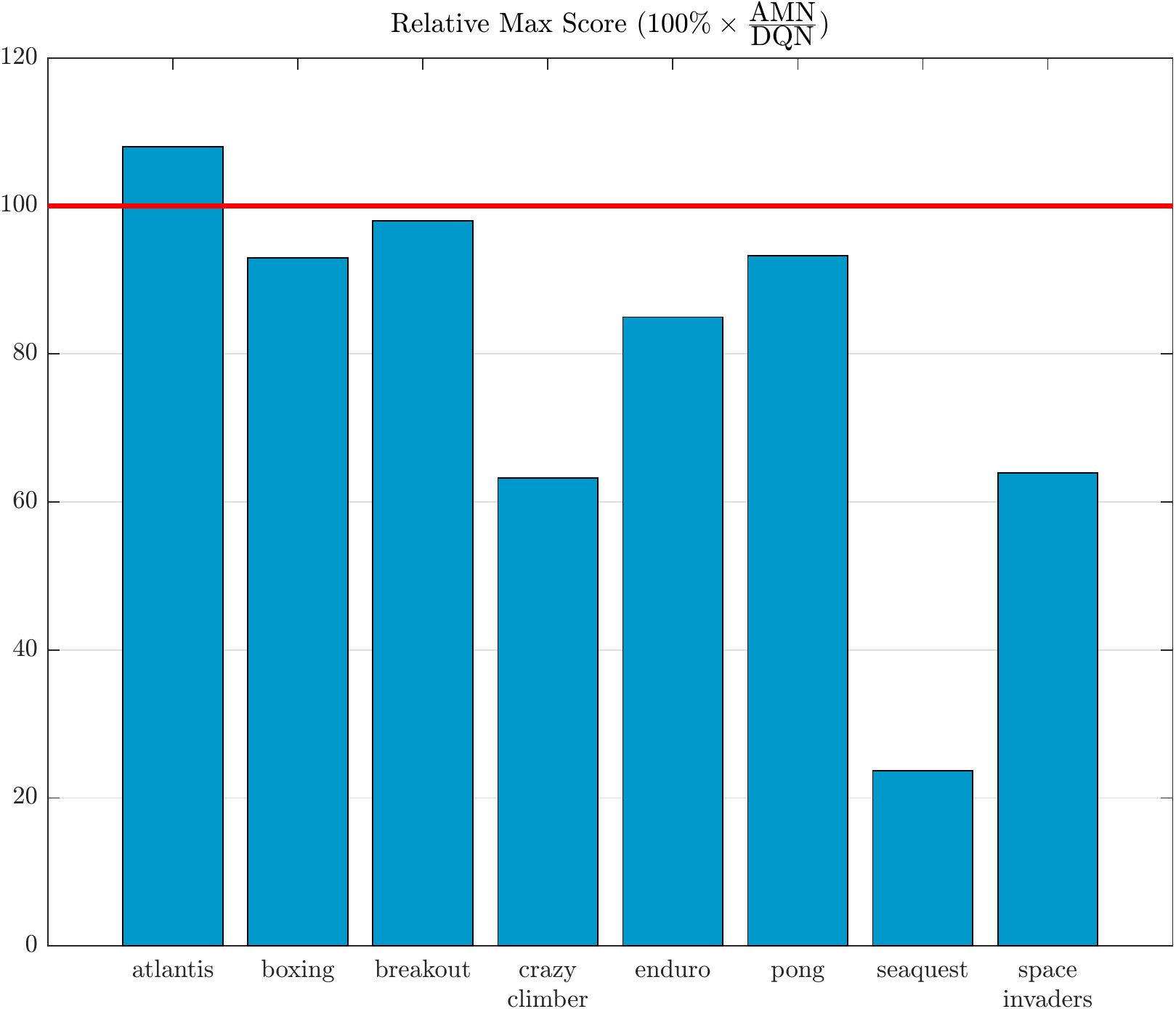}
\end{tabular}
\caption{\small Plots showing relative mean reward improvement (left) and relative max reward improvement (right) of the multitask AMN over the expert DQNs. See Table~\ref{AMN_8games_table} for details on how these values were calculated.}
\label{fig:amn_relative_improvement}
\end{figure}

\end{appendices}

\begin{appendices}

\section{Table 2 Learning Curves}
\label{app:learningcurves}

\begin{figure}[htb!]
\begin{tabular}{cccc}
\includegraphics[width=0.22\linewidth]{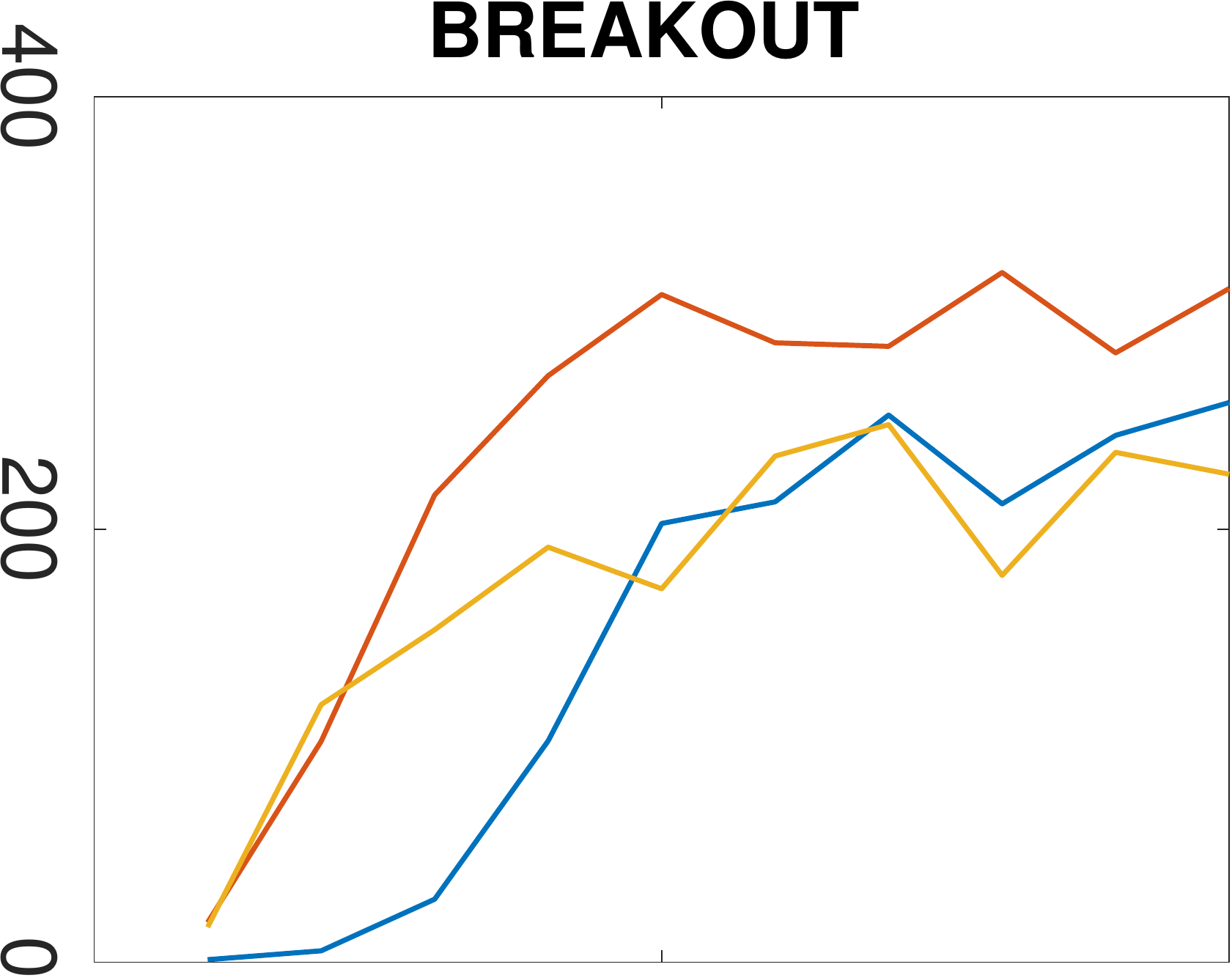} &
\includegraphics[width=0.22\linewidth]{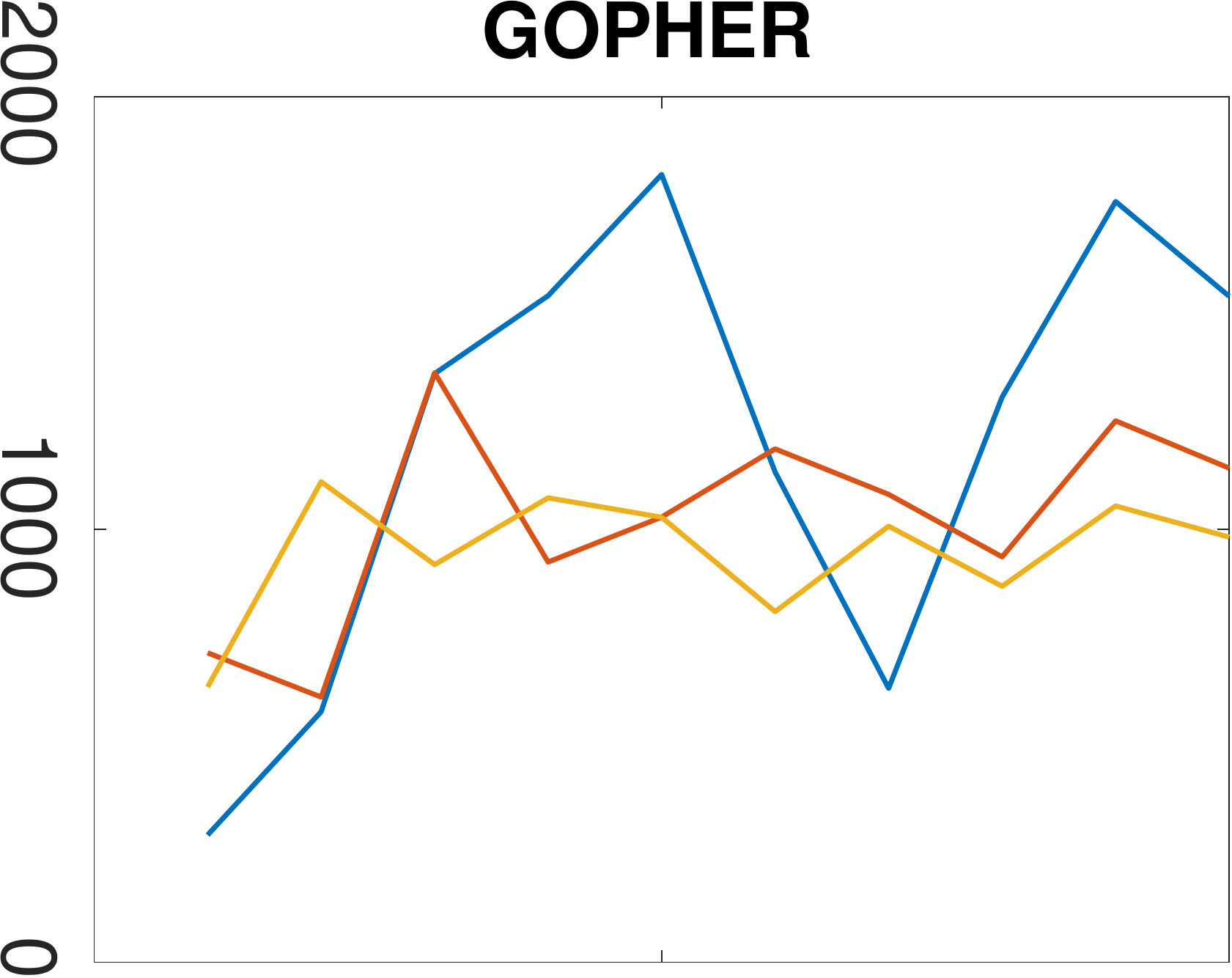} &
\includegraphics[width=0.22\linewidth]{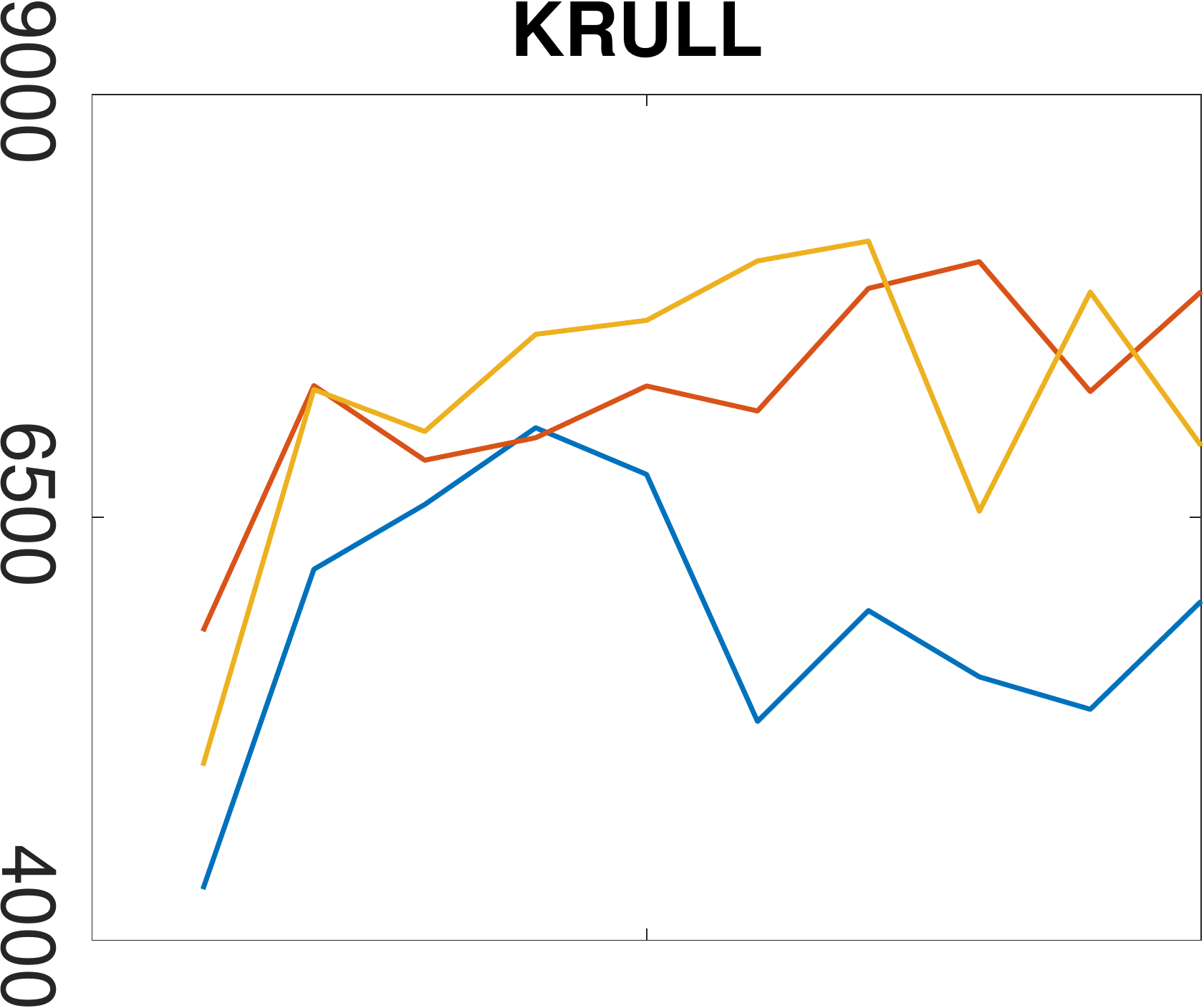} &
\includegraphics[width=0.22\linewidth]{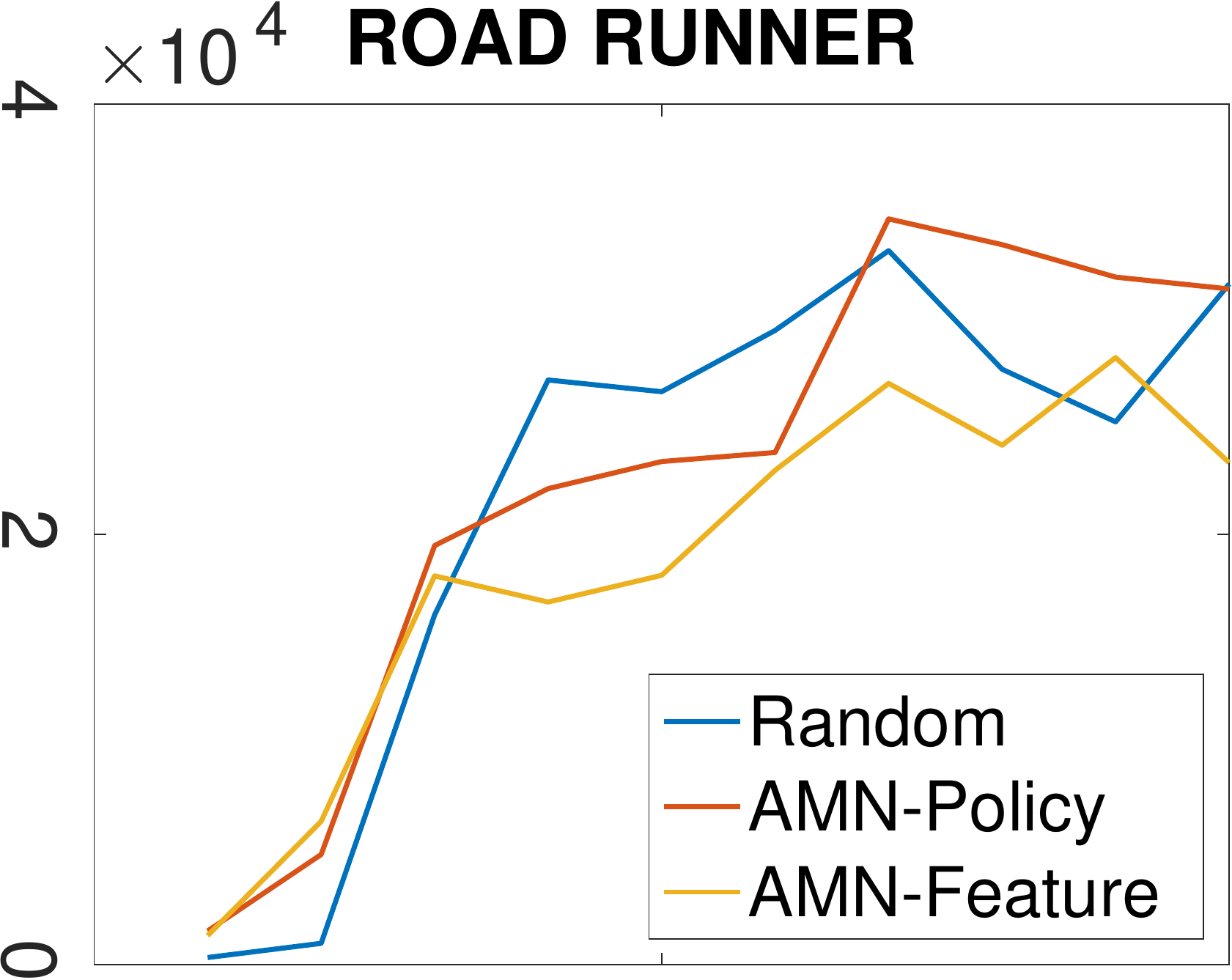} \\
\multicolumn{4}{c}{\begin{tabular}{ccc}
    \includegraphics[width=0.23\linewidth]{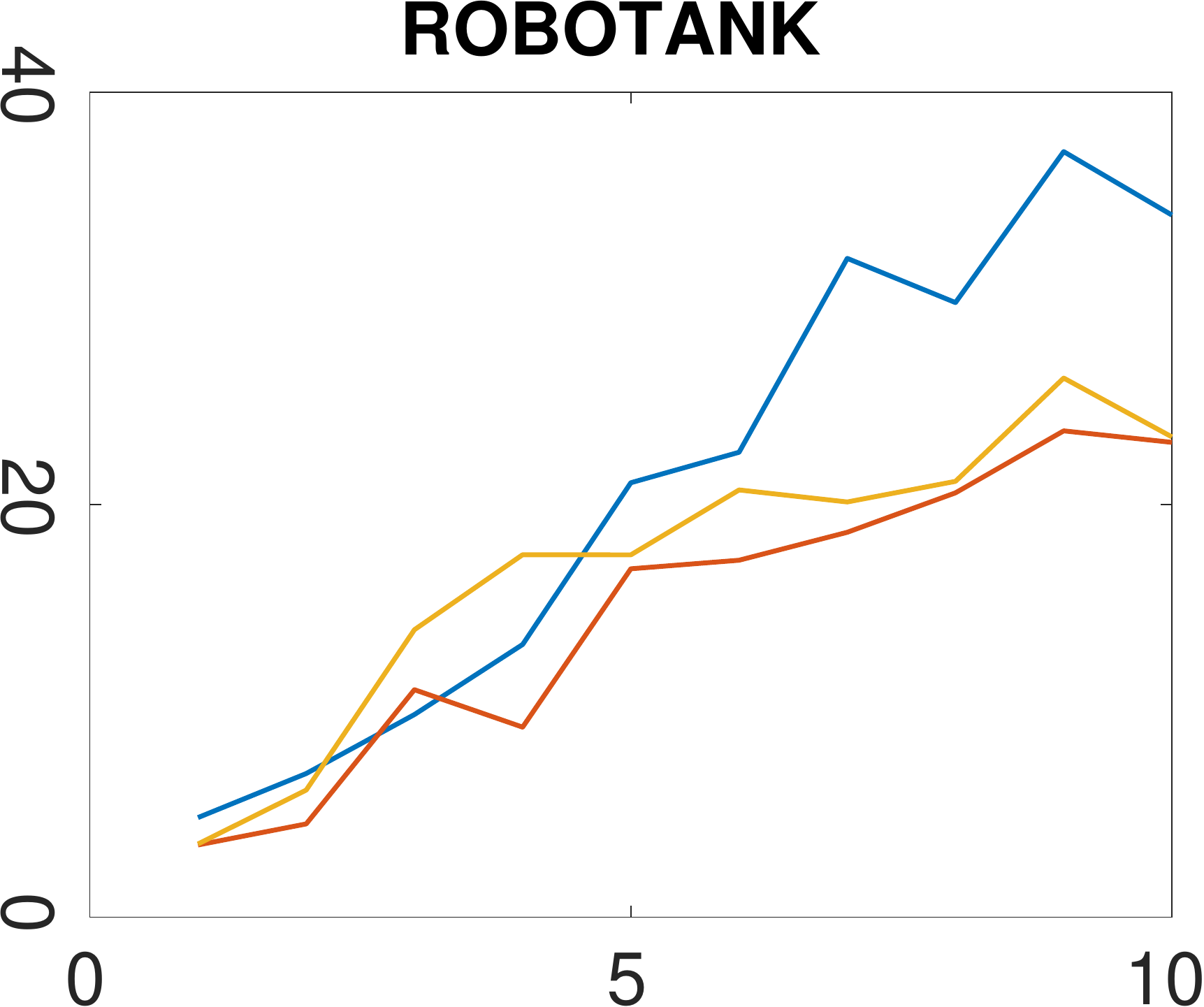} &
    \includegraphics[width=0.23\linewidth]{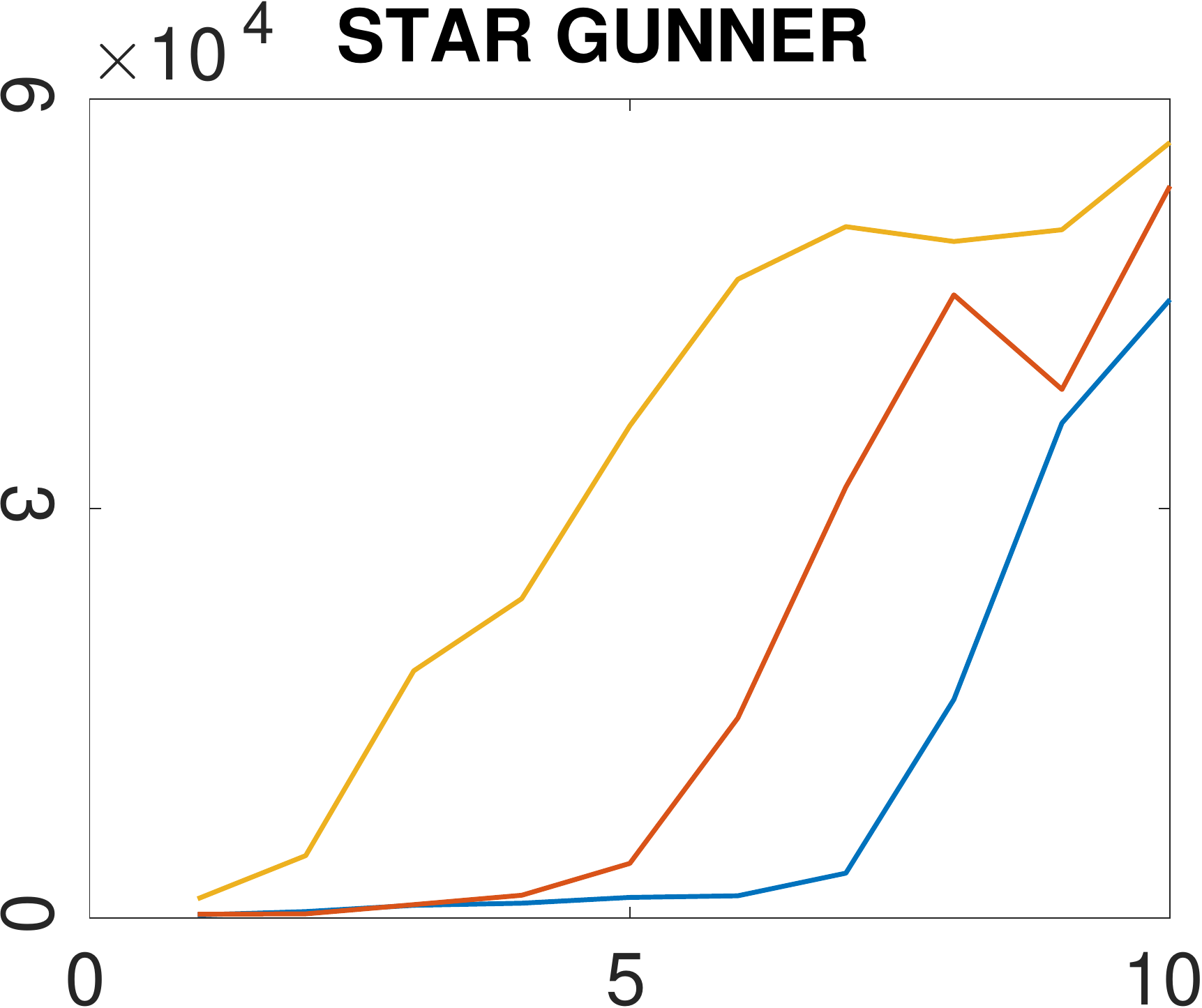} &
    \includegraphics[width=0.23\linewidth]{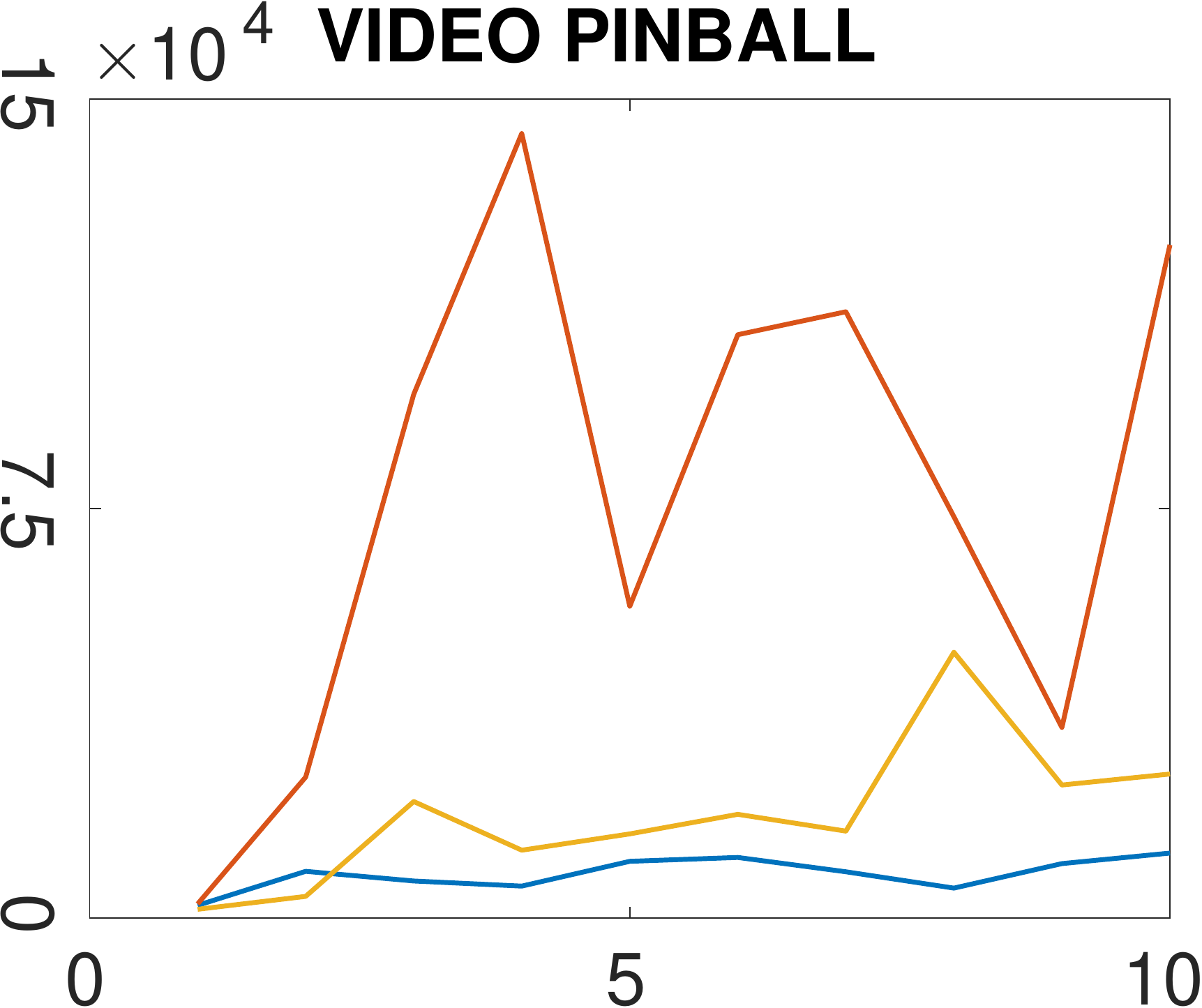}
  \end{tabular}}
\end{tabular}
\caption{\small Learning curve plots of the results in Table\ref{AMN_transfer_table}.}
\label{fig:transfer_learn_curve}
\end{figure}

\end{appendices}

\end{document}